\documentclass[twoside,11pt]{article}

\usepackage{jair}
\usepackage{theapa}
\usepackage{stmaryrd}
\usepackage{graphicx}
\usepackage{latexsym}
\usepackage{times}
\usepackage{amsfonts}
\usepackage{amsmath}
\usepackage{amssymb}
\usepackage{enumerate}
\usepackage{color}

\newcommand{\you}[2]{ {}  {{\textcolor{magenta}{#2}}}}
\newcommand{\jiacomment}[1]{{\bf \textcolor{magenta} {*** Notes by Jia: ***} #1 }}
\newcommand{\tecomment}[1]{{\bf *** TE: #1 ***}}

\newcommand\citeN[1]{\citeauthor{#1} \citeyear{#1}}

\newcommand{\var}{\varphi}

\newcommand{\rto}{\rightarrow}
\newcommand{\lto}{\leftarrow}

\newcommand{\Rto}{\Rightarrow}
\newcommand{\Lto}{\Leftarrow}

\newtheorem{definition}{Definition}
\newtheorem{examp}{Example}
\newenvironment{example}{\begin{examp}\rm}{\end{examp}}
\newtheorem{lemma}{Lemma}
\newtheorem{proposition}{Proposition}
\newtheorem{theorem}{Theorem}
\newtheorem{corollary}[theorem]{Corollary}
\newenvironment{proof}{{\bf Proof:}}{\hfill\rule{2mm}{2mm}\\ }
\newenvironment{proof*}{{}}{\hfill\rule{2mm}{2mm}\\ }

\long\def\comment#1{}

\newcommand{\Not}{not \,}

\newcommand{\Pos}{\textit{Pos}}

\newcommand{\Neg}{\textit{Neg}}

\newcommand{\Body}{Body}

\newcommand{\lfp}{\textit{lfp}}

\newcommand{\Th}{\textit{Th}}
\newcommand{\HB}{\textit{HB}}

\newcommand{\DL}{\textit{DL}}
\newcommand{\NDL}{\textit{Ndl}}
\newcommand{\UC}{\textit{UC}}

\newcommand{\NEXP}{\textmd{\rm NEXP}}
\newcommand{\EXP}{\textmd{\rm EXP}}
\newcommand{\coNEXP}{\textmd{\rm co-NEXP}}
\newcommand{\NP}{\textmd{\rm NP}}
\newcommand{\coNP}{\textmd{\rm co-NP}}
\newcommand{\Pol}{\textmd{\rm P}}

\ShortHeadings{Embedding Description Logic Programs into Default Logic}
{Yisong, Jia-Huai, Li Yan,  Yi-Dong \& Eiter}

\begin{document}

\title{Embedding Description Logic Programs into Default Logic}


\author{\name Yisong Wang \email yswang168@gmail.com \\
        \addr Department of Computer Science, Guizhou University, Guiyang, Guizhou Province, China, 550025\\
       \name Jia-Huai You \email you@cs.ualberta.ca \\
       \name Li Yan Yuan \email yuan@cs.ualberta.ca\\
       \addr Department of Computing Science, University of Alberta, Canada,
       \AND
       \name Yi-Dong Shen \email ydshen@ios.ac.cn \\
       \addr State Key Laboratory of Computer Science, Institute of Software, Chinese Academy of Sciences, China
       \AND
       \name Thomas Eiter \email eiter@kr.tuwien.ac.at\\
       \addr Institut f\"{u}r Informationssysteme, Technische Universit\"{a}t Wien, Favoritenstra{\ss}e 9-11, A-1040 Vienna, Austria}


\maketitle

\begin{abstract}
Description logic programs (dl-programs) under the  answer set semantics
formulated by Eiter {\em et al.} have been considered as a prominent formalism for integrating rules and
ontology knowledge bases. A question of interest has been whether dl-programs can be captured in a general formalism of nonmonotonic logic.
In this paper, we study the possibility of embedding dl-programs into default logic. We show that dl-programs under the strong and weak answer set semantics can be embedded in default logic by
combining two
translations, one of which eliminates the constraint operator from
nonmonotonic dl-atoms and the other translates a dl-program into a
default theory.
For dl-programs without nonmonotonic dl-atoms but with the negation-as-failure operator,
our embedding is polynomial, faithful,
and modular.
In addition, our default logic encoding can be extended in a simple way to capture recently proposed weakly well-supported answer set semantics, for arbitrary dl-programs.
These results
reinforce the argument that default logic can serve as a fruitful foundation for query-based approaches to integrating ontology and rules. With its simple syntax and intuitive semantics, plus available computational results,
default logic can be considered an attractive approach to integration of ontology and rules.

\comment{
\tecomment{Say: DL is in sense ``simpler'' than other formalisms
(like AEL, MKNF etc) and more computational results are available,
which make this more attractive? // It is ok to me. How about Prof. You? -Yisong}
}

\end{abstract}

\section{Introduction}
Logic programming under the answer set semantics (ASP) has been
recognized as an expressive nonmonotonic reasoning framework for declarative
problem solving and knowledge representation \cite{marek99,Niemela99}. Recently, there has been an
extensive interest in combining ASP with other logics or
reasoning mechanisms.\comment{(e.g.,
\cite{DBLP:conf/ijcai/MotikR07,DBLP:journals/ai/EiterILST08,Mellarkod:AMAI2008}).}
One of the main interests in this direction is the integration of
ASP with description logics (DLs) for the Semantic Web. This is due to the fact that,
although ontologies expressed in DLs and rules in ASP are two
prominent knowledge representation formalisms, each of them has
limitations on its own. As (most) DLs are fragments of (many sorted)
first order logic, they do not support default, typicality, or
nonmonotonic reasoning in general. On the other hand, thought there are some recent attempts to extend ASP beyond propositional logic, the core, effective
reasoning methods are  designed essentially for computation of ground programs; in particular, ASP typically does not reason with unbounded or infinite domains, nor does it support
quantifiers. An integration of the two can offer features of both.

A number of proposals for integrating ontology and (nonmonotonic) rules
have been put forward
\cite{eiter-auto-ontology2007,DBLP:Bruijn:RR:2007,DBLP:journals/ai/EiterILST08,Motik:JACM:2010,Rosati,DBLP:Rosati:KR:2006,DBLP:TOCL:Analyti:2011,DBLP:Lukasiewicz:TKDE:10,DBLP:Lee:LPNMR:2011}.
The existing approaches
can be roughly classified into three categories.
In the first, typically a nonmonotonic formalism is adopted which naturally embodies
both first-order logic and
rules, where
ontology
and rules are written in the same language resulting in a tight coupling \cite{DBLP:Bruijn:TOCL,Motik:JACM:2010,DBLP:Lukasiewicz:TKDE:10}.
The second is a loose approach:
an ontology knowledge base and
rules share the same constants but not the same predicates,
and inference-based communication is
via a well-designed interface, called
dl-atoms \cite{DBLP:journals/ai/EiterILST08}.
In the third approach, rules are treated as
hybrid formulas where in model building
the predicates in the language of the ontology are interpreted classically, whereas those
in the language of rules are interpreted nonmonotonically \cite{Rosati,DBLP:Rosati:KR:2006,DBLP:Bruijn:RR:2007}.

The loose coupling approach above stands out as quite unique and it
possesses
some advantages.
In many practical situations, we would like
to combine existing knowledge bases,
possibly
under different logics. In this
case, a notion of interface is natural and necessary.
The formulation of dl-programs
adopts such interfaces to ontology knowledge bases. It is worth noticing
that dl-programs share many similarities with another recent
interesting formalism, called {\em nonmonotonic} {\em multi-context systems},
in which knowledge bases under arbitrary logics communicate through \emph{bridge rules} \cite{Brewka-AAAI-07}.

Informally,
a dl-program is a
pair $(O,P)$, where
$O$ is an ontology knowledge base expressed in a description logic,
and $P$ a logic program, where rule bodies may contain
queries to the knowledge base $O$, called {\em dl-atoms}. Such queries allow to specify
inputs from a logic program to the ontology knowledge base. In more detail, a dl-atom is of the form
\[\DL[S_1\ op_1\ p_1,\ldots,S_m\ op_m\ p_m;Q](\vec t)\]
where $Q(\vec t)$ is a query to $O$, and for each $i~(1\leq i\leq
m)$, $S_i$ is a concept or a role in $O$, $p_i$ is a predicate
symbol in $P$ having the same arity as $S_i$, and the
operator $op_i\in\{\oplus,\odot,\ominus\}$. Intuitively, $\oplus$
(resp., $\odot$) increases $S_i$ (resp., $\neg S_i$) by the extension
of $p_i$, while $\ominus$ (called the {\em constraint operator})
constrains $S_i$ to $p_i$, i.e., for an expression $S
\ominus p$, for any tuple of constants $\vec t$, in the absence of
$p({\vec t}$) we infer $\neg S({\vec t})$. Eiter {\em et al}.
proposed weak and strong answer sets
for dl-programs \cite{DBLP:journals/ai/EiterILST08}, which were further investigated
from the perspective of loop formulas \cite{Yisong:ICLP:2010} and from the perspective of the logic of here-and-there \cite{DBLP:Fink:JELIA:2010}.


The interest in dl-programs is also due to a technical aspect \-- it has been a challenging task to embed dl-programs into a general nonmonotonic logic.
For example, MKNF \cite{DBLP:conf/ijcai/Lifschitz91} is arguably among the most expressive and versatile formalisms for integrating rules and description logic knowledge bases \cite{Motik:JACM:2010}. Although
Motik and Rosati were able to show a polynomial embedding of a number of other integration formalisms into MKNF, for dl-programs they only
showed that
if a dl-program does not contain the constraint operator $\ominus$, then it can be translated to
a (hybrid) MKNF knowledge base while preserving its strong answer sets.\footnote{The theorem given in \cite{Motik:JACM:2010} (Theorem 7.6) only claims to preserve satisfiability. In a personal communication with Motik, it is confirmed that the proof of the theorem indeed establishes a one-to-one correspondence.}
The embedding into quantified equilibrium logic in \cite{DBLP:Fink:JELIA:2010} is under the assumption that
all dl-atoms containing an occurrence of $\ominus$ are nonmonotonic. They do not deal with the case when a dl-atom involving $\ominus$ may be monotonic.
The embedding into first-order autoepistemic logic (AEL) is under the weak answer set semantics \cite{DBLP:conf/kr/BruijnET08}. For the strong answer set semantics, it is obtained by an embedding of MKNF into first-order autoepistemic logic together with the embedding of dl-programs into MKNF.
Thus it only handles the dl-programs without the constraint operator.

In this paper, we investigate the possibility of embedding dl-programs into default logic \cite{Reiter1980},  under various notions of  answer set semantics.  Our interest in default
logic is due to the fact that it is one of the
dominant nonmonotonic formalisms, yet despite the fact that
default logic naturally accommodates first-order logic and rules
(defaults), curiously it has not been considered explicitly as
a framework for integrating ontology and rules. Since the loose approach can be viewed as query-based, the question arises as whether default logic can be viewed as a foundation for query-based approaches to integration of ontologies and rules.

We shall note that the problem of embedding dl-programs into default logic is nontrivial. In fact, given the difficulties in dealing with dl-programs by other expressive nonmonotonic logics, one can expect great technical subtlety in this endeavor.  Especially, the treatment of equality is a nontrivial issue.

A main technical result of this paper is that  dl-programs
can be translated to default theories while preserving their strong and weak
answer sets.
This is achieved in two steps. In the first, we
investigate the operators in dl-programs and observe that the
constraint operator $\ominus$ is the only one causing a dl-atom to
be nonmonotonic, and a dl-atom may still be monotonic even though it
mentions the constraint operator $\ominus$. To eliminate $\ominus$
from nonmonotonic dl-atoms, we propose a translation $\pi$ and show
that, given a dl-program $\cal K$, the strong and weak answer sets
of $\cal K$ correspond exactly to the strong and weak answer sets of
$\pi(\cal K)$, respectively, i.e., when restricted to the language
of $\cal K$, the strong and weak answer sets of $\pi(\cal K)$ are
precisely those of $\cal K$, and vice versa.
An immediate consequence of this result is that it improves a result of \cite{Motik:JACM:2010}, in that we now know that a much larger class of dl-programs, the class of
{\em normal dl-programs}, can be translated to MKNF knowledge bases,
where a dl-program is normal if
it has  no monotonic dl-atoms that mention the constraint operator $\ominus$.

For the weak answer set semantics, the translation above can be relaxed so that all dl-atoms containing $\ominus$ can be translated uniformly, and the resulting translation is polynomial.
However, for the strong answer set semantics, the above translation relies on the knowledge whether a dl-atom is monotonic or not. In this paper, we present a number of results regarding  the upper and lower bounds of determining this condition for description logics $\cal SHIF$ and $\cal SHOIN$ \cite{DBLP:journals/ai/EiterILST08}. These results have a broader implication as they apply to the work of \cite{DBLP:Fink:JELIA:2010} in embedding dl-programs under strong answer sets into quantified equilibrium logic.

In the second step,
we present two
approaches to translating dl-programs to default theories in a polynomial, faithful, and modular manner \cite{Janhunen:AMAI:1999}.\footnote{This means a polynomial time transformation that preserves the intended semantics, uses the symbols of the original language, and translates
parts (modules) of the given dl-program independently of each other.} The difference between the two is on the handling of inconsistent ontology knowledge bases.
In the first one, an inconsistent ontology knowledge base trivializes the resulting default theory, while following the spirit of dl-programs, in the second approach nontrivial answer sets may still exist in the case of an inconsistent ontology knowledge base.
We show that, for a dl-program ${\cal K}$ without nonmonotonic dl-atoms, there is a one-to-one correspondence between
the strong answer sets of $\cal K$ and the extensions of its corresponding default
theory (whenever the underlying knowledge base is consistent for the first approach). This, along with the result given in the first step, shows that
dl-programs under the strong answer set semantics can be embedded into default logic.

It has been argued that some strong answers may incur self-supports. To overcome this blemish, weakly and strongly well-supported answer set semantics are recently proposed \cite{Yi-Dong:IJCAI:2011}.
Surprisingly, dl-programs under the weakly well-supported semantics
can be embedded into default logic by a small enhancement to our
approach in the second step above, and the resulting  translation is again polynomial, faithful and modular.
Furthermore, if nonmonotonic dl-atoms do not appear in the scope of the default negation $\Not$, the strongly well-supported semantics coincides with the weakly well-supported semantics. Since default negation already provides a language construct to express default inferences, it can be argued that one need not use the constraint operator $\ominus$ inside it. In this sense, our default logic encoding captures the strongly well-supported semantics as well.

We note that, in embedding dl-programs without nonmonotonic dl-atoms into default logic, one still can use the negation-as-failure operator $not$ in dl-programs to express nonmonotonic inferences. The same assumption was adopted in defining a well-founded semantics for dl-programs \cite{Eiter:TOCL2011}. Under this assumption, all the major semantics for dl-programs coincide,
and they all can be embedded into default logic by a polynomial, faithful, and modular translation.
Thus,
the results of this paper not only reveal insights and technical subtleties in capturing dl-programs under various semantics by default logic, but also strengthen the prospect that the latter can serve as a foundation for query-based integration of rules and ontologies.

The main advantage of using default logic to characterize integration of ontology and rules in general, and semantics of dl-programs in particular, is its simple syntax and intuitive semantics,
which has led to a collection of computational results in the literature (see, e.g., \cite{li-you92,Cholewinski:1999,defaultLogicImplementation,YinChenJelia2010}). Interestingly, the more recent effort is on applying ASP techniques to compute default extensions. As long as defaults can be finitely grounded, which is the case for the approach of this paper, these techniques can be extended by combining an ASP-based default logic engine with a
description logic reasoner, with the latter being applied as a black box.
In contrast, the computational issues are completely absent in the approach under AEL \cite{DBLP:conf/kr/BruijnET08}, and only addressed briefly at an abstract level for the approach based on MKNF \cite{Motik:JACM:2010}.
Furthermore, the representation of dl-programs in default logic leads to new insights in computation for dl-programs, one of which is that the iterative construction of default extensions provides a direct support to well-supportedness for answer sets, so that justifications for positive dependencies can be realized for free.

%

The main contributions of this paper are summarized as follows.
\begin{itemize}
\item
We show that dl-programs under the weak and strong answer set semantics can be faithfully and modularly rewritten without constraint operators. The rewriting is polynomial for the weak answer set semantics.
\item
To embed arbitrary dl-programs into default logic, we present faithful and modular \cite{Janhunen:AMAI:1999} translations for the strong answer set semantics, the weak answer set semantics and the weakly well-supported semantics. The translations are also polynomial for the latter two semantics.
\item
For the strong answer set semantics, the embedding depends on the
knowledge of monotonicity of dl-atoms and is polynomial relative
to this knowledge, i.e., if the set of monotonic dl-atoms is known.
In general, determining this set is intractable; as we
show, determining whether a dl-atom is monotonic is $\EXP$-complete
under the description logic $\cal SHIF$ and
$\Pol^{\NEXP}$-complete under the description logic $\cal SHOIN$
(and thus not more expensive than deciding the existence of some strong
or weak answer set of a dl-program under these description logics).
%

\item
For the two semantics for which we do not provide a polynomial embedding, namely the strong answer set semantics and the strongly well-supported semantics,
there are broad classes of dl-programs for which a polynomial embedding can be easily inferred from our results.
For the class of dl-programs where nonmonotonic dl-atoms do not appear in the scope of default negation $\Not$,
our embedding is polynomial, faithful, and modular under the strongly well-supported semantics; and for
the class of dl-programs where the constraint operator does not appear in a positive dl-atom in rules,
our embedding is again polynomial, faithful, and modular
under the strong answer set semantics.
\end{itemize}

The paper is organized as follows. In the next section, we recall
the basic definitions of description logics and dl-programs. In
Section~\ref{Sec:Eliminating}, we present a transformation to
eliminate the constraint operator from nonmonotonic dl-atoms.
In
Section~\ref{Sec:ToDefault},
we give
transformations from dl-programs to default theories, followed by Sections~\ref{Sec:Related-work} and 6
on related work and concluding remarks respectively.

\section{Preliminaries}
In this section, we briefly review the basic notations for
description logics \cite{Badder:Handbook:DL:2007} and description logic programs
\cite{DBLP:journals/ai/EiterILST08}.

\subsection{Description logics}
Description Logics are a family of class-based (concept-based)
knowledge representation formalisms. We assume a set \textbf E of
{\em elementary datatypes} and a set \textbf V of {\em data values}.
A {\em datatype theory} $\textbf D=(\Delta^{\textbf D},\cdot^\textbf
D)$ consists of a {\em datatype} (or {\em concrete}) {\em domain}
$\Delta^\textbf D$ and a mapping $\cdot^\textbf D$ that assigns to
every elementary datatype a subset of $\Delta^\textbf D$ and to
every data value an element of $\Delta^\textbf D$. Let
$\Psi=(\textbf A\cup \textbf R_A\cup \textbf R_D,\textbf I\cup
\textbf V)$ be a vocabulary, where $\textbf A,\textbf R_A,\textbf
R_D$, and $\textbf I$ are pairwise disjoint (denumerable) sets of
{\em atomic concepts, abstract roles, datatype} (or {\em concrete})
roles, and {\em individuals}, respectively. 

A {\em role} is an element of $\bf R_A\cup\bf R_A^{-}\cup\bf R_D$,
where $\textbf R_A^{-}$ means the set of inverses of all
$R\in\textbf R_A$. {\em Concepts} are inductively defined as: (1)
every atomic concept $C\in\textbf A$ is a concept, (2) if
$o_1,o_2,\ldots$ are individuals from $\textbf I$, then
$\{o_1,o_2,\ldots\}$ is a concept (called {\em oneOf}), (3) if $C$
and $D$ are concepts, then also $(C\sqcap D)$, $(C\sqcup D)$, and
$\neg C$ are concepts (called {\em conjunction, disjunction}, and
{\em negation} respectively). (4) if $C$ is a concept, $R$ is an
abstract role from $\textbf R_A\cup\textbf R_A^{-}$, and $n$ is a
nonnegative integer, then $\exists R.C,\forall R.C,\ge nR,$ and
$\leq nR$ are concepts (called {\em exists, value, atleast}, and
{\em atmost restriction}, respectively), (5) if $D$ is a datatype,
$U$ is a datatype role from $\textbf R_D$, and $n$ is a nonnegative
integer, then   $\exists U.D,\forall U.D,\ge nU$, and $\leq nU$ are
concepts (called {\em datatype exists, value, atleast},  and {\em
atmost restriction}, respectively).

An {\em axiom} is an expression of one of the forms: (1)
$C\sqsubseteq D$, called  {\em concept inclusion axiom}, where $C$
and $D$ are concepts; (2) $R\sqsubseteq S$, called {\em role
inclusion axiom}, where either $R,S\in \textbf R_A$ or
$R,S\in\textbf R_D$; (3) Trans($R$), called {\em transitivity
axiom}, where $R\in\textbf R_A$; (4) $C(a)$, called {\em concept
membership axiom}, where $C$ is a concept and $a\in\textbf I$; (5)
$R(a,b)$ (resp., $U(a,v)$), called {\em role membership axiom} where
$R\in\textbf R_A$ (resp., $U\in\textbf R_D$) $a,b\in \textbf I$
(resp., $a\in\textbf I$ and $v$ is a data value), (6) $a\thickapprox
b$ (resp., $a\not\thickapprox b$), called {\em equality} (resp.,
{\em inequality}) {\em axiom}  where $a,b\in\textbf I$.

A {\em description logic (DL) knowledge base} $O$ is a finite set of
axioms. The $\mathcal{SHOIN}(\textbf D)$ {\em knowledge base}
consists of a finite set of above axioms, while the
$\mathcal{SHIF}(\textbf D)$ {\em knowledge base} is the one of
$\mathcal{SHOIN}(\textbf D)$, but without the \textit{oneOf}
constructor and with the \textit{atleast} and \textit{atmost}
constructors limited to 0 and 1.

The semantics of the two description logics are defined in terms of
general first-order interpretations. An {\em interpretation}
$\mathcal I=(\Delta^\mathcal I,\cdot^\mathcal I)$ with respect to a
datatype theory $\textbf D=(\Delta^\textbf D,\cdot^\textbf D)$
consists of a nonempty (abstract) {\em domain} $\Delta^\mathcal I$
disjoint from $\Delta^\textbf D$, and a mapping $\cdot^\mathcal I$
that assigns to each atomic concept $C\in\textbf A$ a subset of
$\Delta^\mathcal I$, to each individual $o\in\textbf I$ an element
of $\Delta^\mathcal I$, to each abstract role $R\in\textbf R_A$ a
subset of $\Delta^\mathcal I\times\Delta^\mathcal I$, and to each
datatype role $U\in\textbf R_D$ a subset of $\Delta^\mathcal
I\times\Delta^\textbf D$. The mapping $\cdot^\mathcal I$ is extended
to all concepts and roles as usual (where $\#S$ denotes the
cardinality of a set $S$):
\begin{itemize}
  \item $(R^{-})^\mathcal I=\{(a,b)|(b,a)\in R^\mathcal I\}$;
  \item $\{o_1,\ldots,o_n\}^\mathcal I=\{o_1^\mathcal I,\ldots,o_n^\mathcal I\}$;
  \item $(C\sqcap D)^\mathcal I=C^\mathcal I\cap D^\mathcal I$, $(C\sqcup D)^\mathcal I=C^\mathcal I\cup D^\mathcal I$,
        $(\neg C)^\mathcal I=\Delta^\mathcal I\setminus C^\mathcal I$;
  \item $(\exists R. C)^\mathcal I=\{x\in\Delta^\mathcal I|\exists y: (x,y)\in R^\mathcal I\wedge y\in  C^\mathcal I\}$;
  \item $(\forall R.C)^\mathcal I=\{x\in\Delta^\mathcal I|\forall y: (x,y)\in R^\mathcal I\rto y\in C^\mathcal I\}$;
  \item $(\ge nR)^\mathcal  I=\{x\in\Delta^\mathcal I|\#(\{y|(x,y)\in R^\mathcal I\})\ge n\}$;
  \item $(\le nR)^\mathcal  I=\{x\in\Delta^\mathcal I|\#(\{y|(x,y)\in R^\mathcal I\})\le n\}$;
  \item $(\exists U.D)^\mathcal I=\{x\in\Delta^\mathcal I|\exists y:(x,y)\in U^\mathcal I\wedge y\in D^\textbf D\}$;
  \item $(\forall U.D)^\mathcal I=\{x\in\Delta^\mathcal I|\forall y:(x,y)\in U^\mathcal I\rto y\in D^\textbf D\}$;
  \item $(\ge nU)^\mathcal  I=\{x\in\Delta^\mathcal I|\#(\{y|(x,y)\in U^\mathcal I\})\ge n\}$;
  \item $(\le nU)^\mathcal  I=\{x\in\Delta^\mathcal I|\#(\{y|(x,y)\in U^\mathcal I\})\le n\}$.
\end{itemize}

Let $\mathcal I=(\Delta^\mathcal I,\cdot^\mathcal I)$ be an
interpretation respect to $\textbf D=(\Delta^\textbf D,\cdot^\textbf
D)$, and $F$ an axiom. We say that $\mathcal I$ {\em satisfies} $F$, written
$\mathcal I\models F$, is defined as follows: (1) $\mathcal I\models
C\sqsubseteq D$ iff $C^\mathcal I\subseteq D^\mathcal I$; (2)
$\mathcal I\models R\sqsubseteq S$ iff $R^\mathcal I\subseteq
S^\mathcal I$; (3) $\mathcal I\models \textmd{Trans}(R)$ iff
$R^\mathcal I$ is transitive; (4) $\mathcal I\models C(a)$ iff
$a^\mathcal I\in C^\mathcal I$; (5) $\mathcal I\models R(a,b)$
$(a^\mathcal I,b^\mathcal I)\in R^\mathcal I$ (resp., $\mathcal
I\models U(a,v)$ iff $(a^\mathcal I,v^\textbf D)\in U^\mathcal I)$;
(6) $\mathcal I\models a\thickapprox b$ iff $a^\mathcal I=b^\mathcal
I$ (resp., $\mathcal I\models a\not\thickapprox b$ iff $a^\mathcal
I\neq b^\mathcal I$).
$\mathcal I$ {\em satisfies} a DL knowledge base $O$, written
$\mathcal I\models O$, if $\mathcal I\models F$ for any $F\in O$. In
this case, we call $\mathcal I$ a {\em model} of $O$. An axiom $F$
is a {\em logical consequence} of a DL knowledge base $O$, written
$O\models F$, if any model of $O$ is also a model of $F$.

\subsection{Description logic programs}
Let $\Phi=(\mathcal {P,C})$ be a first-order vocabulary with
nonempty finite sets $\cal C$ and $\cal P$ of constant symbols and
predicate symbols respectively such that $\cal P$ is disjoint from
${\bf A\cup R}$ and $\cal C\subseteq \bf I$. {\em Atoms} are formed
from the symbols in $\cal P$ and $\cal C$ as usual.

A {\em dl-atom} is an expression of the form
\begin{equation}\label{dl:atom}
  \DL[S_1\ op_1\ p_1, \ldots,S_m\ op_m\ p_m;Q](\vec t),\ \  (m\ge 0)
\end{equation}
where
\begin{itemize}
  \item each $S_i$ is either a concept, a role or its negation,\footnote{
  We allow negation of a role for convenience, so that
  we can replace ``$S\odot p$" with an equivalent form ``$\neg S\oplus p$" in dl-atoms. The negation of a role is not explicitly present in
\cite{DBLP:journals/ai/EiterILST08}. As discussed there, negative
role assertions can be emulated in ${\cal SHIF}$ and ${\cal SHOIN}$
(and in fact also in  ${\cal ALC}$). } or a special
  symbol in $\{\thickapprox,\not\thickapprox\}$;
  \item $op_i\in\{\oplus,\odot,\ominus\}$ (we call $\ominus$ the {\em constraint operator});
  \item $p_i$ is a unary predicate symbol in $\cal P$ if $S_i$ is a concept, and a binary predicate symbol in $\cal P$
   otherwise. The $p_i$'s are called {\em input predicate symbols};
  \item $Q(\vec t)$ is a {\em dl-query}, i.e., either (1) $C(t)$ where $\vec t=t$;
    (2) $C\sqsubseteq D$ where $\vec t$ is an empty  argument list;
    (3) $R(t_1,t_2)$ where $\vec t=(t_1,t_2)$;
    (4) $t_1\thickapprox t_2$ where $\vec t=(t_1,t_2)$;
    or their negations, where $C$ and $D$ are concepts, $R$ is a role, and $\vec t$
  is a tuple  of constants.
\end{itemize}

The precise meanings of $\{\oplus, \odot, \ominus\}$ will be defined
shortly. Intuitively, $S \oplus p$ extends $S$ by the extension of
$p$. Similarly, $S \odot p$ extends $\neg S$ by the extension of
$p$, and $S \ominus p$ constrains $S$ to $p$.
%
%
A {\em dl-rule} (or simply a {\em rule}) is an expression of the
form
\begin{equation}\label{dl:rule}
  A\lto B_1,\ldots,B_m,\Not B_{m+1},\ldots,\Not B_n, (n\ge m\ge 0)
\end{equation}
where $A$ is an atom, each $B_i ~(1\leq i\leq n)$ is  an
atom\footnote{Different from that of
\cite{DBLP:journals/ai/EiterILST08}, we consider ground atoms
instead of literals for convenience.} or a dl-atom. We refer to $A$
as its {\em head}, while the conjunction of $B_i~(1\leq i\leq m)$ and
$\Not B_j~(m+1\leq j\leq n)$ is its {\em body}. For convenience, we
abbreviate a rule in the form (\ref{dl:rule}) as
\begin{equation}\label{dl:rule:set}
  A\lto \Pos,\Not \Neg
\end{equation}
where $\Pos=\{B_1,\ldots,B_m\}$ and $\Neg=\{B_{m+1},\ldots,B_n\}$.
Let $r$ be a rule of the form (\ref{dl:rule:set}). If
$\Neg=\emptyset$ and $\Pos=\emptyset$, $r$ is a {\em fact} and we may
write it as ``$A$" instead of ``$A\lto$". A {\em description logic
program} ({\em dl-program}) $\mathcal K=(O,P)$ consists of a DL
knowledge base $O$ and a finite set $P$ of dl-rules. In what follows
we assume the vocabulary of $P$ is implicitly given by the constant
symbols and predicate symbols occurring in $P$,
$\cal C$ consists of the constants occurring in atoms of $P$, and $P$ is grounded (no atoms containing variables) unless stated
otherwise.

Given a dl-program $\mathcal K=(O,P)$, the {\em Herbrand base} of
$P$, denoted by $\textit{HB}_P$, is the set of atoms occurring in
$P$ and the ones formed from the predicate symbols of $\cal P$
occurring in some dl-atoms of $P$ and the constant symbols in
$\mathcal C$.
\footnote{Note that this slightly deviates from the usual convention
of the Herbrand base; ground atoms that are not in the Herbrand base
as considered here are always false in answer sets.}
It is clear that $\HB_P$ is in polynomial size of
$\cal K$. An {\em interpretation} $I$ (relative to $P$) is a subset
of $\HB_P$. Such an $I$ is a {\em model} of an atom or dl-atom $A$
under $O$, written $I\models_OA$, if the following holds:
\begin{itemize}
  \item if $A\in\HB_P$, then $I\models_OA$ iff $A\in I$;
  \item if $A$ is a dl-atom $\DL(\lambda;Q)(\vec t)$ of the form (\ref{dl:atom}), then $I\models_OA$ iff $
    O(I;\lambda)\models Q(\vec t)$ where
    $O(I;\lambda)=O\cup\bigcup_{i=1}^{m}A_i(I)$\mbox{ and, for $1\leq i\leq m$,}
    \[A_i(I)=\left\{
    \begin{array}{ll}
        \{S_i(\vec e)\mid p_i(\vec e)\in I\}, & \hbox{if $op_i=\oplus$;} \\
        \{\neg S_i(\vec e)\mid p_i(\vec e)\in I\}, & \hbox{if $op_i=\odot$;} \\
        \{\neg S_i(\vec e)\mid p_i(\vec e)\notin I\}, & \hbox{if $op_i=\ominus$;}
    \end{array}
    \right.
    \]
\end{itemize}
where $\vec e$ is a tuple of constants over $\mathcal C$.
As we allow negation of role,
$S\odot p$ can be replaced with $\neg S\oplus p$ in any dl-atom. In addition, we can shorten $S_1\ op\ p, \ldots, S_k\ op\ p$
as $(S_1\sqcup\ldots\sqcup S_k)\ op\ p$ where $S_i\ op\ p$ appears in $\lambda$ for all $i~(1\le i\le k)$
and $op\in\{\oplus,\odot, \ominus\}$. Thus dl-atoms can be equivalently rewritten into ones without using the operator $\odot$, and every predicate $p$ appears at most once for each operator $\oplus$ and $\ominus$. For instance,
the dl-atom $\DL[S_1\oplus p, S_2\oplus p, S_1\ominus p, S_2\ominus p,Q](\vec t)$ can be equivalently written as
$\DL[(S_1\sqcup S_2) \oplus p, (S_1\sqcup S_2)\ominus p,Q](\vec t)$.

An interpretation $I\subseteq {\HB_P}$ is a {\em model}
of ``$\Not A$", written $I\models_O\Not A$, if $I$ is not a model of $A$, i.e., $I\not\models_OA$. The
interpretation $I$ is a {\em model} of a dl-rule of the form
(\ref{dl:rule:set}) iff $I\models_O B$ for any $B\in\Pos$ and
$I\not\models_OB'$ for any $B'\in\Neg$ implies that $I\models_OA$. An interpretation $I$ is
a {\em model} of a dl-program $\mathcal K=(O,P)$, written
$I\models_O\mathcal K$, iff $I$ is a model of each rule of $P$. 

\subsubsection{Monotonic dl-atoms}

A dl-atom $A$ is {\em monotonic} (relative to a dl-program $\mathcal
K=(O,P)$) if $I\models_ OA$ implies $I'\models_ OA$, for all $I'$ such that
$I\subseteq I'\subseteq \textit{HB}_P$, otherwise $A$ is {\em
nonmonotonic}. It is
clear that if a dl-atom does not mention the constraint operator then
it is monotonic. However, a dl-atom may be monotonic even if it
mentions the constraint operator. For example, the dl-atom $\DL[S\odot
p,S\ominus p;\neg S](a)$  is a tautology (which is monotonic).
%

Evidently, the constraint operator is the only one that may cause a
dl-atom to be nonmonotonic.
This sufficient condition for monotonicity can be efficiently checked; for the case
where the constraint operator may appear, the following generic upper bound on
complexity is easily derived. We refer to the \emph{query complexity}
of a ground dl-atom $A$ of form (\ref{dl:atom}) in $\cal K$ as the complexity of deciding, given
$K=(O,P)$, $A$, and an arbitrary interpretation $I$, whether $O(I;\lambda)\models
A$ holds.

\begin{proposition}
\label{prop:mon-dl-complexity}
Let ${\cal K}=(O,P)$ be a (ground) dl-program, and
$A$ be a dl-atom occurring in $P$ which has query
complexity in class $C$. Then
deciding whether $A$ is monotonic is in $\coNP^{C}$.
\end{proposition}
\begin{proof}
Indeed, to show that $A$ of form (\ref{dl:atom}) is nonmonotonic, one can guess
restrictions $I_A$ and $I'_A$ of interpretations  $I$ and $I'$, respectively, to the predicates occurring
in $A$ such that $I_A \subseteq I'_A$ and $I_A\models_O A$ but $I'_A
\not\models_O A$
(clearly, $J\models_O A$ iff $J_A\models_O A$ for arbitrary
interpretations $J$). The guess for $I_A$ and
$I'_A$ is of polynomial size in the size of $\cal K$ (assuming that the set of
constants $\cal C$ is explicit in $\cal K$, or is constructible in  polynomial
time), and preparing $O(I_A;\lambda)$ and
$O(I'_A;\lambda)$ is feasible in polynomial time (in fact, easily in
logarithmic space). Using the oracle, we can decide
$O(I_A;\lambda)\models  Q(\vec t)$ and
$O(I'_A;\lambda) \models  Q(\vec t)$, and thus  $I_A\models_O A$ but $I'_A
\not\models_O A$. Overall, the complexity is in $\coNP^{C}$.
\end{proof}

Depending on the underlying description logic, this upper bound might
be lower or complemented by a matching hardness result. In
fact, for ${\cal SHIF}$ and ${\cal SHOIN}$, the latter turns out to be
the
case. DL-atoms over these description logics have a query complexity that is
complete for $C$ = $\EXP$ and $C$ = $\NEXP$, respectively.
By employing well-known identities of complexity classes, we obtain
the following result.

\begin{theorem}
\label{theo:mon-dl-complexity}
Given a (ground) dl-program ${\cal K}=(O,P)$ and
a dl-atom $A$ occurring in $P$,
deciding whether $A$ is monotonic is
(i) $\EXP$-complete, if $O$ is
a ${\cal SHIF}$ knowledge base and (ii) $\Pol^{\NEXP}$-complete,
if  $O$ is a ${\cal SHOIN}$ knowledge base.
\end{theorem}
\begin{proof} The membership part for (i) ${\cal SHIF}$ follows easily from
Proposition~\ref{prop:mon-dl-complexity} and the fact that ${\cal
SHIF}$ has query complexity in $\EXP$; indeed, each
dl-query evaluation $O(I;\lambda)\models_O$ can be transformed in
polynomial time to deciding satisfiability of a  ${\cal
SHIF}$ knowledge base, which is $\EXP$-complete in general
\cite{tobies,Horrock:ISWC:2003}. Now $\coNP^{\EXP}$ = $\EXP$ = $\NP^{\EXP}$;
indeed, the computation tree of a nondeterministic Turing machine
with polynomial running time and $\EXP$ oracle access has single
exponential (in the input size) many nodes, which can be traversed in exponential
time; simulating an oracle call in a node is possible in exponential
time in the size of the (original) input. Overall, this yields an
exponential time upper bound.

The membership part for (ii)  ${\cal SHOIN}$ follows analogously from
Proposition~\ref{prop:mon-dl-complexity} and the fact that ${\cal
SHOIN}$ has $\coNEXP$-complete query complexity, which follows
from $\NEXP$-completeness of the knowledge base satisfiability problem in ${\cal
SHOIN}$ (for both unary and binary number encoding; see
\cite{Horrock:ISWC:2003,hartmann05complexity}). Now
$\coNP^{\coNEXP}$ = $\coNP^{\NEXP}$ = $\Pol^{\NEXP}$ (= $\NP^{\NEXP}$); here the
second equality holds by results in \cite{DBLP:Hemachandra:jcss:1989}.

The hardness parts for (i) and (ii) are shown by reductions of
suitable $\EXP$- resp.\ $\Pol^{\NEXP}$-complete
problems, building on constructions in
\cite{DBLP:journals/ai/EiterILST08} (see Appendix~\ref{app:monotonic}).
\end{proof}

For convenience, we use $\DL_P$ to denote the set of all dl-atoms that
occur in $P$, $\DL_P^{+}\subseteq \DL_P$ to denote the set of
monotonic dl-atoms, and $\DL_P^{?}=\DL_P\setminus \DL_P^{+}$.   Note that this is different from that of \cite{DBLP:journals/ai/EiterILST08} where $\DL_P^+$ is assumed to be a set of ground dl-atoms in $\DL_P$ which are known to be monotonic, while $\DL_P^?$ denotes the set of remaining dl-atoms. Thus $\DL_P^?$ is allowed to contain monotonic dl-atoms as well in \cite{DBLP:journals/ai/EiterILST08}.   Our definition represents the ideal situation where monotonicity can be finitely verified, which is the case for decidable description logic knowledge bases.
Note also
that by Theorem~\ref{theo:mon-dl-complexity}, for ${\cal SHIF}$ and
${\cal SHOIN}$ knowledge bases computing $\DL_P^{+}$ is possible
with no respectively mild complexity increase compared to basic reasoning tasks
in the underlying description logic.

\subsubsection{Some classes of dl-programs}

A dl-program $\mathcal K=(O,P)$ is {\em positive}, if (i) $P$ is
``not"-free, and (ii) every dl-atom is monotonic relative to
$\mathcal K$.
Positive dl-programs have attractive semantics
properties; e.g., it is evident that a positive dl-program $\mathcal K$ has a (set inclusion) least model.

From the results above, we easily obtain the following
results on recognizing positive dl-programs.

\begin{proposition}
\label{prop:positive-dl-complexity}
Deciding whether a given (not necessarily ground) dl-program ${\cal K}=(O,P)$
is positive is in co-NP$^{C}$, if every dl-atom in the ground version
of $P$ has query complexity in $C$.
\end{proposition}
\begin{proof}
${\cal K}$ is not positive if either (i) $P$ is not ``not''-free, which
can be  checked in polynomial time, or (ii) some dl-atom $A$ in the
ground version of $P$ is nonmonotonic; such an $A$ can be guessed
and verified, by the hypothesis, in polynomial time with an oracle
for $C$; hence the result.
\end{proof}

\begin{theorem}
\label{theo:positive-dl-complexity}
Deciding whether a given (not necessarily ground) dl-program
${\cal K}=(O,P)$ is positive is (i) $\EXP$-complete, if $O$
is a ${\cal SHIF}$ knowledge base and (ii)
$\Pol^{\NEXP}$-complete, if $O$ is a ${\cal SHOIN}$
knowledge base.
\end{theorem}
\begin{proof}
The membership parts are immediate from
Proposition~\ref{prop:positive-dl-complexity},  and the hardness parts
from the hardness proofs in Theorem~\ref{theo:mon-dl-complexity}:
the atom $A$ is monotonic relative to the constructed dl-program $\cal
K$ iff $\cal K$ is positive.
\end{proof}

Thus, the test whether a dl-program is positive (and similarly,
whether all dl-atoms in it are monotonic) for ${\cal SHIF}$ and
${\cal SHOIN}$ knowledge bases is also not expensive compared to basic
reasoning tasks.

Besides positive dl-programs, another important subclass are {\em canonical
dl-programs}, where a dl-program ${\cal K}=(O,P)$ is {\em canonical},
if $P$ mentions no constraint operator. Clearly,
canonical dl-programs are easy to recognize. The same holds for the
more general class of {\em normal dl-programs}, where a dl-program
${\cal K}=(O,P)$ is {\em normal}, if no monotonic dl-atom occurs in
$P$ that mentions the constraint operator. Note that normal
dl-programs are not positive in general; since monotonic dl-atoms
mentioning the constraint operator are rather exceptional, the
normal dl-programs include most dl-programs relevant for practical applications.

\begin{example}
\label{exam:dl:program:1}
 Consider the following dl-programs, which we will refer to repeatedly
in the sequel.
\begin{itemize}
\item
$\mathcal K_1=(O_1,P_1)$ where $O_1=\{S\sqsubseteq S'\}$ and
$P_1=\{p(a)\lto \DL[S\oplus p;S'](a)\}$. The single dl-atom in $P_1$ has no
constraint operator, and thus ${\cal K}_1$ is canonical (hence also normal); moreover, since
`$\Not$'' does not occur in $P_1$, ${\cal K}_1$ is also positive.

\item
$\mathcal K_2=(O_2,P_2)$ where $O_2=\emptyset$ and $P_2=\{p(a)\lto
\DL[S\oplus p, S'\ominus q;S\sqcap\neg S'](a)\}$. Here,
the constraint operator occurs in $P_2$, thus ${\cal K}_2$ is not
canonical. Furthermore, the single dl-atom in $P_2$ is nonmonotonic,
hence ${\cal K}_2$ is also not positive. However, ${\cal K}_2$ is normal.
%
\end{itemize}
\end{example}

\subsubsection{Strong and weak answer sets}

Let ${\cal K}=(O,P) $ be a positive dl-program. The immediate consequence operator $\gamma_{\cal K}:2^{\HB_P}\rto 2^{\HB_P}$
is defined as, for any $I\subseteq\HB_P$,
\begin{align*}
  \gamma_\mathcal K(I)=\{h\mid h\lto\Pos \in P\mbox{ and }I\models_OA\mbox{ for any }A\in\Pos\}.
\end{align*}
Since $\gamma_\mathcal K$ is monotonic, the least
fix-point of $\gamma_{\cal K}$ always exists which is the least model of $\mathcal K$. By $\lfp(\gamma_{\cal K})$
we denote the least fix-point of $\gamma_{\cal K}$, which can be iteratively constructed as below:
\begin{itemize}
  \item $\gamma_\mathcal K^0=\emptyset$;
  \item $\gamma_\mathcal K^{n+1}=\gamma_\mathcal K(\gamma_\mathcal K^n)$.
\end{itemize}
It is clear that the least fixpoint $\lfp(\gamma_\mathcal
K)=\gamma_\mathcal K^\infty$.

We are now in the position to recall the semantics of dl-programs.
Let $\mathcal K=(O,P)$ be a dl-program. The {\em strong
dl-transform} of $\mathcal K$ relative to $O$ and an interpretation
$I\subseteq \textit{HB}_P$, denoted by $\mathcal K^{s,I}$, is the
positive dl-program $(O,sP^I_O$), where $sP^I_O$ is obtained from
$P$ by deleting:
\begin{itemize}
  \item the dl-rule $r$ of the form (\ref{dl:rule}) such that either
    $I\not\models_OB_i$ for some $1\leq i\leq m$ and
    $B_i\in DL_P^?$, or $I\models_OB_j$ for some $m+1\leq j\leq n$; and
  \item the nonmonotonic dl-atoms and $\Not A$ from the remaining dl-rules where $A$ is an
  atom or a dl-atom.
\end{itemize}
The interpretation $I$ is a {\em strong answer set} of $\mathcal K$
if it is the least model of $\mathcal K^{s,I}$, i.e., $I=\lfp(\gamma_{{\cal K}^{s,I}})$.\footnote{Note that,
under our notion of $\DL_P^{?}$, namely $\DL_P^{?}$ is the set of nonmonotonic dl-atoms w.r.t.\! a given dl-program, the strong answer set semantics is the strongest among possible variations under the definition of \cite{DBLP:journals/ai/EiterILST08}, where $\DL_P^{?}$ may contain monotonic dl-atoms,
in that given a dl-program ${\cal K}$, any strong answer set of ${\cal K}$
under our definition
is a strong answer set of ${\cal K}$ under the definition of \cite{DBLP:journals/ai/EiterILST08}.}

The {\em weak dl-transform} of $\mathcal K$ relative to $O$ and an
interpretation $I\subseteq \HB_P$, denoted by $\mathcal K^{w,I}$, is
the positive dl-program $(O,wP^I_O)$, where $wP_O^I$ is obtained
from $P$ by deleting:
\begin{itemize}
  \item the dl-rules of the form (\ref{dl:rule}) such that either
    $I\not\models_OB_i$ for some $1\leq i\leq m$ and
    $B_i\in \DL_P$, or $I\models_OB_j$ for some $m+1\leq j\leq n$; and
  \item the dl-atoms and $\Not A$ from the remaining dl-rules where
  $A$ is an atom or dl-atom.
\end{itemize}
The interpretation $I$ is a {\em weak answer set} of $\mathcal K$ if
$I$ is the least model of $\mathcal K^{w,I}$, i.e., $I=\lfp(\gamma_{{\cal K}^{w,I}})$.

The following proposition shows that, given a dl-program ${\cal
K}=(O,P)$, if $O$ is inconsistent then strong and weak answer sets
of $\cal K$ coincide, and are minimal.
\begin{proposition}\label{prop:2.3.1}
  Let ${\cal K}=(O,P)$ be a dl-program where $O$ is inconsistent and
  $I\subseteq\HB_P$.  Then
  \begin{enumerate}[(i)]
  \item $I$ is a strong answer set of $\cal K$ if and only if $I$ is a weak answer
  set of $\cal K$.
  \item The strong and weak answer sets of $\cal K$ are minimal under set inclusion.
  \end{enumerate}
\end{proposition}
\begin{proof}
  By the inconsistency of $O$, it is clear that every dl-atom $A$ occurring in $P$ is monotonic and $M\models_OA$ for any $M\subseteq\HB_P$.

  (i) Note that the only difference between $sP_O^I$ and $wP_O^I$ is
  that there exist some dl-atoms in $sP_O^I$ but not in $wP_O^I$,
  i.e., for any dl-rule $r=(h\lto \Pos,\Not\Neg)$ in $P$,
  $(h\lto\Pos)$ belongs to $sP_O^I$ if and only if $(h\lto\Pos')$ belongs to
  $wP_O^I$ where $\Pos'=\{h\in\HB_P\mid h\in\Pos\}$. However note that
  $\emptyset\models_OA$ for any dl-atom $A\in\Pos\setminus\Pos'$. It
  follows that $\lfp(\gamma_{{\cal K}^{s,I}})=\lfp(\gamma_{{\cal
  K}^{w,I}})$. Consequently
  $I$ is a strong answer set of $\cal K$ if and only if $I$
  is a weak answer set of $\cal K$.

  (ii) By Theorem 4.13 of \cite{DBLP:journals/ai/EiterILST08}, the
  strong answer sets of $\cal K$ are minimal. It implies that the
  weak answer sets of $\cal K$ are minimal as well by (i) of the proposition.
\end{proof}

\begin{example}[Continued from Example~\ref{exam:dl:program:1}]
\label{exam:dl:program:1a}
Reconsider the dl-programs in Example~\ref{exam:dl:program:1}.
\begin{itemize}
\item
The dl-program $\mathcal K_1=(O_1,P_1)$, where $O_1=\{S\sqsubseteq S'\}$ and
$P_1=\{p(a)\lto \DL[S\oplus p;S'](a)\}$, has a unique
strong answer set $I_1 = \emptyset$ and two weak answer sets $I_1$
and $I_2 = \{p(a)\}$. The interested reader may verify the
following: $O_1(I_2;S\oplus p)=O_1\cup\{S(a)\}$, and clearly $O_1\not\models S'(a)$ and
$\{S(a),S\sqsubseteq S'\}\models S'(a)$. So the weak dl-transformation
relative to $O_1$ and $I_2$ is ${\cal K}_1^{w, I_2}=(O_1,\{p(a)
\leftarrow \})$. Since $I_2$ coincides with the least model of
$\{p(a) \leftarrow \}$, it is a weak answer set of $\mathcal K_1$.
Similarly, one can verify that the strong dl-transformation relative to
$O_1$ and $I_2$ is ${\cal K}_1^{s,I_2}={\cal K}_1$. Its least model is
the empty set, so $I_2$ is not a strong answer set of $\mathcal K_1$.

\item
For the dl-program  $\mathcal K_2=(O_2,P_2)$, where $O_2=\emptyset$ and
$P_2=\{p(a)\lto \DL[S\oplus p, S'\ominus q;S\sqcap\neg S'](a)\}$,
both $\emptyset$ and $\{p(a)\}$ are strong and weak answer sets.
%
\end{itemize}
\end{example}

These dl-programs show that strong (and weak) answer sets may not be
(set inclusion) minimal. It has been shown that if a dl-program
contains no nonmonotonic dl-atoms then its strong answer sets are
minimal (cf. Theorem 4.13 of
\cite{DBLP:journals/ai/EiterILST08}). However, this does not hold
for weak answer sets as shown by the dl-program
$\mathcal K_1$ above, even if it is positive. 
It has also been shown that strong answer sets are always weak
answer sets, but not vice versa.  Thus the question rises: is it the case that, for any dl-program $\mathcal K$ and interpretation $I$, if $I$ is a
weak answer set of $\mathcal K$, then there is
$I' \subseteq I$ such that $I'$ is a strong answer of $\mathcal K$? We give a negative answer to this
question by the following example.
\begin{example}\label{exam:2}
  Let $\mathcal K=(\emptyset,P)$ where $P$ consists of
  \begin{align*}
    p(a)  \lto \DL[S\oplus p;S](a), \hspace{.6cm}     p(a)  \lto \Not \DL[S\oplus p;S](a).
  \end{align*}
 Note that $\cal K$ is canonical and normal, but not positive.
 Intuitively, $P$ expresses reasoning by cases: regardless of
  whether the dl-atom $A = \DL[S\oplus p;S](a)$ evaluates to false,
  $p(a)$ should be true.
  Let $I=\{p(a)\}$. We have that $wP_O^I=\{p(a)\lto\}$, thus $I$ is a weak answer set of $\mathcal K$.
  However, note that $sP^I_O=\{p(a)\lto \DL[S\oplus p;S](a)\}$.
  The least model of $\mathcal K^{s,I}$ is $\emptyset~(\neq I)$. So that $I$ is not a
  strong answer set of $\mathcal K$. Now consider $I'=\emptyset$. We have $sP_O^{I'}=\{p(a)  \lto \DL[S\oplus p;S](a),~~ p(a)\lto\}$.
  The least model of $\mathcal K^{s,I'}$ is $\{p(a)\}~(\neq I')$.
  Thus $I'$ is not a strong
  answer set of $\mathcal K$.
  In fact, $\mathcal K$ has no strong answer sets at all.
 This is in line with the intuition that, as
  $O=\emptyset$ is empty, $p(a)$ can not be foundedly derived without
  the assumption that $p(a)$ is true.
\end{example}

\section{Eliminating the Constraint Operator from Nonmonotonic Dl-atoms}
\label{Sec:Eliminating}
Intuitively, translating a nonmonotonic dl-atom into a monotonic is to replace $S\ominus p$ with $S\odot p'$ where
$p'$ is a fresh predicate having the same arity as $p$ and $p'$ stands for the negation of $p$.
In what follows, we show that the constraint operator can be eliminated from nonmonotonic dl-atoms while preserving both weak and strong answer sets.
As mentioned previously, we assume that the signatures $\cal P$ and $\cal C$ are implicitly given for a given dl-program $\cal K$.
Any predicate symbol not occurring in $\cal K$ is a fresh one.

\begin{definition}[$\pi({\cal K})$]
Let ${\cal K}=(O,P)$ be a dl-program. We define $\pi({\cal K})=(O,\pi(P))$ where $\pi(P)=\bigcup_{r\in P}\pi(r)$ and $\pi(r)$, assuming $r$ is of the form
(\ref{dl:rule}), consists of
\begin{enumerate} [(i)]
  \item the rule
  \begin{equation}\label{trans:pi:1}
    A\lto \pi(B_1),\ldots,\pi(B_m), \pi(\Not B_{m+1}),\ldots,\pi(\Not B_n)
  \end{equation}
  where \[\pi(B)=
  \left\{
    \begin{array}{ll}
      B, & \hbox{if $B$ is an atom or a monotonic dl-atom;} \\
      \Not \pi_B, & \hbox{if $B$ is a nonmonotonic dl-atom,} \\
    \end{array}
  \right. \]
 in which $\pi_B$ is a fresh propositional atom, and
  \[\pi(\Not B)=
  \left\{
    \begin{array}{ll}
      \Not B, & \hbox{if $B$ is an atom;} \\
      \Not \DL[\pi(\lambda);Q](\vec t), & \hbox{if $B=\DL[\lambda,Q](\vec t)$,}
    \end{array}
  \right. \]
   where $\pi(\lambda)$ is obtained from $\lambda$ by replacing each ``$S\ominus p$" with ``$S\odot \pi_p$", and $\pi_p$ is a fresh predicate having the same arity as $p$;

  \item for each nonmonotonic dl-atom $B\in\{B_1,\ldots,B_m\}$, the following rule:
  \begin{equation} \label{trans:pi:2}
    \pi_B\lto \pi(\Not B)
  \end{equation}
  where $\pi_B$ is the same atom as mentioned in (i) and

  \item for each predicate $p$ such that ``$S\ominus p$" occurs in some nonmonotonic dl-atom of $r$, the instantiations of the rule:
   \begin{equation}\label{trans:pi:3}
     \pi_p(\vec x)\lto \Not p(\vec x)
   \end{equation}\underline{}
   where $\vec x$ is a tuple of distinct variables matching the arity of $p$, and $\pi_p$ is the same predicate as mentioned in (i).
\end{enumerate}
\end{definition}
Intuitively, the idea in $\pi$ is the following. Recall that ``$S\ominus p$" means ``infer  $\neg S(\vec
c)$ in absence of  $p(\vec c)$". Thus if $\pi_p(\vec c)$ stands for the
absence of $p(\vec c)$ then ``$S\ominus p$" should
have the same meaning as that of ``$S\odot \pi_p$". Thus, a
nonmonotonic dl-atom can be re-expressed by a monotonic dl-atom and ``\textit{not}".
Note that $\pi(P)$ may still contain dl-atoms with the constraint
operator, but they are all monotonic dl-atoms.
\begin{example}\label{exam:dl-program:3}
 Let us consider the following dl-programs.
  \begin{itemize}
    \item
        Let $\mathcal K_1=(\emptyset,P_1)$ where $P_1$ consists of
        \[p(a)\lto\Not \DL[S\ominus p;\neg S](a).\]
 Note that ${\cal K}_1$ is normal but neither canonical nor positive.
        It is not difficult to verify that $\mathcal K_1$ has two weak answer sets $\emptyset$ and $\{p(a)\}$.
        They are strong answer sets of $\mathcal K_1$ as well.
        According to the translation $\pi$, we have $\pi(\mathcal K_1)=  (\emptyset,\pi(P_1))$, where $\pi(P_1)$
        consists of
        \begin{align*}
         p(a)   \lto  \Not \DL[S\odot \pi_p; \neg S](a), \hspace{1cm}             \pi_p(a)  \lto  \Not p(a).
        \end{align*}
        It is easy to see that $\pi(\mathcal K_1)$ has only two weak answer sets, $\{p(a)\}$ and $\{\pi_p(a)\}$, which
        are also strong answer sets of $\pi(\mathcal K_1)$. They correspond to $\{p(a)\}$
        and $\emptyset$ respectively   when restricted to $\HB_{P_1}$.


    \item Let $\mathcal K_2=(\emptyset,P_2)$ where $P_2$ consists of
        \[p(a)\lto\Not \DL[S\ominus p, S'\odot q,S'\ominus q;\neg S\sqcap \neg S'](a).\]
      Recall that the dl-atom $\DL[S'\odot q,S'\ominus q;\neg S](a)$
      is a tautology, hence monotonic; thus ${\cal K}_2$ is not normal.
      The strong and weak answer sets of ${\cal K}_2$ are the same as those of ${\cal K}_1$. Please note that
      $\pi(P_2)$ consists of
       \begin{align*}
         & p(a)   \lto  \Not \DL[S\odot \pi_p,S'\odot q, S'\odot \pi_q; \neg S\sqcap\neg S'](a),\\
         & \pi_p(a)  \lto  \Not p(a), \hspace{2cm} \pi_q(a)\lto\Not q(a).
        \end{align*}
        The strong and weak answer sets of $\pi({\cal K}_2)$
        are $\{\pi_q(a), \pi_p(a)\}$ and $\{\pi_q(a), p(a)\}$. They correspond to
        $\emptyset$ and $\{p(a)\}$ respectively when restricted to  $\HB_{P_2}$.

      \item Let ${\cal K}_3$ be the dl-program ${\cal K}_2$ in Example
        \ref{exam:dl:program:1}. 
        Then
      $\pi({\cal K}_3)=(\emptyset,P')$ where $P'$ consists of
      \begin{align*}
        & p(a) \lto \Not \pi_A, \hspace{2cm} \pi_q(a)\lto \Not q(a),\\
        & \pi_A \lto \Not \DL[S\oplus p,S'\odot \pi_q,S\sqcap\neg S'](a)
      \end{align*}
      where $A=\DL[S\oplus p,S'\ominus q;S\sqcap\neg S'](a)$. One can check that $\pi({\cal K}_3)$ has
      two strong answer sets, $\{\pi_q(a),\pi_A\}$ and $\{\pi_q(a),p(a)\}$, which are
      $\emptyset$ and $\{p(a)\}$ whenever restricted to the original Herbrand base.
  \end{itemize}
\end{example}


The main insight revealed by
the translation $\pi$ is, while a negative dl-atom is rewritten by replacing a $\ominus$ expression by a $\odot$ expression, any positive nonmonotonic dl-atom is negated twice, which emulates ``double negation" in nested expressions \cite{Lifschitz1999nested}.\footnote{A similar logic treatment has been found in a number of recent approaches to the semantics of various classes of logic programs,
e.g., in the ``double negation" interpretation of weight constraint programs \cite{FL2005:TPLP,LiuYouTPLP2011}.}


Although the translation $\pi$ provides an interesting characterization,  due to the difficulty of checking the monotonicity of a
dl-atom, for an arbitrary dl-program the translation can be expensive as it
depends on checking the entailment relation over the underlying
description logic. However, for the class of normal dl-programs,
$\pi$ takes polynomial time since checking the monotonicity of
dl-atoms amounts to checking the existence of the constraint operator,
and predicates occurring in dl-atoms have the arity at most 2.

%

We now proceed to show some properties of the translation $\pi$.

For any dl-program $\cal K$, $\pi(\cal K)$ has no
nonmonotonic dl-atoms left. Thus, by Theorem 4.13 of
\cite{DBLP:journals/ai/EiterILST08}, we have
\begin{proposition}\label{prop:3.1.1}
  Let ${\cal K}$ be a dl-program. If $I\subseteq\HB_{\pi(P)}$ is a
  strong answer set of $\pi(\cal K)$ then $I$ is minimal, i.e, there
  is no $I'\subset I$ such that $I'$ is a strong answer set of
  $\pi(\cal K)$.
\end{proposition}
\begin{proof}
  It is evident by Theorem 4.13 of
  \cite{DBLP:journals/ai/EiterILST08} and $\DL_{\pi(P)}^?=\emptyset$.
\end{proof}



The dl-programs in the above example show that the translation
$\pi$  preserves both strong and weak answer sets of a given dl-program in the extended
language, i.e., the strong and weak answer sets of
$\pi(\cal K)$ are those of $\cal K$ when restricted to the language of $\cal K$. In what follows,
we formally build up a one-to-one mapping between answer sets of
a dl-program $\cal K$ and those of $\pi(\cal K)$.

For convenience, given a dl-program ${\cal K}=(O,P)$ and $I\subseteq\HB_P$, we denote $\pi(I)=I\cup \pi_1(I)\cup \pi_2(I)$ where
\begin{align*}
  & \pi_1(I)=\{\pi_p(\vec c)\in\HB_{\pi(P)}\mid p(\vec c)\notin I\}, \mbox{ and}\\
  & \pi_2(I)=\{\pi_A\in\HB_{\pi(P)}\mid A\in \DL_P^?\ \&\ I\not\models_OA\}.
\end{align*}

\begin{lemma}\label{lem:main:1}
  Let $\mathcal K=(O,P)$ be a dl-program, $I\subseteq\HB_P$. Then
  \begin{enumerate}[(i)]
    \item  for any atom  $A$ occurring in $P$
     \[I\models_OA\ \ \textit{iff}\ \ I\cup \pi_1(I)\models_O\pi(A)\ \ \textit{iff}\ \ \pi(I)\models_O\pi(A);\]
    \item for any  dl-atom $A=\DL[\lambda;Q](\vec t)$ occurring in $P$,
     \[I\models_OA\ \ \textit{iff}\ \ I\cup \pi_1(I)\models_O\DL[\pi(\lambda);Q](\vec t) \ \ \textit{iff}\ \ \pi(I)\not\models_O\pi(\Not A). \]
  \end{enumerate}
\end{lemma}
\begin{proof}
  (i) It is obvious since $\pi(A)=A$ and predicates of the form
  $\pi_p$ and $\pi_A$ do not occur in $\cal K$.

  (ii) If there is no constraint operator occurring in $\lambda$ then $\DL[\pi(\lambda);Q](\vec t)=\DL[\lambda;Q](\vec t)$. Thus in this case, it is trivial as predicates of the form
  $\pi_p$ and $\pi_A$ do not occur in $\cal K$, and $\pi(\Not A)=\Not A$.

  Suppose there exists at least one constraint operator in $\lambda$. It is clear that $I\cup \pi_1(I)\models_O\DL[\pi(\lambda);Q](\vec t)$ if and only if $\pi(I)\not\models_O\pi(\Not A)$, and
  evidently, for any atom $\pi_p(\vec c)\in\HB_{\pi(P)}$, $\pi_p(\vec c)\in \pi_1(I)$ if and only if $p(\vec c)\notin I$.
  For clarity and without loss of generality, let $\lambda=(S_1\oplus p_1,S_2\ominus p_2)$. We have that\\
  $I \models_O\DL[\lambda;Q](\vec t)$\\
  iff $ O\cup\{S(\vec e)\mid p_1(\vec e)\in I \}\cup \{\neg S_2(\vec e)\mid p_2(\vec e)\notin I\}\models Q(\vec t)$\\
  iff $ O\cup\{S(\vec e)\mid p_1(\vec e)\in I\}\cup \{\neg S_2(\vec e)\mid \pi_{p_2}(\vec e)\in \pi_1(I)\}\models Q(\vec t)$\\
  iff $ O\cup\{S(\vec e)\mid p_1(\vec e)\in I\cup \pi_1(I)\}\cup \{\neg S_2(\vec e)\mid \pi_{p_2}(\vec e)\in I\cup \pi_1(I)\}\models Q(\vec t)$\\
  iff $ I\cup \pi_1(I)\models_O\DL[S_1\oplus p_1,S_2\odot \pi_{p_2};Q](\vec t)$\\
  iff $ I\cup \pi_1(I)\models_O \DL[\pi(\lambda);Q](\vec t)$\\
  iff $\pi(I)\not\models_O\pi(\Not A)$.

  The above proof can be extended to the case where $\lambda=(S_1\oplus p_1,\ldots, S_m\oplus p_m, S'_1\ominus q_1,\ldots, S'_n\ominus q_n;Q](\vec t)$.
\end{proof}

\begin{lemma}\label{lem:s}
  Let $\mathcal K=(O,P)$ be a dl-program and $I\subseteq\HB_P$. Then
  \begin{enumerate}[(i)]
    \item $\pi_1(I)=\{\pi_p(\vec c)\in\HB_{\pi(P)}\}\cap \lfp(\gamma_{[\pi(\mathcal K)]^{s,\pi(I)}})$,
    \item $\pi_2(I)=\{\pi_A\in\HB_{\pi(P)}\}\cap\lfp(\gamma_{[\pi(\mathcal K)]^{s,\pi(I)}})$, and
    \item $\gamma_{{\cal K}^{s,I}}^k=\HB_P\cap \gamma^k_{[\pi(\mathcal K)]^{s,\pi(I)}}$ for any $k\ge 0$.
  \end{enumerate}
\end{lemma}
\begin{proof}
  (i) It is evident that, for any atom $\pi_p(\vec c)\in\HB_{\pi(P)}$, the rule $(\pi_p(\vec c)\lto \Not p(\vec c))$ is in $\pi(P)$.
  We have that \\
  $\pi_p(\vec c)\in\pi_1(I)$ \\
  iff $p(\vec c)\notin I$ \\
  iff $p(\vec c)\notin \pi(I)$ \\
  iff the rule $(\pi_p(\vec c)\lto)$ belongs to $s[\pi(P)]^{s,\pi(I)}_O$\\
  iff $\pi_p(\vec c)\in \lfp(\gamma_{[\pi(\mathcal K)]^{s,\pi(I)}})$.

  (ii) It is clear that, for any $\pi_A\in\pi_2(I)$, the rule $(\pi_A\lto \pi(\Not A))$ is in $\pi(P)$ such that
  $A\in \DL_P^?$ and $I\not\models_OA$. Let $A=\DL[\lambda;Q](\vec t)$. We have that \\
  $\pi_A\in \pi_2(I)$\\
  iff $\pi_A\in\HB_{\pi(P)}$ and $I\not\models_OA$\\
  iff $\pi(I)\not\models_O \DL[\pi(\lambda);Q](\vec t)$ (by (ii) of Lemma \ref{lem:main:1})\\
  iff the rule $(\pi_A\lto)$ belongs to $s[\pi(P)]^{s,\pi(I)}_O$\\
  iff $\pi_A\in \lfp(\gamma_{[\pi(\mathcal K)]^{s,\pi(I)}})$.

  (iii) We show this by induction on $k$.

  Base: It is obvious for $k=0$.

  Step: Suppose it holds for $k=n$. Let us consider the case $k=n+1$. For any atom $\alpha\in \HB_P$,
  $\alpha\in\gamma^{n+1}_{{\cal K}^{s,I}}$
  if and only if there is a rule
  \[\alpha\lto\Pos,\NDL,\Not\Neg\]
  in $P$, where $\Pos$ is a set of atoms and monotonic dl-atoms and $\NDL$ is a set of nonmonotonic dl-atoms such that
  \begin{itemize}
    \item $\gamma^n_{{\cal K}^{s,I}}\models_OA$ for any $A\in \Pos$,
    \item $I\models_OB$ for any $B\in\NDL$, and
    \item $I\not\models_OC$ for any $C\in\Neg$.
  \end{itemize}
  It follows that
  \begin{itemize}
    \item $\gamma^n_{{\cal K}^{s,I}}\models_OA$ if and only if
    $\gamma^n_{[\pi({\cal K})]^{s,\pi(I)}}\models_OA$, by the  inductive assumption,
    \item $I\models_OB$ if and only if $\pi_B\not\in \pi(I)$, by the definition of $\pi_2(I)$,
    i.e., $\pi(I)\not\models_O\pi_B$, and
    \item $I\not\models_OC$ if and only if $\pi(I)\models_O\pi(\Not C)$ for any $C\in\Neg$, by Lemma \ref{lem:main:1}.
  \end{itemize}
  Thus we have that $\alpha\in\gamma^{n+1}_{{\cal K}^{s,I}}$ if and only if $\alpha\in\gamma^{n+1}_{[\pi(\mathcal K)]^{s,\pi(I)}}\cap\HB_P$.
\end{proof}

Now we have the following key theorem: there exists a one-to-one mapping between the
strong answer sets of a dl-program $\cal K$ and those of $\pi(\cal K)$.
\begin{theorem}\label{thm:delete:ominus:s}
  Let $\mathcal K=(O,P)$ be a dl-program. Then
  \begin{enumerate}[(i)]
    \item if $I$ is a strong answer set of $\mathcal K$ then  $\pi(I)$ is  a strong
  answer set of $\pi(\cal K)$;
    \item if $I^*$ is a strong answer set of $\pi(\cal K)$ then $I^*\cap\HB_P$ is a strong answer set of $\cal K$.
  \end{enumerate}
\end{theorem}
\begin{proof}
  (i) We have that
  \begin{flalign*}
     \lfp(\gamma_{[\pi({\cal K})]^{s,\pi(I)}})
   = &\lfp(\gamma_{[\pi({\cal K})]^{s,\pi(I)}})\cap (\HB_P\cup\{\pi_p(\vec c)\in\HB_{\pi(P)}\}\cup \{\pi_A\in\HB_{\pi(P)}\})\\
   = & [\HB_P\cap \lfp(\gamma_{[\pi({\cal K})]^{s,\pi(I)}})]\\
     &  \cup[\{\pi_p(\vec c)\in\HB_{\pi(P)}\}\cap\lfp(\gamma_{[\pi({\cal K})]^{s,\pi(I)}})]\\
     &  \cup[\{\pi_A\in\HB_{\pi(P)}\}\cap\lfp(\gamma_{[\pi({\cal K})]^{s,\pi(I)}})]\\
   = & [\HB_P\cap \bigcup_{i\ge 0}\gamma_{[\pi({\cal K})]^{s,\pi(I)}}^i]\cup
       \pi_1(I)\cup \pi_2(I),\mbox{ by (i) and (ii) of Lemma \ref{lem:s}} \\
   = & \bigcup_{i\ge 0}[\HB_P\cap \gamma_{[\pi({\cal K})]^{s,\pi(I)}}^i]\cup
       \pi_1(I)\cup \pi_2(I)\\
   = & \bigcup_{i\ge 0}\gamma_{{\cal K}^{s,I}}^i\cup
       \pi_1(I)\cup \pi_2(I),\mbox{ by (iii) of Lemma \ref{lem:s}} \\
   = & I\cup \pi_1(I)\cup \pi_2(I), \mbox{ since $I$ is a strong answer set of $\cal K$}\\
   = & \pi(I).
  \end{flalign*}
  It follows that $\pi(I)$ is a strong answer set of $\pi(\cal K)$.

  (ii) We prove $I^*=\pi(\HB_P\cap I^*)$ first.
  \begin{flalign*}
   I^* = & I^*\cap (\HB_P\cup\{\pi_p(\vec c)\in\HB_{\pi(P)}\}\cup \{\pi_A\in\HB_{\pi(P)}\})\\
       = & (I^*\cap \HB_P)\cup (I^*\cap \{\pi_p(\vec c)\in\HB_{\pi(P)}\})\cup (I^*\cap \{\pi_A\in\HB_{\pi(P)}\})\\
       = & (I^*\cap\HB_P) \cup \pi_1(\HB_P\cap I^*)\cup \pi_2(\HB_P\cap I^*), \mbox{ by (i) and (ii) of Lemma \ref{lem:s}}\\
       = & \pi(I^*\cap\HB_P).
  \end{flalign*}
  Let $I=I^*\cap\HB_P$. We have that
  \begin{flalign*}
    \lfp(\gamma_{{\cal K}^{s,I}}) = & \bigcup_{i\ge 0}\gamma^i_{{\cal K}^{s,I}}\\
    = & \bigcup_{i\ge 0}(\HB_P\cap \gamma_{[\pi({\cal K})]^{s,\pi(I)}}^i),\mbox{ by (iii) of Lemma \ref{lem:s}}\\
    = & \HB_P\cap \bigcup_{i\ge 0}\gamma_{[\pi({\cal K})]^{s,\pi(I)}}^i\\
    = & \HB_P\cap\lfp(\gamma_{[\pi({\cal K})]^{s,\pi(I)}})\\
    = & \HB_P\cap \pi(I)\mbox{ since $\pi(I)=I^*$ is a strong answer set of $\pi(\cal K)$}\\
    = & I.
  \end{flalign*}
  It follows that $I$ is a strong answer set of $\cal K$.
\end{proof}

Please note that, we need to determine the monotonicity of dl-atoms in the translation $\pi$ which is not tractable generally, and the translation does nothing for monotonic dl-atoms.
That is, the ``double negation" interpretation applies only to positive nonmonotonic dl-atoms. If we deviate from this condition, the translation no longer works for strong answer sets.
For example,
one may question whether monotonic
dl-atoms can be handled like nonmonotonic dl-atoms,  and if so, the translation turns out to be polynomial. Unfortunately we give
a negative answer below.


\begin{example}\label{exam:4:new}
Consider the dl-program $\mathcal K_1=(\emptyset,P_1)$ where
$P_1=\{p(a)\lto \DL[S\oplus p, S'\ominus q; S](a)\}$. The dl-atom $A=\DL[S\oplus p, S'\ominus q; S](a)$
is monotonic.
Thus, ${\cal K}_1$ is positive but neither canonical nor normal.
 It is evident that $\emptyset$ is the unique strong answer set of ${\cal K}_1$.
If we apply $\pi$ to eliminate the constraint operator in monotonic dl-atoms as what $\pi$ does for nonmonotonic dl-atoms, we would get the dl-program $(\emptyset, P_1')$ where $P_1'$ consists of
\begin{align*}
  p(a) \lto \Not \pi_A, \qquad \pi_A\lto \Not \DL[S\oplus p, S'\odot \pi_q; S](a), \qquad \pi_q(a)\lto \Not q(a).
\end{align*}
One can verify that this dl-program has two strong answer sets,
$\{p(a),\pi_q(a)\}$ and $\{\pi_A,\pi_q(a)\}$, which are $\{p(a)\}$ and
$\emptyset$ respectively when restricted to $\HB_P$. However, we know that $\{p(a)\}$ is not
a strong answer set of ${\cal K}_1$. That is, such a translation may introduce some strong answer sets that do not correspond to any  of the original dl-program in this case.

One may argue that $\pi$ should treat monotonic dl-atoms in the same manner as treating nonmonotonic dl-atoms in default negation. However, for the dl-program ${\cal K}_2=(\emptyset,P_2)$ where $P_2$ consists
of \[p(a)\lto \DL[S\odot p, S\ominus p;\neg S](a),\]
we have that the resulting dl-program $(\emptyset,P_2')$ where $P_2'$ consists of
\begin{align*}
  p(a)\lto \DL[S\odot p, S\odot \pi_p, \neg S](a), \hspace{1cm} \pi_p(a)\lto \Not p(a).
\end{align*}
This dl-program has no strong answer sets at all. But the original dl-program has a unique
strong answer $\{p(a)\}$. Even if we replace every $p$ occurring in the dl-atom with $\pi_p$, the answer is still negative.
\end{example}

Similarly, we can show a one-to-one mapping between the
weak answer sets of a dl-program $\cal K$ and those of $\pi(\cal K)$.
\begin{theorem}\label{thm:delete:ominus:w}
  Let $\mathcal K=(O,P)$ be a dl-program. 
   Then
  \begin{enumerate}[(i)]
    \item if $I$ is a weak answer set of $\mathcal K$, then  $\pi(I)$ is  a weak
  answer set of $\pi(\cal K)$;
    \item if $I^*$ is a weak answer set of $\pi(\cal K)$, then $I^*\cap\HB_P$ is a weak answer set of $\cal K$.
  \end{enumerate}
\end{theorem}
\begin{proof}
  See Appendix \ref{app:B}.
\end{proof}

As a matter of fact, there is a simpler translation that preservers
weak answer sets of dl-programs.

\begin{definition}[$\pi^*(\cal K)$]
Let $\pi^*(\cal K)$ be the same translation as $\pi(\cal K)$ except that it does not
distinguish nonmonotonic dl-atoms from dl-atoms, i.e., it handles
monotonic dl-atoms in the
way $\pi(\cal K)$ deals with nonmonotonic dl-atoms.
\end{definition}
It is clear that $\pi^*$ is polynomial. For instance, let us consider
the dl-program ${\cal K}_2$ in Example \ref{exam:4:new}. We have that $\pi^*({\cal K}_2)=(\emptyset, \pi^*(P_2))$
where $\pi^*(P_2)$ consists of
\begin{align*}
   p(a)\lto \Not \pi_A,\hspace{1.5cm} \pi_p(a)\lto \Not p(a),\hspace{1.5cm}
   \pi_A\lto\Not\DL[S\odot p,S\odot\pi_p;\neg S](a)
\end{align*}
where $A=\DL[S\odot p,S\ominus p;\neg S](a)$. The interested readers can verify that $\{p(a)\}$ is the unique weak answer set of $\pi^*({\cal K}_2)$.

\begin{proposition}\label{prop:delete:minus:w}
  Let ${\cal K}=(O,P)$ be a dl-program.
  Then
   \begin{enumerate}[(i)]
    \item If $I\subseteq\HB_P$ is a weak answer set of $\cal K$, then $\pi(I)$ is a weak answer set of $\pi^*(\cal K)$.
    \item If $I^*$ is a weak answer set of $\pi^*(\cal K)$, then $I^*\cap\HB_P$ is a weak answer set of $\cal K$.
  \end{enumerate}
\end{proposition}
\begin{proof}
The proof is similar to the one of Theorem \ref{thm:delete:ominus:w}.
\end{proof}


Note that, to remove the constraint operator from nonmonotonic dl-atoms of a
dl-program, in general we must extend the underlying language. This
is because there are dl-programs whose strong answer sets are
not minimal, but the translated dl-program contains no nonmonotonic
dl-atoms hence its strong answer sets are minimal (cf. Theorem 4.13
of \cite{DBLP:journals/ai/EiterILST08}). Therefore, we conclude that there is no
transformation not using extra
symbols that eliminates the constraint operator from normal
dl-programs while preserving strong answer sets.

Recall that \citeN{Motik:JACM:2010} introduced a
polynomial time transformation to translate a dl-atom mentioning no
constraint operator into a first-order sentence and proved that,
given a canonical dl-program $\mathcal K$, there is a one-to-one
mapping between the strong answer sets of $\mathcal K$ and the MKNF
models of the corresponding MKNF knowledge base
(Theorem 7.6 of \cite{Motik:JACM:2010}).
Theorem \ref{thm:delete:ominus:s} above extends their result from
canonical dl-programs to normal
dl-programs, by applying
the translation $\pi$  first.
In particular, the combined transformation is still polynomial for normal dl-programs.

\section{Translating Dl-programs to Default Theories}\label{Sec:ToDefault}
Let us briefly recall the basic notions of default logic
\cite{Reiter1980}. We assume a first-order language $\cal L$ with
a signature consisting of predicate, variable and constant symbols,
including equality. A {\em default theory $\Delta$} is a pair
$(D,W)$ where $W$ is a set of closed formulas (sentences) of $\cal
L$, and $D$ is a set of {\em defaults} of the form:
\begin{equation}\label{default}
    \frac{\alpha:\beta_1,\ldots,\beta_n}{\gamma}
\end{equation}
where $\alpha$ (called {\em premise}), $\beta_i$  $(0\leq i\leq n)$ (called {\em justification}),
\footnote{\citeN{Reiter1980} used $n\geq 1$; the
generalization we use is common and insignificant
for our purposes.} $\gamma$ (called {\em conclusion}) are formulas of $\cal
L$. A default $\delta$ of the form (\ref{default}) is {\em closed}
if $\alpha,\beta_i(1\leq i\leq n),\gamma$ are sentences, and a
default theory is {\em closed} if all of its defaults are closed. In
the following, we assume that every default theory is closed, unless
stated otherwise. Let $\Delta=(D,W)$ be a default theory, and let $S$ be a
set of sentences. We define $\Gamma_\Delta(S)$ to be the smallest
set satisfying
\begin{itemize}
  \item $W\subseteq\Gamma_\Delta(S)$,
  \item $\Th(\Gamma_\Delta(S))=\Gamma_\Delta(S)$, and
  \item If $\delta$ is a default of the form (\ref{default}) in $D$, and $\alpha\in\Gamma_\Delta(S)$, and
  $\neg\beta_i\notin S$ for each $i~(1\leq i\leq n)$ then $\gamma\in\Gamma_\Delta(S)$,
\end{itemize}
where $\Th$ is the classical closure operator, i.e., $\Th(\Sigma)=\{\psi\mid \Sigma\vdash\psi\}$ for
a set of formulas $\Sigma$.
A set of sentences $E$ is an {\em extension} of $\Delta$ whenever
$E=\Gamma_\Delta(E)$.  Alternatively, a set of sentences $E$ is an
extension of $\Delta$ if and only if $E=\bigcup_{i\ge 0}E_i$, where
\begin{equation}\label{eq:Ei:default}\left\{
  \begin{array}{ll}
    E_0=W, &  \\
    E_{i+1}=\Th(E_i)\cup\{\gamma\mid \frac{\alpha:\beta_1,\ldots,\beta_n}{\gamma}\in D\ s.t.\
  \alpha\in E_i\ \textit{and}\ \neg\beta_1,\ldots,\neg\beta_n\notin E\}, & \hbox{$i\ge 0$}.
  \end{array}
\right.
\end{equation}
It is not difficult to see that $\alpha\in E_i$ in (\ref{eq:Ei:default}) can be replaced by $E_i\vdash\alpha$.

In this section, we will present two approaches to translating a dl-program to a default
theory which preserves the strong answer sets of dl-programs. In the first, if the given ontology is inconsistent, the resulting default theory is trivialized and possesses a unique
extension that consists of all formulas of ${\cal L}$, while in the
second, following the spirit of dl-programs, an inconsistent ontology does not trivialize the resulting
default theory.\footnote{The two approaches presented here do not preserve weak answer sets of dl-programs, for a good reason.
Technically however, by applying a translation first that makes all dl-atoms occur negatively, we can obtain  translations that preserve
weak answer sets of dl-programs.}  Then we will give a translation from dl-programs under the weakly well-supported answer set semantics \cite{Yi-Dong:IJCAI:2011} to default theories. Before we proceed, let us comment on the impact of equality reasoning in the context of representing dl-programs by default logic.

\subsection{Equality reasoning}
\label{sec:equality}

The answer set semantics of dl-programs are defined with the intention
that equality reasoning in the ontology is fully captured, while at
the same time reasoning with rules is conducted
relative to the Herbrand domain. The latter implies that equality reasoning is not carried over to reasoning with rules.
For example, the dl-program
$$(O,P)= (\{a \approx b\}, \{p(a) \leftarrow \Not p(b),~p(b) \leftarrow \Not p(a)\})$$
has two (strong) answer sets, $\{p(a)\}$ and $\{p(b)\}$, neither of
 which carries equality reasoning in the ontology to the rules. But
 if the dl-program is translated to the default theory $(\{\frac{:
   \neg p(b)}{p(a)}, \frac{: \neg p(a)}{p(b)}\}, \{a \approx b\})$ it
 has -- evaluated under first-order logic with equality -- no extensions.  As suggested in \cite{DBLP:journals/ai/EiterILST08}, one can emulate equality reasoning by imposing
 the unique name assumption (UNA) and a congruence relation on ontology.

 Although congruence and UNA in general allow one to extend equality
 reasoning from the ontology to the rules, we
 will show that, for the purpose of representing dl-programs by default logic, for the standard default encoding
 like in the example above, 
 strong answer sets are preserved by treating $\approx$ as a
 congruence relation on ontology (i.e., replacement of equals by
 equals only applies to the predicates of the ontology); in
 particular, there is no need to adopt the UNA.
 For the default translation that handles inconsistent
 ontologies in the original spirit of dl-programs, neither congruence nor UNA is needed. These results provide additional insights in capturing dl-programs by default logic.

\comment{
Consider the relational first-order language ${\cal L}$ with equality $\approx$, and \you {} { let} $({\cal C}', {\cal P}')$ be the signature of $\cal L$ consisting of denumerable disjoint sets of constant and predicate symbols $\cal C'$ and $\cal P'$, respectively.
An interpretation of formulas in ${\cal L}$ is a tuple ${\cal I}=\langle U,\cdot^I\rangle$, where $U$ is a non-empty domain and $\cdot^I$ is a mapping which maps constants in ${\cal C}'$ to elements in $U$, and assigns
a relation $p^I\subseteq U^n$ to every {\em n}-ary predicate symbol $p\in {\cal P}'$.

Under UNA, an interpretation ${\cal I}=\langle U,\cdot^I\rangle$ is required to satisfy
$a^{I} \not = b^{I}$, for any distinct constants $a,b \in \cal C'$.
That is, such an interpretation maps distinct constants to different elements in the domain of interpretation.

When we say that $\approx$ is interpreted as a congruence relation on ${\cal L}$, we mean that it is just an ordinary predicate that is
reflexive, symmetric and transitive, and allows the replacement of equals by equals for all predicates in ${\cal P}'$, i.e., \you {} { $\approx$ is an equivalence relation and satisfies}
$\forall \vec x,\vec y. (\vec x\approx \vec y)\supset(p(\vec x)\supset p(\vec y))$ where $p\in\cal P'$, and $\vec x,\vec y$ are tuples of distinct variables whose arities match that of the predicate $p$.
When ``replacement of equals by equals" only applies to a subset ${\cal P''} \subseteq {\cal P'}$, we will explicitly say so.
}

Thanks to Fitting, as shown by the following theorem, the equality $\approx$ can be simulated by a congruence in the sense that a first-order formula with equality is satisfiable in
a model with true equality if and only if it is satisfiable in a model where $\approx$ is
interpreted as a congruence relation.

\begin{theorem}[Theorem 9.3.9 of \cite{Fitting:1996}]\label{thm:fitting}
  Let ${\cal L}$ be a first-order language, $S$ a set of sentences  and $X$ a sentence. Then
  $S\models_{\approx} X$ iff $S\cup eq({\cal L})\models X$, where
  $S\models_{\approx}X$ means that $X$ is true in every model of $S$ in which $\approx$ is interpreted as
  an equality relation and $eq({\cal L})$ consists of the following
  axioms:
  \begin{align}
    \label{fit:replacement:1}& \mbox{reflexivity} &&(\forall x) (x\approx x),\\
    \label{fit:replacement:2}& \mbox{function replacement}&& (\forall \vec x,\vec y)[(\vec x\approx \vec y)\supset (f(\vec x)\approx f(\vec y))],\quad \mbox{for every function $f$ of $\cal L$,}\\
\label{fit:replacement:3}&  \mbox{predicate replacement}&& (\forall
\vec x,\vec y)[(\vec x\approx \vec y)\supset  (p(\vec x)\supset p(\vec y))],\quad \mbox{for every predicate $p$ of $\cal L$}.
\end{align}
\end{theorem}

Since $\approx$ is a part of $\cal L$, the symmetry and transitivity of $\approx$  in $\cal L$ can be easily derived from (\ref{fit:replacement:1}) and (\ref{fit:replacement:3}) as illustrated by Fitting \citeyear{Fitting:1996}. In what follows, we take $\approx$ as a congruence, unless otherwise explicitly stated, and
we write $\models$ for $\models_{\approx}$ when it is clear from its context,

Before giving the translation from dl-programs to default theories, we first present a transformation
for dl-atoms, which will be referred to throughout this section.
Let ${\cal K} = (O,P)$ be a dl-program (for convenience, assume $O$ is already translated to a first-order theory), $I \in \HB_P$ an interpretation, and $\tau(C)$ is
a first-order sentence translated from $C$:
\begin{itemize}
    \item if $C$ is an atom in $\HB_P$, then $\tau(C) = C$, and
    \item if $C$ is a dl-atom of the form (\ref{dl:atom}) then $\tau(C)$ is a first-order sentence
     \[\left[\bigwedge_{1\leq i\leq m}\tau(S_i\ op_i\ p_i)\right]\supset Q(\vec t)\ \  \textmd{, where}\]
     \[\tau(S\ op\ p)=\left\{
     \begin{array}{ll}
        \bigwedge_{p(\vec c)\in \HB_P}[p(\vec c)\supset S(\vec c)] & \mbox{if  $op=\oplus$}\\
        \bigwedge_{p(\vec c)\in \HB_P}[p(\vec c)\supset \neg S(\vec c)] & \mbox{if $op=\odot$}\\
        \bigwedge_{p(\vec c)\in \HB_P}[\neg p(\vec c)\supset\neg S(\vec c)] &\mbox{if $op=\ominus$}
     \end{array}
     \right.\]
     where  we identify
     $S(\vec c)$ and $Q(\vec t)$ with their corresponding first-order sentences respectively. Since $\vec t$ and $\vec c$ mention no variables,      $\tau(C)$ has no free variables. Thus $\tau(C)$ is closed.
  \end{itemize}

\subsection{Translation trivializing inconsistent ontology knowledge bases}

We present the first transformation from dl-programs to default theories which preserves strong
answer sets of dl-programs without nonmonotonic dl-atoms.

\begin{definition}[$\tau(\cal K)$]
Let ${\cal K}=(O,P)$ be a
dl-program.
We define $\tau(\cal K)$ to be the default theory $(\tau(P),\tau(O))$
as follows 
\begin{itemize}
  \item  $\tau(O)$ is the congruence rewriting of $O$, i.e., replacing true equality in $O$ by a congruence;
  by abusing the symbol we denote the congruence by $\approx$,  together with the
  axioms (\ref{fit:replacement:1}) and (\ref{fit:replacement:3})
  for every predicate in the underlying language of $O$, denoted by ${\cal A}_O$.%
\footnote{Note that we do not need function replacement axioms here as there are no functions occurring
in $O$.}
Given an ontology $O$, we assume the predicates in the underlying language of $O$ are exactly the ones occurring in $O$.
  \item  $\tau(P)$ consists of, for each dl-rule  of the form (\ref{dl:rule}) in $P$, the default
   \begin{eqnarray*}
      \frac{\bigwedge_{1\leq i\leq m}\tau(B_i):\neg\tau(B_{m+1}),\ldots,\neg\tau(B_n)}{A}
   \end{eqnarray*}
   where $\tau(C)$ is defined in the preceding subsection and equality $\approx$ is now taken as the congruence relation above.
\end{itemize}
\end{definition}
It is evident that, given a dl-program ${\cal K}=(O,P)$, every
extension of $\tau({\cal K})$ has the form $\Th(I\cup \tau(O))$, for some $I\subseteq\HB_P$. Thus, if $O$ is consistent
then every extension of $\tau(\cal K)$ is consistent. On the other hand, if $O$ is inconsistent then $\tau(\cal K)$ has a unique
extension which is inconsistent.
It is clear
that $\tau(\cal K)$ is of polynomial size of the dl-program $\cal
K$, since the size of  $\HB_P$ is polynomial in the size  of $P$.

\begin{example}[Continued from Example \ref{exam:dl:program:1}]
\label{exam:4}
\begin{itemize}
  \item Note that $\tau({\cal K}_1)=(\{d\},W)$ where $W=\{\forall x. S(x)\supset S'(x)\}\cup {\cal A}_{O_1}$ and
  \begin{align*}
    d=\frac{(p(a)\supset S(a))\supset   S'(a):}{p(a)}.
  \end{align*}
  It is easy to see that $\tau({\cal K}_1)$ has a unique extension $\Th(W)$.

  \item Note that $\tau({\cal K}_2)=(\{d\},W)$ where $W={\cal A}_{O_2}$ and
  \begin{align*}
    d=\frac{(p(a)\supset S(a))\wedge (q(a)\supset \neg S'(a))\wedge (\neg q(a)\supset \neg S'(a))\supset S(a)\wedge\neg
    S'(a):}{p(a)}.
  \end{align*}
   One can verify that $\Th(W)$ is the unique extension of $\tau({\cal K}_2)$
  though we know that ${\cal K}_2$ has two strong answer sets,
  $\emptyset$ and   $\{p(a)\}$.
\end{itemize}
\end{example}

The default theory $\tau({\cal K}_2)$ in the above example shows
that if a dl-program $\cal K$ mentions nonmonotonic dl-atoms, then
$\tau(\cal K)$ may have no corresponding extensions for some strong answer sets of $\cal K$. However, the
one-to-one mapping between strong answer sets of $\cal K$ and the extensions of $\tau(\cal K)$
does exist  for dl-programs mentioning no nonmonotonic dl-atoms and whose
knowledge bases are consistent.


In the following, when it is clear from the context,
we will identify a finite set $S$ of formulas as the conjunction of elements in $S$ for convenience.
The following lemma relates a disjunctive normal form to a conjunctive normal form, which is well-known.

\begin{lemma}\label{lem:6}
  Let $A=\{A_1,\ldots,A_n\}$, $B=\{B_1,\ldots,B_n\}$ and $I=\{i\mid 1\leq i\leq n\}$ where  $A_i,B_i~(1\leq i\leq n)$ are atoms.
  Then 
\[\bigvee_{I'\subseteq I}\left(\bigwedge_{i\in I'}A_i\wedge\bigwedge_{j\in I\setminus I'}B_j\right)\equiv\bigwedge_{i\in I}(A_i\vee B_i).\]
\end{lemma}

\begin{lemma}\label{lem:forget}
  Let $M$ be a set of ground atoms, $\psi_i,\varphi_i$ and $\phi$ are
  formulas not mentioning true equality, the  predicates $p, p_1,p_2$ and the
  predicates occurring in $M$, where $1\le i\le n$. Then
  \begin{align*}
    (1) & \bigwedge M\wedge \bigwedge_{1\le i\le n}((p(\vec c_i)\supset \psi_i)\wedge (\neg p(\vec c_i)\supset\varphi_i))\models\phi\ \textit{iff}\
  \bigwedge_{p(c_j)\in M}\psi_j\wedge \bigwedge_{p(\vec c_i)\notin  M}(\psi_i\vee\varphi_i)\models\phi,\\
    (2) &  \bigwedge M\wedge \bigwedge_{1\le i\le n}((p_1(\vec c_i)\supset \psi_i)\wedge (\neg p_2(\vec c_i)\supset\varphi_i))\models\phi\ \textit{iff}\
  \bigwedge_{p_1(\vec c_i)\in M}\psi_i\models\phi.
  \end{align*}
\end{lemma}
\begin{proof}
 (1)  The direction from right to left is obvious as $(\alpha\supset
  \psi)\wedge(\neg \alpha\supset\varphi)\models \psi\vee\varphi$. Let
  us consider the other direction. It suffices to show
\begin{equation}
\label{line:lemma-four}
\bigwedge_{p(\vec c_i)\notin M}((p(\vec c_i)\supset \psi_i)\wedge (\neg p(\vec c_i)\supset\varphi_i))\models\phi\ \ \textit{only if}\
  \bigwedge_{p(\vec c_i)\notin  M}(\psi_i\vee\varphi_i)\models\phi.
\end{equation}
Towards a contradiction, suppose that the left hand side of this
statement holds and there is an interpretation ${\cal I}\models\bigwedge_{1\le i\le n}(\psi_i\vee\varphi_i)$ and ${\cal I}\not\models\phi$. It follows that ${\cal I}\not\models \bigwedge_{1\le i\le n}((p(\vec c_i)\supset \psi_i)\wedge (\neg p(\vec c_i)\supset\varphi_i))$. Thus there exists some $k~(1\le k\le n)$ such that
  ${\cal I}\not\models (p(\vec c_k)\supset \psi_k)\wedge (\neg p(\vec c_k)\supset\varphi_k)$. Without loss of generality, we assume $k=1$. Let us consider the following two cases:
  \begin{itemize}
    \item ${\cal I}\models p(\vec c_1 )$. In this case we have ${\cal I}\not\models\psi_1$, by which ${\cal I}\models\varphi_1$ due to
    ${\cal I}\models\psi_1\vee\varphi_1$. As the formulas $\psi_i,
      \varphi_i ~(1\le i\le n)$ and $\phi$
 do not involve the predicate $p$, the
      interpretation ${\cal I}_1$ which
      coincides with ${\cal I}$ except that ${\cal I}_1\not\models
      p(\vec c_1)$ satisfies the conditions ${\cal I}_1\models\bigwedge_{1\le i\le n}(\psi_i\vee\varphi_i)$ and ${\cal I}_1\not\models\phi$.
From ${\cal I}\models\var_1$ it follows that ${\cal
  I}_1\models\var_1$;  thus ${\cal I}_1\models (p(\vec
c_1)\supset\psi_1)\wedge(\neg p(\vec c_1)\supset \var_1)$. It follows
that there exists some $j~(2\le j\le n)$ such that
    ${\cal I}_1\not\models (p(\vec c_j)\supset\psi_j)\wedge (\neg
p(\vec c_j)\supset\var_j)$. Without loss of generality, we can assume
$j=2$.
 With a similar case analysis and continuing the argument,
it follows that there exists  an interpretation ${\cal I}_{n-1}$ such
that ${\cal I}_{n-1}\models \bigwedge_{1\le i\le
  n}(\psi_i\vee\varphi_i)$, ${\cal I}_{n-1}\not\models\phi$ and ${\cal
  I}_{n-1}\models \bigwedge_{1\le i\le n-1}((p(\vec c_i)\supset
\psi_i)\wedge (\neg p(\vec c_i)\supset\varphi_i))$. It follows that
${\cal I}_{n-1}\not\models (p(\vec c_n)\supset\psi_n)\wedge(\neg
p(\vec c_n)\supset\var_n)$. We can 
finally construct an interpretation ${\cal I}_n$ in a similar way
that satisfies
    \begin{itemize}
      \item ${\cal I}_n\models \bigwedge_{1\le i\le n}(\psi_i\vee\varphi_i)$,
      \item ${\cal I}_n\not\models\phi$, and
      \item ${\cal I}_n\models \bigwedge_{1\le i\le n}((p(\vec c_i)\supset \psi_i)\wedge (\neg p(\vec c_i)\supset\varphi_i))$.
    \end{itemize}
As the latter combined with the assumption implies ${\cal
I}_n\models\phi$, we have a contradiction.

 \item ${\cal I}\not\models p(\vec c_1)$. Similar to the previous case.
  \end{itemize}

  (2) The direction from right to left is obvious again. For the other direction, suppose that there is an interpretation ${\cal I}$ such that ${\cal I}\models \bigwedge_{p_1(\vec c_i)\in M}\psi_i$ and ${\cal I}\not\models\phi$. We construct an interpretation ${\cal I}''$,  which is the same as ${\cal I}$ except that ${\cal I}'\models\bigwedge M$, ${\cal I}'\not\models p_1(c_i)$ if $p_1(c_i)\notin M$, and ${\cal I}'\models p_2(\vec c_j)$ if $p_2(c_j)\notin M$, for every $1\le i,j\le n$.
  It is clear that ${\cal I}'\models \bigwedge_{p_1(\vec c_i)\in M}\psi_i$ and ${\cal I}'\not\models\phi$. However, we have
  ${\cal I}'\models \bigwedge M\wedge \bigwedge_{1\le i\le n}((p_1(\vec c_i)\supset \psi_i)\wedge (\neg p_2(\vec c_i)\supset\varphi_i))$, which implies ${\cal I}'\models\phi$, a contradiction.
\end{proof}
%

Please note here that, in the above lemma, it is crucial that  $\psi_i,\varphi_i$ and $\phi$ mention no true equality. Otherwise, one can check that, if $\approx$ is taken as true equality, then on the one hand we have
\[[(p(c_1)\supset c_1\approx c_2)\wedge (\neg p(c_1)\supset q)]\wedge [(p(c_2)\supset q)\wedge (\neg p(c_2)\supset\neg q)]\models q\]
and on the other we have $(c_1\approx c_2\vee q)\not\models q$. It is clear that this discrepancy will not arise if $\approx$ is treated as a congruence relation and there is no predicate replacement axiom for the predicate $p$.

\begin{lemma}\label{lem:DF:1}
  Let ${\cal K}=(O,P)$ be a dl-program and $I\subseteq\HB_P$. Then
  \begin{enumerate}[(i)]
    \item If $A$ is an atom in $\HB_P$ and $O$ is consistent, then
    $I\models_OA$ iff $\tau(O)\cup I\models\tau(A)$.
    \item If $A=\DL[\lambda;Q](\vec t)$ is a monotonic dl-atom, then
    $I\models_OA$ iff   $\tau(O)\cup I\models\tau(A)$.
  \end{enumerate}
\end{lemma}
\begin{proof}
  (i) Since $A$ is an atom and $O$ mentions no predicates occurring in $I$, we have that
  $\tau(O)\cup I$ is consistent if and only if $O$ is consistent.
  It follows that   $I\models_OA$   iff $A\in I$   iff $I\models \tau(A) $ since $\tau(A)=A$.
  It is obvious that if $I\models\tau(A)$ then $\tau(O)\cup I\models\tau(A)$. It remains to show  that
  $I\models\tau(A)$ if $\tau(O)\cup I\models \tau(A)$. Suppose
  $I\not\models\tau(A)$, i.e., $\tau(A)\notin I$. Thus
there exists an interpretation $\cal I$  such that ${\cal I}\models I$
and ${\cal I}\not\models\tau(A)$. Recall that $\tau(O)$ has no
equality, and it has no predicates in common  with $I$. We can construct an
interpretation ${\cal I}^*$ which coincides with $\cal I$ except
that ${\cal I}^*\models\tau(O)$.  It follows ${\cal I}^*\models\tau(A)$ by ${\cal I}^*\models \tau(O)\cup I$, which contradicts ${\cal I}\not\models\tau(A)$ as $\cal I$ coincides with $\cal I^*$ for the predicate occurring in $\tau(A)$.

  (ii) For clarity, and without loss of generality, let $\lambda=(S_1\oplus p_1,S_2\ominus p_2)$. We have that
  \begin{align}\label{eqn:lem:6:1}
    \nonumber
    & \tau(O)\cup I\models \tau(A)\ \ \textit{iff}\\
    \nonumber
    &\tau(O)\cup I\models \left(\bigwedge_{p_1(\vec e)\in \HB_P}(p_1(\vec e)\supset S_1(\vec e))\wedge \bigwedge_{p_2(\vec e)\in\HB_P}
            (\neg p_2(\vec e)\supset\neg S_2(\vec e))\right)\supset Q(\vec t)\ \  \textit{iff}\\
    & I\wedge \bigwedge_{p_1(\vec e)\in \HB_P}(p_1(\vec e)\supset S_1(\vec e))\wedge \bigwedge_{p_2(\vec e)\in\HB_P}
            (\neg p_2(\vec e)\supset\neg S_2(\vec e))\models \tau(O)\supset Q(\vec t).
  \end{align}

  Let us consider the following two cases:

  (a) $p_1\neq p_2$. We have that Equation (\ref{eqn:lem:6:1}) holds iff $\{S_1(\vec e)\mid p_1(\vec e)\in I\}\models \tau(O)\supset Q(\vec t)$
  by (2) of Lemma \ref{lem:forget}.
  It follows that \\
  $\{S_1(\vec e)\mid p_1(\vec e)\in I\}\models \tau(O)\supset Q(\vec t)$\\
  $\Rto \{S_1(\vec e)\mid p_1(\vec e)\in I\}\cup\{\neg S_2(\vec e)\mid p_2(\vec e)\notin I\}\models \tau(O)\supset Q(\vec t)$\\
  $\Rto \tau(O)\cup \{S_1(\vec e)\mid p_1(\vec e)\in I\}\cup\{\neg S_2(\vec e)\mid p_2(\vec e)\notin I\}\models Q(\vec t)$\\
  $\Rto O\cup \{S_1(\vec e)\mid p_1(\vec e)\in I\}\cup\{\neg S_2(\vec e)\mid p_2(\vec e)\notin I\}\models Q(\vec t)$ (now $\approx$ is taken as an equality, by Theorem \ref{thm:fitting})\\
  $\Rto I\models_OA$.

  On the other hand, let $I'=\{p_2(\vec e)\in\HB_P\}$. We have that \\
  $I\models_OA$ \\
  $\Rto I\cup I'\models_OA$ (since $A$ is monotonic)\\
  $\Rto O\cup\{S_1(\vec e)\mid p_1(\vec e)\in I\cup I'\}\cup \{\neg S_2(\vec e)\mid p_2(\vec e)\notin I\cup I'\}\models Q(\vec t)$\\
  $\Rto O\cup\{S_1(\vec e)\mid p_1(\vec e)\in I\cup I'\}\models Q(\vec t)$\\
  $\Rto O\cup\{S_1(\vec e)\mid p_1(\vec e)\in I\}\models Q(\vec t)$\\
  $\Rto \{S_1(\vec e)\mid p_1(\vec e)\in I\}\models O\supset Q(\vec t)$\\
  $\Rto \{S_1(\vec e)\mid p_1(\vec e)\in I\}\models \tau(O)\supset Q(\vec t)$ (now $\approx$ is taken as a congruence, by Theorem \ref{thm:fitting}).

  (b) $p_1=p_2=p$. By (1) of Lemma \ref{lem:forget},  we have that Equation (\ref{eqn:lem:6:1}) holds iff
  $$\{S_1(\vec e)\mid p(\vec e)\in I\}\cup \{S_1(\vec e)\vee \neg S_2(\vec e)\mid p(\vec e)\in \HB_P\setminus I\}\models \tau(O)\supset Q(\vec t).$$
  It follows that \\
  $\{S_1(\vec e)\mid p(\vec e)\in I\}\cup \{S_1(\vec e)\vee \neg S_2(\vec e)\mid p(\vec e)\in \HB_P\setminus I\}\models \tau(O)\supset Q(\vec t)$\\
  $\Rto \{S_1(\vec e)\mid p(\vec e)\in I\}\cup \{\neg S_2(\vec e)\mid p(\vec e)\in \HB_P\setminus I\}\models \tau(O)\supset Q(\vec t)$\\
  $\Rto \tau(O)\cup \{S_1(\vec e)\mid p(\vec e)\in I\}\cup \{\neg S_2(\vec e)\mid p(\vec e)\notin I\}\models Q(\vec t)$\\
  $\Rto O\cup \{S_1(\vec e)\mid p(\vec e)\in I\}\cup \{\neg S_2(\vec e)\mid p(\vec e)\notin I\}\models Q(\vec t)$
  (now $\approx$ is taken as an equality, by Theorem \ref{thm:fitting})\\
  $\Rto I\models_OA$.

  Conversely, suppose $I\models_OA$. Let $M_1=\{S_1(\vec e)\mid p(\vec e)\in \HB_P\setminus I\}=\{S_1(\vec e_i)\mid 1\leq i\leq k)\}$,
  $M_2=\{\neg S_2(\vec e)\mid p(\vec e)\in \HB_P\setminus I\}=\{\neg S_2(\vec e_i)\mid 1\leq i\leq k\}$ and $J=\{i\mid 1\leq i\leq k\}$.
  Since $A$ is monotonic, for any $J'\subseteq J$, we have that $I\cup \{p(\vec e_i)\mid i\in J'\}\models_OA$, i.e.,
  \[\{S_1(\vec e)\mid p(\vec e)\in I\}\cup \{S_1(\vec e_i)\mid i\in J'\}\cup \{\neg S_2(\vec e_i)\mid  i\in J\setminus J'\}\models O\supset Q(\vec t).\]
  It follows that
  \[\bigwedge_{p(\vec e)\in I}S_1(\vec e)\wedge \bigvee_{J'\subseteq J}\left(\bigwedge_{i\in J'}S_1(\vec e_i)\wedge \bigwedge_{i\in J\setminus J'}\neg S_2(\vec e_i)\right)\models O\supset Q(\vec t)\]
  which implies, by Lemma \ref{lem:6},
   \[\bigwedge_{p(\vec e)\in I}S_1(\vec e)\wedge \bigwedge_{i\in J}(S_1(\vec e_i)\vee \neg S_2(\vec e_i))\models O\supset Q(\vec t)\]
  i.e.,
   \[\bigwedge_{p(\vec e)\in I}S_1(\vec e)\wedge \bigwedge_{p(\vec e)\in \HB_P\setminus I}(S_1(\vec e)\vee \neg S_2(\vec e))\models O\supset Q(\vec t),\]
   and equivalently
     \[\bigwedge_{p(\vec e)\in I}S_1(\vec e)\wedge \bigwedge_{p(\vec e)\in \HB_P\setminus I}(S_1(\vec e)\vee \neg S_2(\vec e))\models \tau(O)\supset Q(\vec t),\]
   where $\approx$ is taken as a congruence relation.
  Consequently, $I\models_OA$ iff $\tau(O)\cup I\models\tau(A)$.
\end{proof}

We note that, in (i) of the above lemma, we can not replace ``$\tau(O)\cup I\models\tau(A)$" by
``$O\cup I\models\tau(A)$" since
$O\cup I\models A$ does not imply $\tau(O)\cup I\models A$.
For instance, let $O=\{a\approx b\}$, $I=\{p(a)\}$ and $A=p(b)$ where
$p$ is a predicate not belonging to the ontology and $\approx$ is equality. Then we have that
$\{a\approx b\}\cup \{p(a)\}\models p(b)$ as $\approx$ is an equality, but
$\tau(O)\cup \{p(a)\}\not\models p(b)$ as $\tau(O)=\{a\approx b\}$
with $\approx$ being a congruence relation;  as $p$ does
  not occur in $O$, no replacement axiom of $p$ is in $\tau(O)$.
%

\begin{lemma}\label{prop:DF:1}
  Let ${\cal K}=(O,P)$ be a dl-program and $I\subseteq\HB_P$ where $O$ is consistent
  and $\DL_P^?=\emptyset$. Then
  $\gamma_{{\cal K}^{s,I}}^i= E_i\cap\HB_P$ for any $i\ge 0$,
  where $E_i$ is defined as (\ref{eq:Ei:default}) for $\tau({\cal K})$ and $E=\Th(\tau(O)\cup  I)$.
\end{lemma}
\begin{proof}
  We prove this by induction on $i$.

  Base: If $i=0$ then it is obvious since $\tau(O)$ is consistent (as $O$ is consistent) and
  $E_0=\tau(O)$.

  Step: Suppose it holds for $i=n$. Now for any $h\in\HB_P$, $h\in\gamma^{n+1}_{{\cal K}^{s,I}}$
  if and only if there exists a dl-rule $(h\lto \Pos,\Not\Neg)$ in $P$ such that
  \begin{itemize}
    \item $\gamma^n_{{\cal K}^{s,I}}\models_OA$ for any $A\in\Pos$, and
    \item $I\not\models_OB$ for any $B\in\Neg$.
  \end{itemize}
  We have that \\
  (i) $I\not\models_OB$ \\
  iff $\tau(O)\cup I \not\models\tau(B)$ (by Lemma \ref{lem:DF:1} and $\DL_P^?=\emptyset$)\\
  iff $E\not\models\tau(B)$.\\
  (ii) $\gamma^n_{{\cal K}^{s,I}}\models_OA$\\
  iff $E_n\cap\HB_P\models_OA$ (by inductive assumption) \\
  iff $\tau(O)\cup E_n\cap\HB_P\models\tau(A)$ (by Lemma \ref{lem:DF:1} and $\DL_P^?=\emptyset$)\\
  iff $E_n\models\tau(A)$ (since $\tau(O)\subseteq E_n\subseteq\Th(\tau(O)\cup\HB_P)$).

  Consequently we have $\gamma^i_{{\cal K}^{s,I}}=E_i\cap\HB_P$ for any $i\ge 0$.
\end{proof}

\begin{theorem}\label{thm:3:in}
  Let ${\cal K}=(O,P)$ be a dl-program such that $\DL_P^?=\emptyset$ and $I\subseteq \HB_P$ .
  If $O$ is consistent then $I$ is a strong answer set of $\cal K$ if and only if $E=\Th(\tau(O)\cup I)$
  is an extension of $\tau(\cal K)$.
\end{theorem}
\begin{proof}
  $(\Rto)$ It suffices to show $E=\bigcup_{i\ge 0}E_i$ where $E_i$ is defined as (\ref{eq:Ei:default}) for
  $\tau(\cal K)$ and $E$.  \\
  $E=\Th(\tau(O)\cup I)$\\
  $\Rto E\equiv \tau(O)\cup\gamma^\infty_{{\cal K}^{s,I}}$ (since $I=\gamma^\infty_{{\cal K}^{s,I}}$)\\\
  $\Rto E\equiv \tau(O)\cup \bigcup_{i\ge 0}E_i\cap\HB_P$ (by Lemma \ref{prop:DF:1})\\
  $\Rto E\equiv \bigcup_{i\ge 0}E_i\cap\HB_P\cup \tau(O)$\\
  $\Rto E\equiv \bigcup_{i\ge 0}E_i$ (since $\tau(O)\subseteq E_i\subseteq\Th(\tau(O)\cup\HB_P)$)\\
  $\Rto E=\bigcup_{i\ge 0}E_i$\\
  $\Rto$ $E$  is an extension of $\tau(\cal K)$.

  $(\Lto)$ $E$ is an extension of $\tau(\cal K)$\\
  $\Rto\ E=\bigcup_{i\ge 0}E_i$ where $E_i$ is defined as (\ref{eq:Ei:default}) for $\tau(\cal K)$ and $E$\\
  $\Rto\ \Th(\tau(O)\cup I)=\bigcup_{i\ge 0}E_i$\\
  $\Rto\ \Th(\tau(O)\cup I)\cap\HB_P=\left(\bigcup_{i\ge 0}E_i\right)\cap \HB_P$\\
  $\Rto\ I=\bigcup_{i\ge 0}(E_i\cap\HB_P)$\\
  $\Rto\ I=\gamma^\infty_{{\cal K}^{s,I}}$ (by Lemma \ref{prop:DF:1})\\
  $\Rto\ I=\lfp(\gamma_{{\cal K}^{s,I}})$\\
  $\Rto$ $I$ is a strong answer set of $\cal K$.
\end{proof}

Since dl-programs can be translated
into ones without nonmonotonic dl-atoms according to Theorem
\ref{thm:delete:ominus:s}, we immediately have the following:
\begin{corollary}\label{thm:4:in}
  Let ${\cal K}=(O,P)$ be a dl-program and $I\subseteq\HB_P$. If  $O$
  is consistent then
  $I$ is a strong answer set of $\cal K$ if and only if $\Th(\tau(O)\cup\pi(I))$ is an extension of $\tau(\pi({\cal K}))$.
\end{corollary}
\begin{proof}
  $I$ is a strong answer set of $\cal K$\\
  iff $\pi(I)$ is a strong answer set of $\pi({\cal K})$ (by Theorem \ref{thm:delete:ominus:s})\\
  iff $\Th(\tau(O)\cup\pi(I))$ is an extension of $\tau(\pi(\cal K))$ (by Theorem \ref{thm:3:in}).
\end{proof}

%

Although the translation $\tau$ given here is kind of ``standard", as it
draw ideas from
\cite{GelfondLifschitz91} and \cite{Motik:JACM:2010}, there are a number of
subtleties in dealing with dl-programs
which make it non-trivial, in addition to the problem of equality.

In translating dl-programs to MKNF knowledge bases,
\citeN{Motik:JACM:2010} did not
consider dl-atoms containing the constraint operator. In addition, there is an essential difference in that their
approach does not work here as illustrated by the next example.
\begin{example}\label{exam:5}
  Let ${\cal K}=(O,P)$ where $O=\{S(b)\}$, $b$  an individual in the description logic
  but not a constant occurring in $P$, and $P$ consist of
  \[p(a)\lto \DL[S\ominus p, S\odot p; S](a).\]
  It is trivial that $\HB_P=\{p(a)\}$ and there is no interpretation of $\cal K$ satisfying the dl-atom $\DL[S\ominus p, S\odot p;S](a)$,
  thus it is monotonic and then the unique strong answer set of $\cal K$ is $\emptyset$.  In terms of Motik and Rosati's translation, we would have the default theory $\Delta=(\{d\},O)$ where
  \[d=\frac{(\forall x .(p(x)\supset \neg S(x))\wedge \forall x.(\neg p(x)\supset \neg S(x)))\rto S(a):}{p(a)}.\]
  Since the sentence $\forall x .(p(x)\supset \neg S(x))\wedge \forall x.(\neg p(x)\supset \neg S(x))$ is classically equivalent to
  $\forall x. \neg S(x)$, the unique extension of $\Delta$ is $\Th(\{S(b),p(a)\})$; when restricted to $\HB_P$, it is $\{p(a)\}$ which
  corresponds to no answer set of $\cal K$ at all. It is not difficult to check that the
  default theory $\tau(\cal K)$ has a unique extension $\Th(\{S(b)\})$ which corresponds to the strong
  answer set $\emptyset$ of $\cal K$.
\end{example}

Another subtle point is that the default translation alone may not capture the semantics of a dl-program.
If a dl-program $\cal K$ mentions nonmonotonic dl-atoms then
it is possible that $\tau(\cal K)$ has some
extensions that do not correspond to any strong answer sets of $\cal K$.
\begin{example}\label{exam:6}
  Let ${\cal K}=(O,P)$ where $O=\emptyset$ and $P$ consists of
  \begin{align*}
     &p(a) \lto q(a),\\
     &q(a) \lto \DL[S_1\oplus p, S_2\ominus q; S_1\sqcup \neg S_2](a).
  \end{align*}
  It is not difficult to check that $A=\DL[S_1\oplus p, S_2\ominus q; S_1\sqcup \neg S_2](a)$
  is nonmonotonic and $\cal K$ has a unique strong answer set $\{p(a),q(a)\}$. But note that the
  default theory $\tau({\cal K})=(D,W)$ where $W={\cal A}_O$ and $D$ consists of
  \begin{align*}
    \frac{q(a):}{p(a)},\qquad
    \frac{(p(a)\supset S_1(a))\wedge (\neg q(a)\supset \neg S_2(a))\supset (S_1(a)\vee\neg S_2(a)):}{q(a)}
  \end{align*}
  has a unique extension  $\Th(W)$ which does not correspond to any strong
  answer set of $\cal K$. However, if we apply the translation $\pi$
  to $\cal K$ first, we will have the dl-program $\pi({\cal K})=(O,\pi(P))$,
  where $\pi(P)$ consists of
  \begin{align*}
    & p(a) \lto q(a), \qquad q(a) \lto \Not \pi_A, \qquad \pi_q(a)\lto \Not q(a),\\
    & \pi_A \lto \Not \DL[S_1\oplus p, S_2\odot \pi_q; S_1\sqcup\neg S_2](a).
  \end{align*}
  It is tedious but not difficult to check that the unique strong answer set of $\pi(\cal K)$ is $\{p(a),q(a)\}$. When
  we apply the translation $\tau$ to $\pi(\cal K)$, we have the default theory $\tau(\pi({\cal K}))=(D',W')$ where
  $W'={\cal A}_O$ and $D'$ consists of
  \begin{align*}
      \frac{q(a):}{p(a)}, \qquad \frac{:\neg \pi_A}{q(a)},\qquad \frac{:\neg q(a)}{\pi_q(a)},\\
     \frac{:\neg [(p(a)\supset S_1(a))\wedge (\pi_q(a)\supset \neg S_2(a))\supset (S_1(a)\vee \neg S_2(a))]}{\pi_A}.
  \end{align*}
  The interested reader can verify that the unique extension of $\tau(\pi({\cal K}))$ is $\Th(\tau(O)\cup\{p(a),q(a)\})$,
  which corresponds to the unique strong answer set of $\cal K$.
\end{example}

We note that the translation $\tau$ does not preserve weak
answer sets of a normal dl-program, as shown by $\tau({\cal K}_2)$
in Example \ref{exam:4}, not even for canonical dl-programs, as
shown by $\tau({\cal K}_1)$ in Example \ref{exam:4}.

To preserve the weak answer sets of a dl-program, one may attempt to
``shift" $\tau(.)$ from premise to justification of a default in the
translation $\tau$; however, this does not work.  Consider the dl-program
${\cal K}=(\emptyset, P)$ where $P=\{p(a)\lto \DL[S\oplus p,S](a)\}$. Under the suggestion, we would have obtained the default
theory $\Delta=(D,W)$, where  $W=\tau(\emptyset)$ and $D$ consists of
\[\frac{:(p(s)\supset S(a))\supset S(a)}{p(a)}.\]
It is clear that $\Delta$ has a unique extension $\Th(\tau(\emptyset)\cup\{p(a)\})$, but we know that $\cal K$ has two weak answer sets,
$\emptyset$ and $\{p(a)\}$. This issue can be addressed by a translation which makes all dl-atoms occur negatively.

\begin{definition}[$\sigma({\cal K})$]
Let $r$ be a dl-rule of the form (\ref{dl:rule}). We define $\sigma(r)$ to be the rule
\[A\lto \Not\sigma(B_1),\ldots,\Not\sigma(B_m),\ldots,\Not B_{m+1},\ldots,\Not B_n\]
where $\sigma(B)=\sigma_B$ if $B$ is a dl-atom, and $B$ otherwise,
where $\sigma_B$ is a fresh propositional atom. For every  dl-program
${\cal K}=(O,P)$, we define  $\sigma({\cal K})=(O,\sigma(P))$ where
$\sigma(P)$ consists of the rules in $$\{\sigma(r)\mid r\in P\}\cup\{\sigma_B\lto\Not B\mid B\in \DL_P\}.$$
\end{definition}
\begin{example}\label{exam:8}
Let us consider the above dl-program ${\cal K}=(O,P)$ where $O=\emptyset$ and $P=\{p(a)\lto \DL[S\oplus p,S](a)\}$. We have that $\sigma({\cal K})=(O,\sigma(P))$ where $A= \DL[S\oplus p,S](a)$ and $\sigma(P)$ consists of the below two dl-rules:
\begin{align*}
   p(a)\lto \Not \sigma_A,\qquad \qquad \sigma_A\lto\Not \DL[S\oplus p,S](a).
\end{align*}
It is easy to see that $\sigma(\cal K)$ has two weak answer sets $\{\sigma_A\}$ and $\{p(a)\}$.
\end{example}

\begin{proposition}\label{prop:6}
  Let ${\cal K}=(O,P)$ be a dl-program and $I\subseteq\HB_P$.
  Then
  $I$ is a weak answer set of $\cal K$
  iff $I'$ is a weak answer set of $\sigma(\cal K)$ where $I'=I\cup\{\sigma_B\mid B\in\DL_P\ \textrm{and}\ I\not\models_OB\}$.
\end{proposition}
\begin{proof}
  As $\sigma_B\in I'$ iff $I\not\models_OB$ for any $B\in\DL_P$, we have that $wP_O^I\subseteq w[\sigma(P)]_O^{I'}$ and
  for any rule $(h\lto\Pos)$ in $w[\sigma(P)]_O^{I'}\setminus wP_O^I$, $\Pos=\emptyset$ and $h$ has the form $\sigma_B$ for some $B\in\DL_P$. Thus we have $I'\setminus I=\lfp(\gamma_{[{\sigma(\cal K)}]^{w,I}})\cap \{\sigma_B\mid B\in\DL_P\}$ and $\lfp(\gamma_{{\cal K}^{w,I}})\cup (I\setminus I')=\lfp(\gamma_{[\sigma({\cal K})]^{w,I'}})$. This completes the proof.
\end{proof}

\begin{proposition}\label{prop:7}
  Let ${\cal K}=(O,P)$ be a dl-program such that $O$ is consistent,
  $\DL_P^?=\emptyset$ and all dl-atoms occur negatively in $P$, i.e.,
  for any rule $(h\lto \Pos,\Not\Neg)$ of $P$, there is no dl-atom in
  $\Pos$.
  Then an interpretation $I\subseteq\HB_P$ is a weak answer set of $\cal K$ iff $E=\Th(I\cup \tau(O))$ is an extension of $\tau(\cal K)$.
\end{proposition}
\begin{proof}
  By Lemma \ref{lem:DF:1}, we can inductively prove $\gamma_{{\cal
      K}^{w,I}}^i=E_i\cap \HB_P$ for any $i\ge 0$ where $E_i$ is
  defined as (\ref{eq:Ei:default}) for $E$ and $\tau(\cal K)$.
   The remainder of the proof is similar to the one of Theorem \ref{thm:3:in}.
\end{proof}

Together with Theorem \ref{thm:delete:ominus:w}, the above two propositions imply a translation from dl-programs with consistent ontologies under the weak answer set semantics  to default theories.
\begin{corollary}
  Let ${\cal K}=(O,P)$ be a dl-program where $O$ is consistent. The below conditions are equivalent to each other:
  \begin{enumerate}[(i)]
    \item An interpretation $I\subseteq\HB_P$ is a weak answer set $\cal K$.
    \item $\Th(\tau(O)\cup \pi(I'))$ is an extension of $\tau(\pi(\sigma(\cal K)))$ where $I'=I\cup\{\sigma_B\mid B\in\DL_P\ \textit{and}\ I\not\models_OB\}$.
  \end{enumerate}
\end{corollary}

One can easily see that the translation $\sigma\cdot\pi\cdot\tau$, i.e., applying $\sigma$ firstly then $\pi$ and finally $\tau$, is polynomial. Thus, under the weak answer set semantics, we obtain a polynomial, faithful and modular translation from dl-programs with consistent ontologies to default theories.

\subsection{Handling inconsistent ontology knowledge bases}

A dl-program may have nontrivial strong answer sets even if its
ontology knowledge base is inconsistent. For instance, let ${\cal
  K}=(O,P)$, where $O=\{S(a),\neg S(a)\}$ and $P=\{p\lto \Not q, ~
q\lto\Not p\}$. Obviously $\cal K$ has two strong answer sets, $\{p\}$
and $\{q\}$, while the translation introduced in the last subsection,
$\tau(\cal K)$, yields a unique extension which is inconsistent. In
combining different knowledge bases, it is highly desirable that the
whole system is not trivialized due to the imperfection of a
subsystem. For dl-programs, this feature is naturally built into the
strong answer set semantics. When considering embedding, it is
important that this feature be preserved.


In Theorem \ref{thm:3:in} and Corollary \ref{thm:4:in}, we require $O$ to be
consistent  and we assume a limited congruence rewriting, i.e., the equality $\approx$ is understood as a congruence and the congruence is applied only to the predicates of underlying description logic.   To relax
these conditions, we propose the following translation $\tau'$ which is
slightly different from $\tau$.

\begin{definition}
Given a
dl-program ${\cal K}=(O,P)$, $\tau'(\cal K)$ is the default theory
$(D,\emptyset)$, where $D$ is the same as the one in the definition of
$\tau$  except for dl-atoms. Suppose $A$ is a dl-atom of the form
(\ref{dl:atom}). We define    $\tau'(A)$ to be the first-order
sentence:
\[
\Big[O\wedge\Big(\bigwedge_{1\leq i\leq
    m}\tau(S_i\ op_i\ p_i)\Big)\Big]\supset Q(\vec t)
\]
where $O$ is identified with its corresponding first-order theory in
which we do not require equality to be a congruence.
\end{definition}
Evidently, given a dl-program $\cal K$, every extension of
$\tau'(\cal K)$ is consistent and has the form
$\Th(I)$ for some $I\subseteq\HB_P$. 

\begin{example} Let ${\cal K}=(O,P)$ be a dl-program where $O=\{S(a), \neg S'(a),S\sqsubseteq S'\}$ and $P$
consists of $p(a)\lto \DL[S\oplus p;\neg S](a)$. It is evident that
$O$ is inconsistent and $\cal K$ has a unique strong answer set
$\{p(a)\}$. Now we have that the corresponding first-order theory of
$O$ is $S(a)\wedge\neg S'(a)\wedge (\forall x. S(x)\supset S'(x))$,
and $\tau'({\cal K})=(\{d\},\emptyset)$ where
\[d=\frac{(O\wedge (p(a)\supset S(a)))\supset \neg S(a):}{p(a)}.\]
It is not difficult to verify that $E=\Th(\{p(a)\})$ is the unique
extension of $\tau'(\cal K)$ which is consistent, while the unique
extension of $\tau(\cal K)$ is inconsistent.
\end{example}

Different from $\tau$ in another aspect, the translation $\tau'$ keeps equality as equality. For instance,
{\color{red} for} the dl-program $\cal K$ in Section~\ref{sec:equality}, we have that $\tau'({\cal K})=(D,\emptyset)$ where $D=\{
  \frac{:\neg p(a)}{p(b)},  \frac{:\neg p(b)}{p(a)}\}$.
Evidently, the default theory $\tau'(\cal K)$ has two extensions $\Th(\{p(a)\})$ and $\Th(\{p(b)\})$.

The translation $\tau'$ is obviously modular. We will show below that it is faithful.
\begin{lemma}\label{lem:DF:9}
  Let ${\cal K}=(O,P)$ be a dl-program, $A$ an atom or a monotonic dl-atom and $I\subseteq\HB_P$.
Then
$I\models_OA$ if and only if $I \vdash\tau'(A)$.
\end{lemma}
\begin{proof}
   The conclusion is evident if $A$ is an atom or $O$ is inconsistent. 
  Suppose $A$ is a dl-atom and $O$ is consistent.
  Let $A=\DL[\lambda;Q](\vec t)$. Thus $\tau(A)$ is of the form $\psi\supset Q(\vec t)$ which implies $\tau'(A)\equiv(O\wedge \psi)\supset Q(\vec t)$. We have that\\
  $I\models_OA$\\
  iff   $\tau(O)\cup I\vdash\tau(A)$ (by Lemma (ii) of \ref{lem:DF:1}, where $\approx$ is taken as a congruence relation)\\
  iff $I\vdash \tau(O)\supset\tau(A)$\\
  iff $I\vdash \tau(O)\supset (\psi\supset Q(\vec t))$\\
  iff $I\vdash (\tau(O)\wedge\psi)\supset Q(\vec t)$\\
  iff $I\vdash (O\wedge\psi)\supset Q(\vec t)$ (by Theorem \ref{thm:fitting}, where $\approx$ is taken as  equality)\\
  iff $I\vdash\tau'(A)$.
\end{proof}
\begin{lemma}\label{prop:DF:3}
  Let ${\cal K}=(O,P)$ be a dl-program such that $\DL_P^?=\emptyset$, $I\subseteq\HB_P$ and  $E=\Th(I)$.
  Then
  $\gamma_{{\cal K}^{s,I}}^i= E_i\cap\HB_P$ for any $i\ge 0$,
  where $E_k$ is defined as (\ref{eq:Ei:default}) for $\tau'({\cal K})$ and $E$.
\end{lemma}
\begin{proof}
  We prove this by induction on $k$.

  Base: It is obvious for $i=0$ since  $E_0=\emptyset$.

  Step: Suppose it holds for $i=n$. For any $h\in\HB_P$, $h\in\gamma^{n+1}_{{\cal K}^{s,I}}$
  if and only if there exists a dl-rule $(h\lto \Pos,\Not\Neg)$ such that
  \begin{itemize}
    \item $\gamma^n_{{\cal K}^{s,I}}\models_OA$ for any $A\in\Pos$, and
    \item $I\not\models_OB$ for any $B\in\Neg$.
  \end{itemize}
   We have that\\
   (i) $I\not\models_OB$ \\
  iff $I\not\models\tau'(B)$ (by Lemma \ref{lem:DF:9})\\
  iff $E\not\models\tau'(B)$ .\\
  (ii) $\gamma_{{\cal K}^{s,I}}^n\models_OA$\\
  iff $E_n\cap\HB_P\models_OA$ (by the inductive assumption)\\
  iff $E_n\cap\HB_P\models\tau'(A)$ (by Lemma \ref{lem:DF:9})\\
  iff $E_n\models\tau'(A)$.%

  It follows that $h\in\gamma^{n+1}_{{\cal K}^{s,I}}$ if and only if $h\in E_{n+1}$.
  Consequently  $\gamma^i_{{\cal K}^{s,I}}=E_i\cap\HB_P$ for any $i\ge 0$.
\end{proof}

In the next theorem and corollary, we present the main results of this section, which extend
Theorem
\ref{thm:3:in} and Corollary \ref{thm:4:in} respectively.
\begin{theorem}\label{thm:6}
  Let ${\cal K}=(O,P)$ be a dl-program such that $\DL_P^?=\emptyset$ and $I\subseteq \HB_P$.
  Then
  $I$ is a strong answer set of $\cal K$ if and only if $E=\Th(I)$
  is an extension of $\tau'(\cal K)$.
\end{theorem}%
\begin{proof}
  $(\Rto)$ It is sufficient to show $E=\bigcup_{i\ge 0}E_i$
  where $E_i$ is defined as (\ref{eq:Ei:default}) for $\tau'(\cal K)$ and $E$. \\
\quad  $E=\Th(I)$\\
\quad   $\Rto E\equiv I$\\
\quad   $\Rto E \equiv \gamma^\infty_{{\cal K}^{s,I}}$ (since $I$ is a strong answer set of $\cal K$)\\
\quad   $\Rto E\equiv\bigcup_{i\ge 0}E_i\cap\HB_P$ (by Lemma \ref{prop:DF:3})\\
\quad   $\Rto E\equiv\bigcup_{i\ge 0}E_i$ (since $E_i\subseteq\Th(\HB_P)$)\\
\quad   $\Rto E=\bigcup_{i\ge 0}E_i$\\
\quad   $\Rto$ $E$ is an extension of $\tau'(\cal K)$.

\noindent $(\Lto)$ $E$ is an extension of $\tau'(\cal K)$\\
\quad   $\Rto\ E=\bigcup_{i\ge 0}E_i$ where $E_i$ is defined as (\ref{eq:Ei:default}) for $\tau'(\cal K)$ and $E$\\
\quad   $\Rto\ \Th( I)=\bigcup_{i\ge 0}E_i$\\
\quad   $\Rto\ \Th(I)\cap\HB_P=\left(\bigcup_{i\ge 0}E_i\right)\cap \HB_P$\\
\quad   $\Rto\ I=\bigcup_{i\ge 0}(E_i\cap\HB_P)$\\
\quad   $\Rto\ I=\gamma^\infty_{{\cal K}^{s,I}}$ (by Lemma~\ref{prop:DF:3})\\
\quad   $\Rto\ I=\lfp(\gamma_{{\cal K}^{s,I}})$\\
\quad   $\Rto$ $I$ is a strong answer set of $\cal K$.
\end{proof}

\begin{corollary}\label{thm:7}
  Let ${\cal K}=(O,P)$ be a dl-program and $I\subseteq\HB_P$. Then
  $I$ is a strong answer set of $\cal K$ if and only if $\Th(\pi(I))$ is an extension of $\tau'(\pi({\cal K}))$.
\end{corollary}
\begin{proof}
  $I$ is a strong answer set of $\cal K$\\
  iff $\pi(I)$ is a strong answer set of $\pi(\cal K)$ (by Theorem \ref{thm:delete:ominus:s})\\
  iff $\Th(\pi(I))$ is an extension of  $\tau'(\pi(\cal K))$ (by Theorem \ref{thm:6}).
\end{proof}
%

Note that, for the dl-program $\cal K$ in Example \ref{exam:8}, we have
$\tau'({\cal K})=(D,\emptyset)$ where $D$ consists of
\[\frac{(p(a)\supset S(a))\supset S(a):}{p(a)}, \qquad \frac{:\neg p(a)}{\neg p(a)}.\]
It is easy to see that $\Th(\{\neg p(a)\})$ is the unique extension of $\tau'(\cal K)$. As $\cal K$ has two
weak answer sets $\emptyset$ and $\{p(a)\}$, the translation $\tau'$ alone does not preserve weak answer sets of dl-programs. However, one can further check that $\tau'(\sigma(\cal K))$ has exact two extensions
$\Th(\{p(a)\})$ and $\Th(\{\neg p(a),\sigma_A\})$.

We show below that, combining with the translation $\sigma$, the translation $\tau'$ actually preserves weak answer sets.

\begin{proposition}\label{prop:8}
  Let ${\cal K}=(O,P)$ be a dl-program such that $O$ is consistent,
  $\DL_P^?=\emptyset$ and all dl-atoms occurs negative in $P$, i.e.,
  there for any rule $(h\lto \Pos,\Not\Neg)$ of $P$, there is no
  dl-atom in $\Pos$.
  Then an interpretation $I\subseteq\HB_P$ is a weak answer set of $\cal K$ iff $E=\Th(I)$ is an extension of $\tau'(\cal K)$.
\end{proposition}
\begin{proof}
The proof is similar to the one of Proposition~\ref{prop:7}.
\end{proof}

Together with Theorem \ref{thm:delete:ominus:w}, the propositions \ref{prop:6} and \ref{prop:8} imply a translation from dl-programs with consistent ontologies under the weak answer set semantics  to default theories.
\begin{corollary}
  Let ${\cal K}=(O,P)$ be a dl-program where $O$ is consistent. The
  following conditions are equivalent: 
  \begin{enumerate}[(i)]
    \item The interpretation $I\subseteq\HB_P$ is a weak answer set $\cal K$.
    \item $\Th(\pi(I'))$ is an extension of $\tau'(\pi(\sigma(\cal K)))$ where $I'=I\cup\{\sigma_B\mid B\in\DL_P\ \textit{and}\ I\not\models_OB\}$.
  \end{enumerate}
\end{corollary}

 Since there are no dl-atoms that occur positively in $\sigma(\cal
 K)$,  the translation $\sigma\cdot\pi$, i.e., applying $\sigma$ first
 and then $\pi$,  is polynomial. Consequently the combination
 $\sigma\cdot\pi\cdot \pi'$ is polynomial as well.
 Therefore, we have a polynomial, faithful and modular translation from dl-programs under the weak answer set semantics to default theories.

\subsection{Under the well-supported semantics}

To avoid circular justifications in some weak and strong answer sets
of dl-programs, recently well-supported semantics for dl-programs was
proposed \cite{Yi-Dong:IJCAI:2011}. In what follows, we will
show that, under the weakly well-supported answer set semantics,
dl-programs can be translated into default theories by an extension of the
translation $\tau$ above.
In particular,
the translation is polynomial, faithful and modular. Let us recall the
basic notions and notations of well-supported semantics below.

Let ${\cal K}=(O,P)$ be a dl-program, $E$ and $I$ two sets of atom in $\HB_P$ with $E\subseteq I$.
The notion that {\em $E$ up to $I$ satisfies an atom (or a dl-atom, or
  their negation by default) $l$ under $O$}, written $(E,I)\models_O
l$, is as follows:
\begin{itemize}
  \item $(E,I)\models_O p$ if $p\in E$; $(E,I)\models_O\Not p$ if $p\notin I$, where $p$ is an atom;
  \item $(E,I)\models_O A$ if for  every $F$ with $E\subseteq F\subseteq I$, $F\models_O A$; $(E,I)\models_O\Not A$
  if there is no $F$ with $E\subseteq F\subseteq I$ such that $F\models_OA$, where $A$ is a dl-atom.
\end{itemize}
The notion ``{\em up to satisfaction}'' is extended for a set of atoms
dl-atoms, and their negation by default in a standard manner\footnote{The notion of ``up to satisfaction" is very similar to that of ``conditional satisfaction" in logic programs with abstract constraints \cite{Son:JAIR2007}.}.
The operator ${\cal T}_{\cal K}: (2^{\HB_P}\times 2^{\HB_P})\rto 2^{\HB_P}$ is defined as:
\[
{\cal T}_{\cal K}(E,I)=\{a \mid (a\lto\Body)\in P\ \textrm{and}\
(E,I)\models_O\Body\}
\]
where $E\subseteq I$. It has been shown that if $I$ is a model of $\cal K$, then the operator is monotone
in the sense that for every, $E_1\subseteq E_2\subseteq I$, ${\cal T}_{\cal K}(E_1,I)\subseteq{\cal T}_{\cal K}(E_2,I)$.
As the operator is also continuous in this sense (thanks to
compactness of answering DL queries), for any model $I$ of $\cal K$ the monotone sequence $\langle {\cal T}^i_{\cal K}(\emptyset,I)\rangle_i^\infty$, where
${\cal T}^0_{\cal K}(\emptyset,I)=\emptyset$, ${\cal T}^{i+1}_{\cal
  K}(\emptyset,I)={\cal T}_{\cal K}({\cal T}^i_{\cal
  K}(\emptyset,I),I)$,  $i\geq 0$,  converges to a fixpoint
denoted ${\cal T}^\infty_{\cal K}(\emptyset, I)$.
%

In the rest of this paper, for convenience we will use the term {\em level mapping justification} to refer to the existence of such a fixpoint, borrowing a concept from a similar characterization for normal logic programs \cite{Fages:JMLCS:1994} as well as for weight constraint programs \cite{DBLP:Liu:FI:2010}.

A model $I$ of $\cal K$ is a {\em weakly (\textrm{resp}. strongly) well-supported answer set} of $\cal K$ if
$I$ coincides with the  fixpoint ${\cal T}^\alpha_{{\cal K}^I}(\emptyset, I)$ (resp. ${\cal T}^\alpha_{\cal K}(\emptyset,I)$, where
${\cal K}^I=(O,P^I)$ and
\[
P^I=\{a\lto \Pos \mid (a\lto \Pos,\Not\Neg)\in P\ \textrm{and}\
I\not\models_OB\mbox{ for every $B\in\Neg$}\}.
\]

As the next proposition shows, the strongly well-supported answer set
semantics coincides with the strong answer set semantics for the
dl-programs that mention no nonmonotonic dl-atoms.

\begin{proposition}\label{prop:9}
Let ${\cal K}=(O,P)$ be a dl-program with $\DL_P^?=\emptyset$ and
$I\subseteq\HB_P$ a model of ${\cal K}$. Then
  $I$ is a strong answer set of $\cal K$ iff $I$ is a strongly well-supported
  answer set of $\cal K$.
\end{proposition}%
\begin{proof}
  $(\Lto)$ This direction is obvious since, for any dl-program, each strongly well-supported answer set is a weakly well-supported answer sets (Corollary 3 of \cite{Yi-Dong:IJCAI:2011}) and each weakly well-supported answer set is a strong answer set (Theorem 6 in \cite{Yi-Dong:IJCAI:2011}).

  $(\Rto)$ It suffices to show $I\subseteq{\cal T}_{\cal K}^\alpha(\emptyset, I)$. Since $I=\gamma^\infty_{{\cal K}^{s,I}}$. We only need to show inductively, $\gamma^n_{{\cal K}^{s,I}}\subseteq {\cal T}^n_{{\cal K}}(\emptyset, I)$ for any $n\ge 0$.

  Base: it is evident for $n=0$.

  Step: Let us consider the case $n+1$. For any atom $h\in\gamma^{n+1}_{{\cal K}^{s,I}}$, there must exist a rule $(h\lto\Pos,\Not\Neg)$ in $P$ s.t.
  \begin{itemize}
    \item $\gamma^n_{{\cal K}^{s,I}}\models_OA$ for any $A\in\Pos$ since $\DL_P^?=\emptyset$, and
    \item $I\not\models_OB$ for any $B\in\Neg$
  \end{itemize}
  Note that all dl-atoms in $P$ are monotonic. It follows that $(\gamma^n_{{\cal K}^{s,I}},I)\models_OA$
  for any $A\in\Pos$ and thus $({\cal T}^n_{\cal K}(\emptyset, I),I)\models_OA$ by the inductive assumption. On the other hand, since $I\not\models_OB$ and $B$ is monotonic, we have that, $I'\not\models_OB$ for any $I'\subseteq I$. It implies $(\emptyset, I)\models_O\Not B$ and thus $({\cal T}^n_{\cal K}(\emptyset, I),I)\models_O\Not B$.
  Consequently $h\in {\cal T}^{n+1}_{\cal K}(\emptyset, I)$ and then $I\subseteq {\cal T}_{\cal K}^\alpha(\emptyset, I)$. It follows  that $I$ is a strongly well-supported answer set of ${\cal K}$.
\end{proof}

Before presenting a translation under weakly well-supported answer set semantics, let us reconsider the dl-program $\cal K$ in Example \ref{exam:6}.
Recall that the dl-program $\cal K$ has a strong answer set $\{p(a),q(a)\}$ and the unique extension of
$\tau(\cal K)$ is $\Th(\tau(\emptyset))$. Actually, $\emptyset$ is not a model of $\cal K$ at all. We
can check that $\cal K$ has neither a weakly well-supported answer
set, nor a strongly well-supported answer set.
Thus the translation $\tau$ works neither for weakly nor for strongly well-supported answer set semantics
of dl-programs.

Surprisingly, a small addition to our default logic encoding will result in a one-one correspondence between the weakly
well-supported answer sets of a dl-program and the corresponding default extensions, for arbitrary dl-programs. Below, we consider the dl-programs whose ontology component is consistent. Formally,
given a dl-program ${\cal K}=(O,P)$ where $O$ is consistent, we define $\tau^*({\cal K})=(D,W)$ where
$\tau^*$ is exactly the same as $\tau$ except that $D$ includes, for each $p(\vec c)\in \HB_P$, the default
\[\frac{:\neg p(\vec c)}{\neg p(\vec c)}.\]
It is evident that any extension $E$ of $\tau^*(\cal K)$ is equivalent
to $\tau(O)\cup I\cup \{ \neg \alpha \mid \alpha \in \HB_P\setminus I\}$
for some $I\subseteq\HB_P$.

\begin{example}
  Let us reconsider the dl-program $\cal K$ in Example \ref{exam:6}.
  The default theory $\tau^*({\cal K})=(D,\tau(\emptyset))$ where $D$ consists of the ones produced by $\tau$ and additionally the ones
  \[\frac{:\neg p(a)}{\neg p(a)},\qquad \frac{:\neg q(a)}{\neg q(a)}.\]
  It is not difficult to check that $\tau^*(\cal K)$ has no extension. This example
  also demonstrates that $\tau^*$ does not preserve the strong answer sets of dl-programs as
  $\cal K$ has a strong answer set $\{p(a),q(a)\}$.
\end{example}

In the following,  given a dl-program ${\cal K}=(O,P)$ and $I\subseteq\HB_P$, we denote $\overline I=\HB_P\setminus I$
and $\neg I=\{\neg \alpha \mid \alpha\in I\}$ for convenience.

\begin{lemma}\label{lem:forget:2}
  Let $M_1$ and $M_2$ be two sets of atoms such that $M_1\cap
  M_2=\emptyset$, $\psi_i, \var_i~(1\le i\le n)$ and $\phi$ are
  formulas not mentioning the predicate $p_1,p_2$ and the predicates
  occurring in $M_1\cup M_2$.
  Then
  \[\bigwedge M_1\wedge\bigwedge\neg M_2\wedge\bigwedge_{1\le i\le n}((p_1(\vec c_i)\supset\psi_i)\wedge (\neg p_2(\vec c_i)\supset\var_i))\models \phi\ \textit{iff}\
  \bigwedge_{p_1(\vec c_i)\in M_1}\psi_i\wedge\bigwedge_{p_2(\vec c_j)\in M_2}\var_j\models\phi.\]
\end{lemma}
\begin{proof}
  The direction from right to left is obvious. Let us consider the other direction.
  Suppose there is an interpretation ${\cal I}$ such that ${\cal I}\models \bigwedge_{p_1(\vec c_i)\in M_1}\psi_i\wedge\bigwedge_{p_2(\vec c_j)\in M_2}\var_j$ but ${\cal I}\not\models\phi$, by which we have
  ${\cal I}\not\models \bigwedge M_1\wedge\bigwedge\neg M_2\wedge\bigwedge_{1\le i\le n}((p_1(\vec c_i)\supset\psi_i)\wedge (\neg p_2(\vec c_i)\supset\var_i))$. It follows that
  ${\cal I}\not\models \bigwedge M_1\wedge\bigwedge\neg M_2\wedge\bigwedge_{p_1(\vec c_i)\notin M_1}(p_1(\vec c_i)\supset\psi_i)\wedge \bigwedge_{p_2(\vec c_j)\notin M_2}(\neg p_2(\vec c_j)\supset\var_j))$. We construct
  the interpretation ${\cal I}'$ that is same to ${\cal I}$ except that
  \begin{itemize}
    \item ${\cal I}'\models\bigwedge M_1$, and ${\cal I}'\models\bigwedge\neg M_2$,
    \item ${\cal I}'\not\models p_1(\vec c_i)$ for every $p_1(\vec c_i)\notin M_1$, and
    \item ${\cal I}'\models p_2(\vec c_j)$ for every $p_2(\vec c_j)\notin M_2$.
  \end{itemize}
  It is clear that ${\cal I}'\models\bigwedge_{p_1(\vec c_i)\in M_1}\psi_i\wedge\bigwedge_{p_2(\vec c_j)\in M_2}\var_j$ and ${\cal I}'\not\models\phi$. However, we have ${\cal I}'\models\phi$ by ${\cal I}'\models \bigwedge M_1\wedge\bigwedge\neg M_2\wedge\bigwedge_{1\le i\le n}((p_1(\vec c_i)\supset\psi_i)\wedge (\neg p_2(\vec c_i)\supset\var_i))$, a contradiction.
\end{proof}

\begin{lemma}\label{lem:9}
  Let ${\cal K}=(O,P)$ be a dl-program, $A=\DL[\lambda;Q](\vec t)$ a dl-atom and $I\subseteq\HB_P$.
  \begin{enumerate}[(i)]
    \item $I\models_OA$ iff $\tau(O)\cup I\cup\neg\overline
      I\models\tau(A)$.
    \item if $I'\subseteq I$ then $(I',I)\models_OA$ iff $\tau(O)\cup I'\cup\neg\overline I\models\tau(A)$.
  \end{enumerate}
\end{lemma}
\begin{proof}
  For clarity and without loss of generality, let $\lambda=(S_1\oplus p_1,S_2\ominus p_2)$.

%

  (i)
  We have that at first
  $\tau(O)\cup I\cup\neg \overline I\models\tau(A)$\\
  iff $\tau(O)\cup I\cup \neg \overline I\models
    (\bigwedge_{\vec e\in\vec {\cal C}}(p_1(\vec e)\supset S_1(\vec e)))\wedge
    (\bigwedge_{\vec e\in\vec {\cal C}}(\neg p_2(\vec e)\supset \neg S_2(\vec e)))\supset Q(\vec t)$\\
  iff $I\cup \neg \overline I\cup \{\bigwedge_{\vec e\in\vec {\cal C}}(p_1(\vec e)\supset S_1(\vec e))\}\cup
    \{\bigwedge_{\vec e\in\vec {\cal C}}(\neg p_2(\vec e)\supset \neg S_2(\vec e))\}\models \tau(O)\supset Q(\vec t)$\\
   iff $\{S_1(\vec e) \mid p_1(\vec e)\in I\}\cup\{\neg S_2(\vec
   e)\mid p_2(\vec e)\notin I\}
    \models \tau(O)\supset Q(\vec t)$ (By Lemma \ref{lem:forget:2})\\
  iff $\tau(O)\cup\{S_1(\vec e)\mid p_1(\vec e)\in I\}\cup\{\neg S_2(\vec e)\mid p_2(\vec e)\notin I\}\models Q(\vec t)$\\
  iff $O\cup\{S_1(\vec e)\mid p_1(\vec e)\in I\}\cup\{\neg S_2(\vec e)\mid p_2(\vec e)\notin I\}\models Q(\vec t)$ (by Theorem \ref{thm:fitting}, where $\approx$ is taken as equality)\\
  iff $I\models_OA$.

(ii) $(\Lto)$ By $\tau(O)\cup I'\cup\neg \overline I\models \tau(A)$, we have that, for any $F$ with $I'\subseteq F\subseteq I$, $\tau(O)\cup F\cup \neg\overline I\models\tau(A)$ which implies $\tau(O)\cup F\cup\neg\overline F\models\tau(A)$. Thus $F\models_OA$ by (i). Consequently $(I',I)\models_OA$.

\noindent  $(\Rto)$ Let $S=I\setminus I'=\{\alpha_1,\ldots,\alpha_k\}$ and $J=\{1,\ldots,k\}$. It is clear that $\neg S=\neg \overline {I'}\setminus \neg \overline I$. Note that for any $F$ with $I'\subseteq F\subseteq I$, $F\models_OA$, which implies $\tau(O)\cup F\cup\neg \overline F\models\tau(A)$ by (i), i.e., for any $J'\subseteq J$, we have that
  $$I'\cup \{\alpha_i\mid i\in J'\}\cup \{\neg\alpha_j\mid j\in J\setminus J'\}\cup\neg\overline I\models O\supset \tau(A)$$
  which implies that\\
  $$\bigvee_{J'\subseteq J}(\bigwedge_{i\in J'}\alpha_i\wedge\bigwedge_{j\in J\setminus J'}\neg \alpha_j)\models
  I'\wedge\neg\overline I\supset (\tau(O)\supset \tau(A)).$$
  Thus we have, by Lemma \ref{lem:6}
  $$\bigwedge_{i\in J}(\alpha_i\vee\neg\alpha_i)\models I'\wedge\neg\overline I\supset (\tau(O)\supset \tau(A))$$
  i.e.,
  $$ I'\cup \neg \overline I\models \tau(O)\supset \tau(A).$$
  Consequently we have $\tau(O)\cup I'\cup\neg \overline I\models\tau(A)$.
\end{proof}

It is easy to see that  if $A$ is an atom and $O$ is consistent, then both (i) and (ii) of the above lemma hold.

\begin{lemma}\label{lem:Ei:consistent}
  Let ${\cal K}=(O,P)$ be a dl-program where $O$ is consistent and
  $I\subseteq\HB_P$ is a model of ${\cal K}$. Then we
  have that, for any $i\ge 0$, $E_i$ is consistent where
  $E_i$ is defined as (\ref{eq:Ei:default}) for $\tau^*(\cal K)$ and $E=\Th(\tau(O)\cup I\cup\neg\overline I)$.
\end{lemma}
\begin{proof}
  It is sufficient to show that $E_i\cap\HB_P\subseteq I$ for every $i\ge 0$.

  Base: It is clear for $i=0$ since $O$ is consistent. For the case $i=1$, we have that $\neg\overline I\subseteq E_1$. If $E_1$ is inconsistent then there must exist a rule $(h\lto\Pos,\Not\Neg)$ in $P$ such that
  \begin{itemize}
    \item $h\in\overline I$,
    \item $E_0\models\tau(A)$ for every $A\in\Pos$, and
    \item $E\not\models\tau(B)$ for every $B\in\Neg$.
  \end{itemize}
  It is evident that $I\not\models_OB$ for every $B\in \Neg$ by (i) of Lemma \ref{lem:9}. And note that \\
  $E_0\models\tau(A)$\\
  $\Rto \tau(O)\models\tau(A)$\\
  $\Rto \tau(O)\cup I\cup\neg\overline I\models\tau(A)$\\
  $\Rto I\models_OA$ by (i) of Lemma \ref{lem:9}.

  It follows that $h\in I$ since $I$ is a model of $\cal K$. It contradicts with $h\in\overline I$.

  Step: Suppose $E_n$ is consistent where $n\ge 1$. For any atom $h\in\HB_P$, $h\in E_{n+1}$ if and only if there exists a
  rule $(h'\lto\Pos',\Not\Neg')$ in $P$ such that
  \begin{itemize}
    \item $E_n\models \tau(A')$ for any $A'\in\Pos'$, and
    \item $E\not\models\tau(B')$ for any $B'\in\Neg'$.
  \end{itemize}
  It is clear that $I\not\models_OB'$ for any $B'\in\Neg'$ by (i) of Lemma \ref{lem:9}. Since
  $E_n$ is consistent by the inductive assumption, we have that $(E_n\cap\HB_P)\cap \overline I=\emptyset$
  by $\neg\overline I\subseteq E_n$. Thus it follows that\\
  $E_n\models\tau(A')$\\
  $\Rto \tau(O)\cup (E_n\cap\HB_P)\cup \neg\overline I\models\tau(A')$\\
  $\Rto \tau(O)\cup I\cup\neg\overline I\models\tau(A')$ since $E_n\cap\HB_P\subseteq I$\\
  $\Rto I\models_OA'$ by (i) of Lemma \ref{lem:9}.

  It implies that $h\in I$ since $I$ is a model of $\cal K$. Thus $E_{n+1}$ is consistent.
\end{proof}

\begin{lemma}\label{lem:10}
  Let ${\cal K}=(O,P)$ be a dl-program where $O$ is consistent and $I\subseteq\HB_P$ a model of ${\cal K}$. Then we
  have that, for any $i\ge 0$,
  \begin{enumerate}[(i)]
    \item ${\cal T}^i_{{\cal K}^I}(\emptyset, I)\subseteq E_{i+1}\cap\HB_P$, and
    \item $E_i\cap\HB_P\subseteq {\cal T}^i_{{\cal K}^I}(\emptyset, I)$
  \end{enumerate} where
  $E_i$ is defined as (\ref{eq:Ei:default}) for $\tau^*(\cal K)$ and $E=\Th(\tau(O)\cup I\cup\neg\overline I)$.
\end{lemma}
\begin{proof}
  We prove (i) and (ii) by induction on $i$.

  (i)
  Base: It is evident for $i=0$.

  Step: Suppose it holds for $i=n$ where $n\ge 0$. For any atom $h\in\HB_P$, we have that
  $h\in {\cal T}^{n+1}_{{\cal K}^I}(\emptyset, I)$ if and only if there exists a rule
  $(h\lto\Pos,\Not\Neg)$ in $P$ such that
  \begin{itemize}
    \item $({\cal T}^n_{{\cal K}^I}(\emptyset, I),I)\models_OA$ for any $A\in\Pos$, and
    \item $I\not\models_OB$ for any $B\in\Neg$.
  \end{itemize}

  By (i) of Lemma \ref{lem:9}, $I\not\models_OB$ iff $E\not\models\tau^*(B)$, and by (ii) of Lemma \ref{lem:9},
  we have \\
  $({\cal T}^n_{{\cal K}^I}(\emptyset, I),I)\models_OA$ \\
  $\Rto \tau(O)\cup {\cal T}^n_{{\cal K}^I}(\emptyset, I)\cup\neg\overline I\models\tau(A)$\\
  $\Rto \tau(O)\cup (E_{n+1}\cap\HB_P)\cup\overline I\models\tau(A)$ (by the induction assumption)\\
  $\Rto E_{n+1}\models\tau(A)$ (since $\tau(O)\cup\neg\overline I\subseteq E_{n+1}$)\\
  $\Rto h\in E_{n+2}$.

  (ii) Base: It is clear for $i=0$. Let us consider the case $i=1$. For any atom $h\in E_1\cap\HB_P$,  there
  exists a rule $(h\lto\Pos,\Not\Neg)$ in $P$ such that
  \begin{itemize}
    \item $E_0\models\tau(A)$ for any $A\in\Pos$, and
    \item $E\not\models\tau(B)$ for any $B\in\Neg$.
  \end{itemize}
  By $E_0\models\tau(A)$, we have $O\models\tau(A)$. Thus $\tau(O)\cup I'\cup\neg\overline I\models\tau(A)$ for any $I'$ such that $I'\subseteq I$. It implies $(\emptyset, I)\models_OA$ by (ii) of Lemma \ref{lem:9}. By (i) of Lemma
  \ref{lem:9} and $E\not\models\tau(B)$, it is evident $I\not\models_OB$. It follows that $h\in {\cal T}^1_{{\cal K}^I}(\emptyset, I)$.

  Step: Suppose it holds for $i=n$ where $n\ge 1$. For any atom $h'\in(E_{n+1}\cap\HB_P)$, there exists
  a rule $(h'\lto\Pos',\Not\Neg')$ in $P$ such that
  \begin{itemize}
    \item $E_n\models\tau(A')$ for any $A'\in\Pos'$, and
    \item $E\not\models\tau(B')$ for any $B'\in\Neg'$.
  \end{itemize}
  Since $I$ is a model of $\cal K$, $E_n$ is consistent by Lemma \ref{lem:Ei:consistent}. Note that for any $n\ge 1$ and
  $\tau(O)\cup\neg\overline I\subseteq E_n$. It implies $E_n\cap\HB_P\subseteq I$. We have that\\
  $E_n\models\tau(A')$\\
  $\Rto O\cup (E_n\cap\HB_P)\cup\neg\overline I\models\tau(A')$\\
  $\Rto (E_n\cap\HB_P,I)\models_OA'$ by (ii) of Lemma \ref{lem:9}\\
  $\Rto ({\cal T}^n_{{\cal K}^I}(\emptyset, I),I)\models_OA'$ by the inductive assumption and the monotonicity of ${\cal T}_{{\cal K}^I}$.

  Notice again that $E\not\models\tau(B')$ implies $I\not\models_OB'$ by (i) of Lemma \ref{lem:9}. Thus
  it follows that
  $h'\in {\cal T}^{n+1}_{{\cal K}^I}(\emptyset, I)$.

This completes the proof.
\end{proof}

Please note that it does not generally hold that ${\cal T}^i_{{\cal K}^I}(\emptyset, I)=E_i\cap\HB_P$ in the above lemma. For instance, let us consider the dl-program ${\cal K}=(\emptyset, P)$ where
$P$ consists of $$p(a)\lto \DL[S\oplus p, S'\ominus q; S\sqcup\neg S'](a).$$

Let $I=\{p(a)\}$. It is obvious that $p(a)\in {\cal T}_{{\cal K}^I}(\emptyset, I)$, i.e. $p(a)\in {\cal T}^1_{{\cal K}^I}(\emptyset, I)$. However, it is clear that $E_0\not\models\tau(A)$ since $E_0=\Th(\tau(\emptyset))$ where $A=\DL[S\oplus p, S'\ominus q; S\sqcup\neg S'](a)$. Thus $p(a)\not\in E_1$.

The theorem below shows that the polynomial and modular translation $\tau^*$ preserves the weakly well-supported
answer set semantics of dl-programs. Thus it is faithful.
\begin{theorem}\label{thm:8}
  Let ${\cal K}=(O,P)$ be a dl-program where $O$ is consistent and $I\subseteq\HB_P$
  a model of $\cal K$. Then we have that $I$ is a
  weakly well-supported answer set of ${\cal K}$ iff $E=\Th(\tau(O)\cup I\cup\neg\overline I)$ is an extension
  of $\tau^*(\cal K)$.
\end{theorem}
\begin{proof}
  $(\Rto)$ To show $E=\bigcup_{i\ge 0}E_i$ where $E_i$ is defined as (\ref{eq:Ei:default}) for $E$ and $\tau^*(\cal K)$,
  it is sufficient to show $E\cap\HB_P=(\bigcup_{i\ge 0}E_i)\cap\HB_P$ since $\tau(O)= E_0$, $\neg\overline I\subseteq E_1$ and $E_i$ is consistent for any $i\ge 0$ by Lemma \ref{lem:Ei:consistent}.

  For any $h\in\HB_P$, it is clear that  $h\in E\cap\HB_P$
  iff $h\in I$
  iff $h\in{\cal T}_{{\cal K}^I}^n(\emptyset,I)$ for some $n\ge 0$ since $I={\cal T}_{{\cal K}^I}^\alpha$.

  On {\color{red} the} one hand, $h\in{\cal T}_{{\cal K}^I}^n(\emptyset,I)$ implies $h\in E_{n+1}\cap\HB_P$ by (i) of Lemma \ref{lem:10} and
  then $h\in\bigcup_{i\ge 0}(E_i\cap\HB_P)$, i.e. $h\in(\bigcup_{i\ge 0}E_i)\cap\HB_P$. On the other hand
  $h\in(\bigcup_{i\ge 0}E_i)\cap\HB_P$ implies $h\in \bigcup_{i\ge 0}(E_i\cap\HB_P)$, i.e. $h\in E_n\cap\HB_P$ for some $n\ge 0$. It follows that $h\in {\cal T}^n_{{\cal K}^I}(\emptyset,I)$ by (ii) of Lemma \ref{lem:10}.
  Thus $h\in I$ and the $h\in E\cap\HB_P$.

  Consequently, we have $E\cap\HB_P=(\bigcup_{i\ge 0}E_i)\cap\HB_P$.

  $(\Lto)$ By Theorem 3 of \cite{Yi-Dong:IJCAI:2011}, it is clear that ${\cal T}^\alpha_{{\cal K}^I}(\emptyset,I)\subseteq I$. We only need to show $I\subseteq{\cal T}_{{\cal K}^I}^\alpha(\emptyset,I)$.
  For any atom $h\in I$, we have that \\
  $h\in E$\\
  $\Rto h\in (\bigcup_{i\ge 0}E_i)\cap\HB_P$ since $E=\bigcup_{i\ge 0}E_i$\\
  $\Rto h\in E_n\cap\HB_P$ for some $n\ge 0$ since $E_i$ is consistent for any $i\ge 0$\\
  $\Rto h\in {\cal T}_{{\cal K}^I}^n(\emptyset,I)$ by (ii) of Lemma \ref{lem:10}\\
  $\Rto h\in {\cal T}^\alpha_{{\cal K}^I}(\emptyset,I)$.

This completes the proof.
\end{proof}

Together with Theorem~\ref{thm:delete:ominus:s} and Proposition~\ref{prop:9}, the above theorem implies another translation
from dl-programs to default theories that preserves the strong answer set semantics.
\begin{corollary}
  Let ${\cal K}=(O,P)$ be a dl-program where $O$ is consistent  and $I\subseteq\HB_P$.
  \begin{itemize}
    \item If $\DL_P^?=\emptyset$  then
     $I$ is a strong answer set of $\cal K$ iff $\Th(\tau(O)\cup I\cup\neg\overline I)$ is an extension of $\tau^*(\cal K)$ iff $I$ is a strongly well-supported answer set of $\cal K$.
    \item  $I$ is a strong answer set of $\cal K$ iff $\Th(\tau(O)\cup \pi(I)\cup\neg \overline{ \pi(I)})$
  is an extension of $\tau^*(\pi(\cal K))$.
  \end{itemize}
\end{corollary}

We note that the translation $\tau^*$ does not preserve the strongly well-supported answer sets
of dl-programs. For instance, let us consider the dl-program ${\cal K}_1$ in Example \ref{exam:dl-program:3}. It is
easy to see that the only strongly well-supported answer set of ${\cal
  K}_1$ is $\emptyset$, while $\tau^*({\cal K}_1)$
has two extensions $\Th(\{\neg p(a)\}\cup\tau(\emptyset))$ and $\Th(\{p(a)\}\cup\tau(\emptyset))$.  However, the translation $\tau^*$ does preserve the strongly well-supported answer sets for a highly relevant class of dl-programs as illustrated by the next proposition. The following lemma is a generalization of Corollary~4 of \cite{Yi-Dong:IJCAI:2011}.

\begin{lemma}\label{lem:14}
  Let ${\cal K}=(O,P)$ be a dl-program such that, for every rule of
  the form (\ref{dl:rule:set}) in $P$, the dl-atom $B$ is monotonic if
  $B\in \Neg$, and $I\subseteq\HB_P$.
  Then $I$ is a weakly well-supported answer set of $\cal K$ iff
  $I$ is a strongly well-supported answer set of $\cal K$.
\end{lemma}
\begin{proof}
  The direction from right to left is implied by Corollary 2 of \cite{Yi-Dong:IJCAI:2011} which asserts this for arbitrary dl-programs. To show the other direction, it suffices to prove
  \[{\cal T}^n_{{\cal K}^I}(\emptyset, I)={\cal T}^n_{\cal K}(\emptyset, I)\]
 for every $n\ge 0$ by induction.

 Base: the case $n=0$ is obvious.

 Step: suppose the statement holds for $n$ and consider the  case $n+1$.
  For any atom $h\in\HB_P$, we have that
  $h \in {\cal T}^{n+1}_{{\cal K}^I}(\emptyset, I)$
  iff there exists a rule $r\in P$ such that
  \begin{itemize}
    \item $({\cal T}^n_{{\cal K}^I}(\emptyset, I), I)\models_O A$ for any $A\in \Pos(r)$, and
    \item $I\not\models_OB$ for any $B\in\Neg(r)$.
  \end{itemize}

  Recall that $I$ is a weakly well-supported answer set of $\cal K$, by which
  $({\cal T}^n_{{\cal K}^I}(\emptyset, I), I)\subseteq I$. It shows that (a) if $B$ is an atom then $I\not\models_OB$ iff $({\cal T}^n_{{\cal K}^I}(\emptyset, I), I)\not\models_O B$, and (b) if $B$ is a monotonic dl-atom then
  $I\not\models_OB$ iff $({\cal T}^n_{{\cal K}^I}(\emptyset, I), I)\not\models_O B$ as well.
  It follows that $h \in {\cal T}^{n+1}_{{\cal K}^I}(\emptyset, I)$ iff $h\in {\cal T}^{n+1}_{\cal K}(\emptyset, I)$ by inductive assumption.
\end{proof}

\begin{proposition}
  Let ${\cal K}=(O,P)$ be a dl-program such that, for every rule of
  the form (\ref{dl:rule:set}) in $P$, the dl-atom $B$ is monotonic if
  $B\in \Neg$, and $I\subseteq\HB_P$.
  Then $I$ is a strongly well-supported answer set of $\cal K$ iff
  $E=\Th(\tau(O)\cup I\cup \neg \overline I)$ is an extension of $\tau^*(\cal K)$.
\end{proposition}
\begin{proof}
  In terms of the definition of weakly and strong well-supported answer sets, it is obvious that\\
\quad  $I$ is a strongly well-supported answer set of $\cal K$\\
\quad  iff $I$ is a weakly well-supported answer set of $\cal K$ by Lemma \ref{lem:14}\\
\quad  iff $\Th(\tau(O)\cup I\cup \neg\overline I)$ is an extension of $\tau^*(\cal K)$ by Theorem \ref{thm:8}.
\end{proof}

At a first glance, in order to preserve the strongly well-supported
answer set semantics, one might suggest to ``shift" $\neg\tau(.)$ for
all dl-atoms from justification to the premise of a default. This does not work, as illustrated by the
dl-program ${\cal K}=(\emptyset, P)$ where $P=\{p(a)\lto\Not \DL[S\oplus p,S'](a)\}$. It is obvious that $\cal K$ has a strongly well-supported answer set $\{p(a)\}$. But according to the suggestion,
we would have the default theory $\Delta=(D,W)$ where $W=\tau(\emptyset)$ and $D$ consists of
\begin{align*}
  \frac{\neg((p(a)\supset S(a))\supset S'(a)):}{p(a)}, \qquad \frac{:\neg p(a)}{\neg p(a)}.
\end{align*}
Its unique extension is $\Th(\{\neg p(a)\}\cup\tau(\emptyset))$, which does not
correspond to any strongly well-supported answer set of $\cal K$. The
reader can further check the dl-program ${\cal K}_1$ in Example
\ref{exam:dl:program:1} and see that ``shifting" $\tau(.)$ for all
dl-atoms from premise to justification of a default does not work
under the weak answer set semantics either.

For general ontologies (consistent or inconsistent), we can slightly
modify the translation $\pi^*$
similarly as  $\tau$ to $\tau'$,   to obtain a
transformation ${\pi^*}'$ and derive analogous results  for it.

\begin{table}[t]
\centering

\caption{Translations from dl-programs with consistent ontologies to default theories}\label{table:1}

\bigskip

\begin{tabular}{l}
\begin{tabular}{l|c|c|c|c}
                & WAS & SAS & WWAS & SWAS \\
  \hline
  Canonical dl-programs & $\sigma\cdot\tau$ & $\tau/\tau^{*}$ & $\tau^{*}$ & $\tau^{*}$ \\
  \hline
  Normal dl-programs    & $\sigma\cdot\pi\cdot\tau$ & $\pi\cdot(\tau/\tau^{*})$ & $\tau^{*}$ & -- \\
  \hline
  Arbitrary dl-programs & $\sigma\cdot\pi\cdot\tau$ & $\pi\cdot(\tau/\tau^{*})$ & $\tau^{*}$ & -- \\
  \hline
\end{tabular}
\\
{\footnotesize
\begin{tabular}{l}
~ \\
  --: unknown; WAS: weak answer
    sets;  SAS: strong answer sets; \\
 WWAS: weakly well-supported answer sets; SWAS: strongly well-supported answer sets.
\end{tabular}
}
\end{tabular}
\end{table}

Let us now summarize the translations in Table \ref{table:1}.
Note that all the translations
$\tau,\tau^*,\sigma$ and $\pi$ are faithful and modular, and the first
three are polynomial. In addition, $\pi$ is polynomial
 relative to the knowledge of the non-monotonic dl-atoms $\DL_P^{?}$, and thus
e.g.\ polynomial for normal dl-programs.
Table \ref{table:1} shows that, for canonical dl-programs with consistent ontologies, we have polynomial, faithful and modular translations for all the semantics, weak answer sets, strong answer sets, weakly well-supported answer sets and strongly well-supported answer sets. 

In addition, under weak answer set and weakly well-supported answer set semantics, all the translations are polynomial, faithful and modular as well. One should note that, for normal dl-programs, the translation is also polynomial, faithful and modular.   There are two unsolved problems, both involving the question whether there exist translations from dl-programs to default theories preserving strongly well-supported answer sets.
In Table~\ref{table:1},  it is assumed that dl-programs have consistent
ontologies. To remove this assumption, it is sufficient to replace $\tau$ (resp., $\tau^*$) with $\tau'$ (resp., $\tau^{*'}$).

\section{Related Work}\label{Sec:Related-work}

Recently, there are some extensive interests in the FLP semantics for various kinds of logic programs \cite{DBLP:Faber:AIJ:2011,Lee:IJCAI:2011,DBLP:journals/ai/Truszczynski10}. Also, in formulating the well-founded semantics for dl-programs, Eiter {\em et al.} proposed a method to eliminate the constraint operator from dl-programs \cite{Eiter:TOCL2011}.
  Moreover, there exist a number of formalisms integrating ontology and (nonmonotonic) rules for the semantics web that can somehow be used to embed dl-programs.
In this section we will relate our work with these approaches.

\subsection{FLP-answer sets of dl-programs}

Dl-programs have been extended to HEX programs that combine
answer set programs with higher-order atoms and external atoms \cite{DBLP:Eiter:IJCAI05B}. In particular,
external atoms can refer, as dl-atoms in dl-programs, to concepts belonging to a classical
knowledge base or an ontology. In such a case one can compare the semantics of the HEX
program with that of the corresponding dl-program. The semantics of HEX programs is based on
the notion of FLP-reduct \cite{DBLP:Faber:JELIA:2004}. We also note that the semantics of dl-programs has been investigated
from the perspective of the quantified logic of here-and-there \cite{DBLP:Fink:JELIA:2010}.
For comparison purpose, we rephrase the FLP-answer set
semantics of dl-programs according to
\cite{DBLP:Eiter:IJCAI05B} in our setting.

Let ${\cal K}=(O,P)$ be a dl-program and $I\subseteq\HB_P$. The {\em FLP-reduct} of $\cal K$ relative to $I$,
written ${\cal K}^{f,I}$, is the dl-program $(O,f\!P_O^I)$ where $f\!P_O^I$ is the set of all rules
of $P$ whose bodies are satisfied by $I$ relative to $O$. An interpretation $I$ is an
FLP-answer set of a dl-program $\cal K$ if $I$ is a minimal model of $f\!P_O^I$  (relative to $O$).
It has been shown that, for a dl-program ${\cal K}=(O,P)$, if $P$ mentions no nonmonotonic dl-atoms, i.e.,
$\DL_P^?=\emptyset$, then the FLP-answer sets
of $\cal K$ coincide with the strong answer sets of $\cal K$
(cf. Theorem 5 of \cite{DBLP:Eiter:IJCAI05B}). Moreover, following the
approach on \cite{Yisong:ICLP:2010},
it can be shown that the FLP-answer
sets of a dl-program are exactly the minimal strong answer sets of the dl-program.

Note that, given a dl-program ${\cal K}=(O,P)$, there are no nonmonotonic dl-atoms in $\pi({\cal K})$. Thus
the strong answer sets of $\pi({\cal K})$ are exactly the FLP-answer sets of $\pi({\cal K})$.
In general however, since FLP-answer sets are minimal strong answer sets and not vice versa, and
$\pi$ preserves strong answer sets, it is clear that
$\pi$ does not preserve the FLP-answer sets of dl-programs. This can be seen from Example \ref{exam:dl-program:3}.
This fact reinforces our argument that there is no transformation to
eliminate the constraint operator from nonmonotonic dl-atoms such that the transformation preserves both strong answer sets and
FLP-answer sets of dl-programs. It is still open to us whether there
is a translation to eliminate the constraint operator from
nonmonotonic dl-atoms while preserving the FLP-answer sets of
dl-programs.

As illustrated by Example \ref{exam:6}, the translations $\tau$ and $\tau^*$ from dl-programs
into default theories do not preserve FLP-answer sets. In addition, the translation $\tau$ may induce some extensions that correspond neither to
strong answer sets nor to FLP-answer sets. Recall that, for dl-programs mentioning no nonmonotonic dl-atoms, the strong
answer sets coincide with the FLP-answer sets. By Theorem \ref{thm:6}, the following Corollary is obvious.
\begin{corollary}
  Let ${\cal K}=(O,P)$ be a dl-program such that $\DL_P^?=\emptyset$
  and $I\subseteq\HB_P$. Then
  $I$ is an FLP-answer set of $\cal K$ if and only if $\Th(I)$ is an extension of $\tau'(\cal K)$.
\end{corollary}

Since the constraint operator is the only that causes a dl-atom to be nonmonotonic, it follows that for dl-programs without the constraint operator, the strong answer set semantics and the FLP-answer set semantics can both be captured by default logic via a polynomial time transformation.

\subsection{Eliminating the constraint operator for well-founded semantics}
To the best of our knowledge, there is only one proposal to remove the
constraint operator in dl-programs, for
the definition of a well-founded semantics for dl-programs \cite{Eiter:TOCL2011}.  In fact, our translation draws ideas from theirs in order to preserve strong answer sets of dl-programs. However, there are subtle differences which make them significantly different in behaviors.
Let us denote their
transformation by $\pi'$. Given a dl-program ${\cal K}=(O,P)$ and
a dl-rule $r\in P$, $\pi'(r)$ consists of
\begin{enumerate}[(1)]
  \item if $S\ominus p$ occurs in a dl-atom of $r$, then $\pi'(r)$ includes
  the instantiated rules obtained from
    \begin{align*}
        \overline p(\vec X)& \lto\Not \DL[S'\oplus p;S'](\vec X).
    \end{align*}\comment{
\tecomment{I think the second rule (constraint is redundant). // Yes it is redundant. -Yisong}
}where $S'$ is a fresh concept (resp., role) name if $S$ is a concept (resp.,
    role) name, $\vec X$ is a tuple of distinct variables matching the
    arity of $p$,
  \item $\pi'(r)$ includes the rule obtained from $r$ by replacing each   ``$S\ominus p$" with ``$\neg S\oplus\overline p$"\footnote{
  It is ``$S\odot \overline p$" according to \cite{Eiter:TOCL2011} which is equivalent to ``$\neg S\oplus \overline p$".}. Let us denote by $\pi'(A)$ the result obtained from $A$ by replacing every $S\ominus p$ with $\neg S\oplus\overline p$ where $A$ is an atom or dl-atom.
\end{enumerate}
Similarly, $\pi'(\mathcal K)=(O,\pi'(P))$ where
$\pi'(P)=\bigcup_{r\in P}\pi'(r)$.
Let us consider the dl-program ${\cal K}_2$ in
Example \ref{exam:dl:program:1}, $\pi'(P_2)$ consists of
\begin{align*}
   p(a)& \lto \DL[S\oplus p, \neg S'\oplus\overline q;S\sqcap \neg S'](a), \\
   \overline q(a) & \lto \Not\DL[S''\oplus q;S''](a).
\end{align*}
It is not difficult to verify that $\pi'(\mathcal K_2)$ has a
unique strong answer set $\{\overline q(a)\}$.
Thus, $\pi'$ loses a strong answer set, as $\{p(a)\}$ is a
strong answer set of $\mathcal K_2$ but there is no corresponding
strong answer set for $\pi'(\mathcal K_2)$.
\comment{
\tecomment{As said, this example only works if one moves to a setting where the
set $\DL_P^?$ is changed in passing from $P$ to $\pi(P)$. To keep it
artificially the same, one can use ``double negation'' for rewritten
atoms as on page 22(?) // Now we do not assume $\DL_P^?$ is given as a parameter, instead it is uniquely determined by that dl-program. -Yisong}
}

The translation $\pi'$ may even remove FLP-answer sets, as illustrated by the next example.
Consider the dl-program ${\cal K}$ in Example \ref{exam:6}.
It is not difficult to verify that the unique FLP-answer set of $\cal K$ is $\{p(a),q(a)\}$. However we have
$\pi'({\cal K})= (\emptyset,\pi'(P))$ where $\pi'(P)$ consists of
\begin{align*}
  p(a) & \lto q(a),\\
  q(a) & \lto \DL[S_1\oplus p, \neg S_2\oplus \overline q; S_1\sqcup \neg S_2](a),\\
  \overline q(a) & \lto \Not\DL[S'\oplus q, S'](a).
\end{align*}
Interested readers can check that $\pi'(\cal K)$ has no FLP-answer sets. Note that since any FLP-answer set is a strong answer set, this is another example where a strong answer set is removed by the translation.


The discussion above leads to a related question \-- whether the translation $\pi'$ introduces extra strong answer sets, for a given dl-program ${\cal K}  = (O,P)$. Note that in our translation $\pi$, for a predicate $p$ we use predicate $\pi_p$ to denote the opposite of $p$, while in the translation $\pi'$, the symbol $\overline p$ is used. After reconciling this name difference, we see that the rule
$\overline p(\vec X) \lto\Not \DL[S'\oplus p;S'](\vec X)$ in the translation $\pi'$, where $S'$ is a fresh concept or role name, is equivalent to rule (\ref{trans:pi:3}) in the translation $\pi$. Then, the only difference is to apply ``double negation" in the case of $\pi$ to positive nonmonotonic dl-atoms.  Given a dl-program ${\cal K}$, suppose an interpretation $I$ is a strong answer set of $\pi'({\cal K})$. Then $I$ is the least model of $\pi'({\cal K})^{s,I}$. It is not difficult to show that, in the fixpoint construction,
for any atom $p \in \HB_{\pi' (P)}$,
$p$ is derivable using $\pi'({\cal K})^{s,I}$ if and only if $p$ is derivable using $\pi({\cal K})^{s,I}$. Therefore, $I$, possibly plus some atoms in the form of $\pi_A$, yields a strong answer set of $\pi({\cal K})^{s,I}$.



\begin{proposition}
Let ${\cal K} = (O,P)$ be a dl-program and $I \subseteq \HB_{\pi'(P)}$ a strong answer set of  $\pi'({\cal K})$.
Then $I\cap \HB_P$ is a strong answer set of $\cal K$.
\label{strictlyStronger}
\end{proposition}
\begin{proof}
  Let $I^*=I\cap\HB_P$, and we prove $I^*$ is a strong answer set of $\cal K$. It is completed by showing $I^*=\lfp(\gamma_{{\cal K}^{s,I^*}})$.

  $(\subseteq)$ We prove the direction by showing $\HB_P\cap \gamma^k_{[\pi'({\cal K})]^{s,I}}\subseteq \lfp(\gamma_{{\cal K}^{s,I^*}})$ for any $k\ge 0$.

  Base: It is trivial for $k=0$.

  Step: Suppose it holds for the case $k$. Let us consider the case $k+1$. For any atom $p$ in $\HB_P$ such that $p\in\gamma^{k+1}_{[\pi'({\cal K})]^{s,I}}$, there exists a rule $(p\lto \Pos, \Not\Neg)$ in $P$ such that
  \begin{itemize}
    \item $\gamma^k_{[\pi'({\cal K})]^{s,I}}\models_O\pi'(A)$ for any $A\in\Pos$, and
    \item $I\not\models_O\pi'(B)$ for any $B\in\Neg$.
  \end{itemize}
  It follows that
  \begin{itemize}
    \item If $A$ is an atom or monotonic dl-atom then $\HB_P\cap \gamma^k_{[\pi'({\cal K})]^{s,I}}\models_OA$ by Lemma \ref{lem:main:1}. It follows   $\lfp(\gamma_{{\cal K}^{s,I^*}})\models_OA$ by the inductive assumption. By (ii) of Lemma \ref{lem:main:1}, if $A$ is nonmonotonic then we have $I^*\models_OA$  since $\pi'(A)$ is monotonic, and $\gamma^k_{[\pi'({\cal K})]^{s,I}}\models_O\pi'(A)$ implies $I\models_O\pi'(A)$.
    \item $I^*\not\models_OB$ for any $B\in\Neg$ by Lemma \ref{lem:main:1}.
  \end{itemize}
  Thus we have that $p\in \lfp(\gamma_{{\cal K}^{s,I^*}})$.

  $(\supseteq)$ We prove this direction by showing that  $\gamma^k_{{\cal K}^{s,I^*}}\subseteq   I$ for any $k\ge 0$.

  Base: It is trivial for $k=0$.

  Step: Suppose it holds for the case $k$. Let us consider the case $k+1$. For any atom $p\in \gamma^{k+1}_{{\cal K}^{s,I^*}}$, there exists a rule $(p\lto\Pos,\Not\Neg)$ in $P$ such that
  \begin{itemize}
    \item $\gamma^{k}_{{\cal K}^{s,I^*}}\models_OA$ for any atom and monotonic dl-atom $A\in\Pos$, and $I^*\models_OA$ for any nonmonotonic dl-atom in $\Pos$, and
    \item $I^*\not\models_OB$ for any $B\in\Neg$.
  \end{itemize}
  It follows that
  \begin{itemize}
    \item In the case $A$ is an atom or monotonic dl-atom, we have $I\models_OA$ by the inductive assumption, by which $I\models_O\pi'(A)$ in terms of
        Lemma \ref{lem:main:1}. If $A$ is nonmonotonic then $I\models_O\pi'(A)$ by $I^*\models_OA$.
    \item By Lemma \ref{lem:main:1}, we have $I\not\models_O\pi'(B)$.
  \end{itemize}
  Consequently we have $p\in I$.
\end{proof}

Another interesting observation is that, for the two removed strong answer sets in the examples above, neither is well-supported in the sense of \cite{Yi-Dong:IJCAI:2011}, as neither possesses a level mapping justification.
One would like to know whether $\pi'$ removes all answer sets that are not well-supported. The answer is no, as evidenced by the next example.
Consider the dl-program ${\cal K}_1$ of Example \ref{exam:dl-program:3}, i.e.,
$\mathcal K_1=(\emptyset,P_1)$ where $P_1$ consists of $p(a)\lto\Not \DL[S\ominus p;\neg S](a).$
It is not difficult to see that ${\cal K}_1$ has two strong answer sets,
$\emptyset$ and $\{p(a)\}$, and the latter is not well-supported.
Now $\pi'({\cal K}_1)=(\emptyset, \pi'(P_1))$
where $\pi'(P_1)$ consists of
\begin{align*}
   p(a) & \lto \Not \DL[\neg S\oplus \overline{p};\neg S](a),\\
   \overline p(a) & \lto \Not \DL[S'\oplus p,S'](a).
\end{align*}
It can be verified that both $\{\overline p(a)\}$ and $\{p(a)\}$ are strong answer sets of $\pi'({\cal K}_1)$.
That is, the strong answer set $\{p(a)\}$ that is not well-supported is retained by $\pi'$. Therefore, the translation $\pi'$ cannot be used as a means to interpret a dl-program under the strongly well-supported semantics.

Continuing the above example by considering the FLP-semantics, we note that $\emptyset$ is the unique
FLP-answer set of ${\cal K}_1$, and the reader can verify that both $\{\overline p(a)\}$ and $\{p(a)\}$ are FLP-answer sets of $\pi'({\cal K}_1)$.
While $\{\overline p(a)\}$ corresponds to the FLP-answer set $\emptyset$ of ${\cal K}_1$
when restricted to $\HB_{P_1}$, the FLP-answer set $\{p(a)\}$ of $\pi'({\cal K}_1)$ has no
corresponding FLP-answer set of ${\cal K}_1$. This shows that extra FLP-answer sets may be introduced by $\pi'$.

The next example shows that the translation $\pi'$ may remove weakly well-supported answer sets.
Recall the dl-program
${\cal K}=(\emptyset,P)$ where $P = \{p(a)\lto \DL[S\odot p, S\ominus p;\neg S](a)\}$. It can be verified
that $\{p(a)\}$ is a weakly well-supported answer set of ${\cal K}$ (it is also strongly well-supported simply because there is no negative dl-atom in the rule).
The $\pi'$ translation results in
$$
\begin{array}{ll}
p(a) \lto \DL[S\odot p, \neg S \oplus \overline p; \neg S] (a),\\
\overline p(a) \lto \Not \DL[S' \oplus p, S'] (a).
\end{array}
$$
It is clear that $\pi'({\cal K})$ has no strong answer sets. Thus, the translation $\pi'$ is too strong for the weakly well-supported semantics.

To summarize, the translation $\pi'$ defined for the well-founded semantics of dl-programs is too strong for the strong answer set semantics, and for the FLP semantics and well-supported semantics, it is sometimes too strong and sometimes too weak.


\subsection{Other embedding approaches}
As to embedding dl-programs into other formalisms that integrate ontology and (nonmonotonic) rules for the semantic web, there are a number of proposals, such as first-order autoepistemic logic \cite{DBLP:conf/kr/BruijnET08},
MKNF knowledge base \cite{Motik:JACM:2010}, quantified equilibrium
logic \cite{DBLP:Fink:JELIA:2010}, and first-order  stable logic programs \cite{DBLP:AIJ:Ferraris:2011,DBLP:Lee:LPNMR:2011}.
In addition to the differences between default logic and those formalisms,\footnote{A discussion of these differences is out of the scope of this paper.} we also considered the weakly and strongly well-supported answer set semantics of dl-programs, recently proposed by \cite{Yi-Dong:IJCAI:2011}.

The discussion below will be based on the strong answer set semantics.
As we mentioned at the end of Section 3, the embedding presented by
Motik and Rosati works only for canonical dl-programs. By the result
of this paper, their embedding can be now extended to normal
dl-programs by applying first the translation $\pi$.
For dl-programs without nonmonotonic dl-atoms, our embedding does not
introduce new predicates.
 The latter is done by the translation of
dl-programs into first-order stable logic programs
\cite{DBLP:AIJ:Ferraris:2011} by  \citeN{DBLP:Lee:LPNMR:2011}, even for canonical dl-programs.

As commented earlier, the current embedding into quantified
equilibrium logic \cite{DBLP:Fink:JELIA:2010}
works for normal dl-programs only, as the authors
adopt a convention that all dl-atoms
containing an occurrence of $\ominus$ are nonmonotonic. The embedding
of dl-programs into first-order autoepistemic logic in \cite{DBLP:conf/kr/BruijnET08} is under the weak
answer set semantics. For the strong answer set semantics, it is
obtained indirectly, by embedding MKNF into first-order autoepistemic
logic, together with the embedding of dl-programs into MKNF. Thus it works for canonical
dl-programs only.

We also notice that, to relate default theories with dl-programs,
\citeN{DBLP:journals/ai/EiterILST08} and
\citeN{DBLP:Dao-Tran:ecsqaru:2009} presented transformations of
a class of default theories, in which only conjunctions of literals are
permitted in defaults, to canonical dl-programs (with variables) and
to cq-programs respectively. Informally, cq-programs can be viewed
as a generalization of canonical dl-programs, where the heads of dl-rules
can be disjunctive and
queries in dl-atoms can be also (decidable) conjunctive queries over
the ontology.
Our
transformation from normal dl-programs to default theories provides
a connection from the other side.
Clearly the class of normal logic programs is a subclass of the normal dl-programs.
Already \citeN{GelfondLifschitz91} have shown that normal logic programs
under answer set semantics correspond  to default logic.
This has now been generalized by our results for normal dl-programs. The work here can be similarly generalized to deal with strong negation as well.

\section{Conclusion}

In this paper, we have studied how dl-programs under various answer set semantics may be captured in default logic.
 Starting with the semantics  in the seminal paper \cite{DBLP:journals/ai/EiterILST08},
we showed that dl-programs under weak and strong answer set semantics can be embedded into default logic. This is
achieved by two key translations: the first is the translation $\pi$ that eliminates the constraint operator from nonmonotonic dl-atoms, and
the second is a translation $\tau$ that transforms a dl-program to a
default theory while preserving strong answer sets of normal dl-programs, provided that the given ontology knowledge base is consistent.
This proviso is not necessary under translation $\tau'$, which
preserves strong answer sets even if the given ontology knowledge base is inconsistent. It also preserves weak answer sets if in addition all dl-atoms occur under default negation. Both translations
$\tau$ and $\tau'$ are polynomial and modular, without resorting to extra symbols.

The translation $\pi$ depends on the knowledge of whether a dl-atom is
monotonic. We have given the precise complexity to determine this
property, for ontology knowledge bases in the description logics $\cal SHIF$ and $\cal SHOIN$.

The importance of these results is that, for all current approaches to representing strong answer sets, either such an approach directly depends on this knowledge
\cite{DBLP:Fink:JELIA:2010,DBLP:Lee:LPNMR:2011}, or the underlying assumption can be removed, with this knowledge and the translation $\pi$ above \cite{DBLP:conf/kr/BruijnET08,Motik:JACM:2010}.

%
Furthermore, the translations $\tau$ and $\tau'$ can be refined
to polynomial, faithful, and modular translations  $\tau^*$ and
$\tau^{*'}$, respectively, which capture the recently proposed weakly well-supported semantics for arbitrary dl-programs \cite{Yi-Dong:IJCAI:2011}.
This is somewhat surprising as the resulting
translations are like writing dl-rules by
defaults in a native language, enhanced only by normal defaults of the form $\frac{:\neg p(\vec c)}{\neg p(\vec c)}$. Apparently, the key is that the iterative definition of default extensions provides a free ride to the weak well-supportedness based on a notion of level-mapping, but not to the strong well-supportedness. This is an interesting insight.
One would expect bigger challenges in representing the same semantics in other nonmonotonic logics.

For the class of dl-programs that mention no constraint operator, i.e.\ the class of canonical dl-programs, all major semantics coincide, including strongly well-supported answer sets, weakly well-supported answer sets,
FLP-answer sets, and strong answer sets. Thus, the translation $\tau'$ can be viewed as a generic representation of dl-programs in default logic. In other words, there is a simple, intuitive way to understand the semantics of (canonical) dl-programs in terms of default logic.
Fortunately, many practical dl-programs are canonical as argued in \cite{Eiter:TOCL2011}. At the same time,
we understand the precise complexity of checking monotonicity of a dl-atom, for some major description logics.
These results strengthen the
prospect of default logic as a foundation for
query-based approaches to integrating ontologies and rules.
In this sense, default
logic can be seen as a promising framework for integrating ontology and rules.
We will look into this issue further in future work.

Though we have presented a faithful and modular embedding for
dl-programs under strong answer set semantics, the embedding is not
polynomial.
It remains as an interesting issue whether there exists
such a polynomial embedding.
In addition, we have shown that $\tau^*$ preserves strongly
well-supported answer sets of a highly relevant class of dl-programs,
viz.\ the one in which nonmonotonic dl-atoms do not occur negatively. It
remains open whether there exists a
faithful, modular  embedding
for arbitrary
dl-programs under the strongly well-supported answer set semantics
into default logic.


\bibliographystyle{theapa}

\begin{thebibliography}{}

\bibitem[\protect\BCAY{Analyti, Antoniou,\ \BBA\ Dam{\'a}sio}{Analyti
  et~al.}{2011}]{DBLP:TOCL:Analyti:2011}
Analyti, A., Antoniou, G., \BBA\ Dam{\'a}sio, C.~V. \BBOP2011\BBCP.
\newblock \BBOQ \textsc{MW}eb: A principled framework for modular web rule
  bases and its semantics\BBCQ\
\newblock {\Bem ACM Transactions on Computational Logic (TOCL)}, {\Bem
  12\/}(2), 17:1--17:46.

\bibitem[\protect\BCAY{Baader, Calvanese, McGuinness, Nardi,\ \BBA\
  Patel-Schneider}{Baader et~al.}{2007}]{Badder:Handbook:DL:2007}
Baader, F., Calvanese, D., McGuinness, D.~L., Nardi, D., \BBA\ Patel-Schneider,
  P.~F. \BBOP2007\BBCP.
\newblock {\Bem The Description Logic Handbook: Theory, Implementation, and
  Applications\/} (2nd \BEd).
\newblock Cambridge University Press, New York, NY.

\bibitem[\protect\BCAY{Bartholomew, Lee,\ \BBA\ Meng}{Bartholomew
  et~al.}{2011}]{Lee:IJCAI:2011}
Bartholomew, M., Lee, J., \BBA\ Meng, Y. \BBOP2011\BBCP.
\newblock \BBOQ First-order extension of the flp stable model semantics via
  modified circumscription\BBCQ\
\newblock In {\Bem The Twenty-Second International Joint Conference on
  Artificial Intelligence (IJCAI-11)}, \BPGS\ 724--730, Barcelona, Spain.
  IJCAI/AAAI.

\bibitem[\protect\BCAY{{B\"{o}rger}, {Gr\"{a}del},\ \BBA\
  Gurevich}{{B\"{o}rger} et~al.}{1997}]{BGG2001}
{B\"{o}rger}, E., {Gr\"{a}del}, E., \BBA\ Gurevich, Y. \BBOP1997\BBCP.
\newblock {\Bem The Classical Decision Problem}.
\newblock Springer.
\newblock Second printing 2001.

\bibitem[\protect\BCAY{Brewka\ \BBA\ Eiter}{Brewka\ \BBA\
  Eiter}{2007}]{Brewka-AAAI-07}
Brewka, G.\BBACOMMA\  \BBA\ Eiter, T. \BBOP2007\BBCP.
\newblock \BBOQ Equilibria in heterogeneous nonmonotonic multi-context
  systems\BBCQ\
\newblock In {\Bem Proceedings of the Twenty-second AAAI Conference on
  Artificial Intelligence (AAAI 2007)}, \BPGS\ 385--390, Vancouver, British
  Columbia, Canada. AAAI Press.

\bibitem[\protect\BCAY{Chen, Wan, Zhang,\ \BBA\ Zhou}{Chen
  et~al.}{2010}]{YinChenJelia2010}
Chen, Y., Wan, H., Zhang, Y., \BBA\ Zhou, Y. \BBOP2010\BBCP.
\newblock \BBOQ dl2asp: Implementing default logic via answer set
  programming\BBCQ\
\newblock In {\Bem Proceedings 12th European Conference on Logics in Artificial
  Intelligence}, \BPGS\ 104--116.

\bibitem[\protect\BCAY{Cholewi\'{n}ski, Marek, Mikitiuk,\ \BBA\
  Truszczy\'{n}ski}{Cholewi\'{n}ski et~al.}{1999}]{Cholewinski:1999}
Cholewi\'{n}ski, P., Marek, V.~W., Mikitiuk, A., \BBA\ Truszczy\'{n}ski, M.
  \BBOP1999\BBCP.
\newblock \BBOQ Computing with default logic\BBCQ\
\newblock {\Bem Artificial Intelligence}, {\Bem 112\/}(1-2), 105--146.

\bibitem[\protect\BCAY{Dao-Tran, Eiter,\ \BBA\ Krennwallner}{Dao-Tran
  et~al.}{2009}]{DBLP:Dao-Tran:ecsqaru:2009}
Dao-Tran, M., Eiter, T., \BBA\ Krennwallner, T. \BBOP2009\BBCP.
\newblock \BBOQ Realizing default logic over description logic knowledge
  bases\BBCQ\
\newblock In {\Bem Symbolic and Quantitative Approaches to Reasoning with
  Uncertainty, 10th European Conference, ECSQARU 2009, Verona, Italy},
  \lowercase{\BVOL}\ 5590 of {\Bem Lecture Notes in Computer Science}, \BPGS\
  602--613. Springer.

\bibitem[\protect\BCAY{de~Bruijn, Eiter, Polleres,\ \BBA\ Tompits}{de~Bruijn
  et~al.}{2007}]{eiter-auto-ontology2007}
de~Bruijn, J., Eiter, T., Polleres, A., \BBA\ Tompits, H. \BBOP2007\BBCP.
\newblock \BBOQ Embedding non-ground logic programs into autoepistemic logic
  for knowledge-base combination\BBCQ\
\newblock In {\Bem Proceedings of International Joint Conference On Artificial
  Intelligence (IJCAI-07)}, \BPGS\ 304--309, Hyderabad, India.

\bibitem[\protect\BCAY{de~Bruijn, Eiter, Polleres,\ \BBA\ Tompits}{de~Bruijn
  et~al.}{2011}]{DBLP:Bruijn:TOCL}
de~Bruijn, J., Eiter, T., Polleres, A., \BBA\ Tompits, H. \BBOP2011\BBCP.
\newblock \BBOQ Embedding nonground logic programs into autoepistemic logic for
  knowledge-base combination\BBCQ\
\newblock {\Bem ACM Transactions on Computational Logic (TOCL)}, {\Bem
  12\/}(3), 20:1--20:39.

\bibitem[\protect\BCAY{de~Bruijn, Eiter,\ \BBA\ Tompits}{de~Bruijn
  et~al.}{2008}]{DBLP:conf/kr/BruijnET08}
de~Bruijn, J., Eiter, T., \BBA\ Tompits, H. \BBOP2008\BBCP.
\newblock \BBOQ Embedding approaches to combining rules and ontologies into
  autoepistemic logic\BBCQ\
\newblock In {\Bem Principles of Knowledge Representation and Reasoning:
  Proceedings of the Eleventh International Conference, KR 2008}, \BPGS\
  485--495, Sydney, Australia. AAAI Press.

\bibitem[\protect\BCAY{de~Bruijn, Pearce, Polleres,\ \BBA\ Valverde}{de~Bruijn
  et~al.}{2007}]{DBLP:Bruijn:RR:2007}
de~Bruijn, J., Pearce, D., Polleres, A., \BBA\ Valverde, A. \BBOP2007\BBCP.
\newblock \BBOQ Quantified equilibrium logic and hybrid rules\BBCQ\
\newblock In {\Bem Web Reasoning and Rule Systems, First International
  Conference, RR 2007}, \lowercase{\BVOL}\ 4524 of {\Bem Lecture Notes in
  Computer Science}, \BPGS\ 58--72, Innsbruck, Austria. Springer.

\bibitem[\protect\BCAY{Eiter, Ianni, Lukasiewicz, Schindlauer,\ \BBA\
  Tompits}{Eiter et~al.}{2008}]{DBLP:journals/ai/EiterILST08}
Eiter, T., Ianni, G., Lukasiewicz, T., Schindlauer, R., \BBA\ Tompits, H.
  \BBOP2008\BBCP.
\newblock \BBOQ Combining answer set programming with description logics for
  the semantic web\BBCQ\
\newblock {\Bem Artifical Intelligence}, {\Bem 172\/}(12-13), 1495--1539.

\bibitem[\protect\BCAY{Eiter, Ianni, Schindlauer,\ \BBA\ Tompits}{Eiter
  et~al.}{2005}]{DBLP:Eiter:IJCAI05B}
Eiter, T., Ianni, G., Schindlauer, R., \BBA\ Tompits, H. \BBOP2005\BBCP.
\newblock \BBOQ A uniform integration of higher-order reasoning and external
  evaluations in answer-set programming\BBCQ\
\newblock In {\Bem the Nineteenth International Joint Conference on Artificial
  Intelligence (IJCAI-05)}, \BPGS\ 90--96, Edinburgh, Scotland, UK.
  Professional Book Center.

\bibitem[\protect\BCAY{Eiter, Lukasiewicz, Ianni,\ \BBA\ Schindlauer}{Eiter
  et~al.}{2011}]{Eiter:TOCL2011}
Eiter, T., Lukasiewicz, T., Ianni, G., \BBA\ Schindlauer, R. \BBOP2011\BBCP.
\newblock \BBOQ Well-founded semantics for description logic programs in the
  semantic web\BBCQ\
\newblock {\Bem ACM Transactions on Computational Logic (TOCL)}, {\Bem
  12\/}(2), 11:1--11:41.

\bibitem[\protect\BCAY{Faber, Leone,\ \BBA\ Pfeifer}{Faber
  et~al.}{2004}]{DBLP:Faber:JELIA:2004}
Faber, W., Leone, N., \BBA\ Pfeifer, G. \BBOP2004\BBCP.
\newblock \BBOQ Recursive aggregates in disjunctive logic programs: Semantics
  and complexity\BBCQ\
\newblock In {\Bem Logics in Artificial Intelligence, 9th European Conference,
  JELIA 2004}, \lowercase{\BVOL}\ 3229 of {\Bem Lecture Notes in Computer
  Science}, \BPGS\ 200--212, Lisbon, Portugal. Springer.

\bibitem[\protect\BCAY{Faber, Pfeifer,\ \BBA\ Leone}{Faber
  et~al.}{2011}]{DBLP:Faber:AIJ:2011}
Faber, W., Pfeifer, G., \BBA\ Leone, N. \BBOP2011\BBCP.
\newblock \BBOQ Semantics and complexity of recursive aggregates in answer set
  programming\BBCQ\
\newblock {\Bem Artificial Intelligence}, {\Bem 175\/}(1), 278--298.

\bibitem[\protect\BCAY{Fages}{Fages}{1994}]{Fages:JMLCS:1994}
Fages, F. \BBOP1994\BBCP.
\newblock \BBOQ Consistency of clark's completion and existence of stable
  models\BBCQ\
\newblock {\Bem Journal of Methods of Logic in Computer Science}, {\Bem 1},
  51--60.

\bibitem[\protect\BCAY{Ferraris, Lee,\ \BBA\ Lifschitz}{Ferraris
  et~al.}{2011}]{DBLP:AIJ:Ferraris:2011}
Ferraris, P., Lee, J., \BBA\ Lifschitz, V. \BBOP2011\BBCP.
\newblock \BBOQ Stable models and circumscription\BBCQ\
\newblock {\Bem Artificial Intelligence}, {\Bem 175\/}(1), 236--263.

\bibitem[\protect\BCAY{Ferraris\ \BBA\ Lifschitz}{Ferraris\ \BBA\
  Lifschitz}{2005}]{FL2005:TPLP}
Ferraris, P.\BBACOMMA\  \BBA\ Lifschitz, V. \BBOP2005\BBCP.
\newblock \BBOQ Weight constraints as nested expressions\BBCQ\
\newblock {\Bem Theory and Practice of Logic Programming}, {\Bem 5\/}(1-2),
  45--74.

\bibitem[\protect\BCAY{Fink\ \BBA\ Pearce}{Fink\ \BBA\
  Pearce}{2010}]{DBLP:Fink:JELIA:2010}
Fink, M.\BBACOMMA\  \BBA\ Pearce, D. \BBOP2010\BBCP.
\newblock \BBOQ A logical semantics for description logic programs\BBCQ\
\newblock In {\Bem Logics in Artificial Intelligence - 12th European
  Conference, JELIA 2010}, \lowercase{\BVOL}\ 6341 of {\Bem Lecture Notes in
  Computer Science}, \BPGS\ 156--168, Helsinki, Finland. Springer.

\bibitem[\protect\BCAY{Fitting}{Fitting}{1996}]{Fitting:1996}
Fitting, M. \BBOP1996\BBCP.
\newblock {\Bem First-Order Logic and Automated Theorem Proving\/} (2nd Edition
  \BEd).
\newblock Texts in Computer Science. Springer-Verlag, Berlin, Germany.

\bibitem[\protect\BCAY{Gelfond\ \BBA\ Lifschitz}{Gelfond\ \BBA\
  Lifschitz}{1991}]{GelfondLifschitz91}
Gelfond, M.\BBACOMMA\  \BBA\ Lifschitz, V. \BBOP1991\BBCP.
\newblock \BBOQ Classical negation in logic programs and disjunctive
  databases\BBCQ\
\newblock {\Bem New Generation Computing}, {\Bem 9}, 365--385.

\bibitem[\protect\BCAY{Hemachandra}{Hemachandra}{1989}]{DBLP:Hemachandra:jcss:1989}
Hemachandra, L.~A. \BBOP1989\BBCP.
\newblock \BBOQ The strong exponential hierarchy collapses\BBCQ\
\newblock {\Bem Journal of Computer Systtem Science}, {\Bem 39\/}(3), 299--322.

\bibitem[\protect\BCAY{Horrocks\ \BBA\ Patel-Schneider}{Horrocks\ \BBA\
  Patel-Schneider}{2003}]{Horrock:ISWC:2003}
Horrocks, I.\BBACOMMA\  \BBA\ Patel-Schneider, P.~F. \BBOP2003\BBCP.
\newblock \BBOQ Reducing owl entailment to description logic
  satisfiability\BBCQ\
\newblock In {\Bem International Semantic Web Conference (ISWC)},
  \lowercase{\BVOL}\ 2870 of {\Bem Lecture Notes in Computer Science}, \BPGS\
  17--29, Sanibel Island, FL, USA. Springer.

\bibitem[\protect\BCAY{Janhunen}{Janhunen}{1999}]{Janhunen:AMAI:1999}
Janhunen, T. \BBOP1999\BBCP.
\newblock \BBOQ On the intertranslatability of non-monotonic logics\BBCQ\
\newblock {\Bem Annals of Mathematics and Artificial Intelligence}, {\Bem
  27\/}(1-4), 79--128.

\bibitem[\protect\BCAY{Lee\ \BBA\ Palla}{Lee\ \BBA\
  Palla}{2011}]{DBLP:Lee:LPNMR:2011}
Lee, J.\BBACOMMA\  \BBA\ Palla, R. \BBOP2011\BBCP.
\newblock \BBOQ Integrating rules and ontologies in the first-order stable
  model semantics (preliminary report)\BBCQ\
\newblock In {\Bem Logic Programming and Nonmonotonic Reasoning - 11th
  International Conference, LPNMR 2011}, \lowercase{\BVOL}\ 6645 of {\Bem
  Lecture Notes in Computer Science}, \BPGS\ 248--253, Vancouver, Canada.
  Springer.

\bibitem[\protect\BCAY{Li\ \BBA\ You}{Li\ \BBA\ You}{1992}]{li-you92}
Li, L.\BBACOMMA\  \BBA\ You, J. \BBOP1992\BBCP.
\newblock \BBOQ Making default inferences from logic programs\BBCQ\
\newblock {\Bem Computational Intelligence}, {\Bem 7}, 142--153.

\bibitem[\protect\BCAY{Lifschitz}{Lifschitz}{1991}]{DBLP:conf/ijcai/Lifschitz91}
Lifschitz, V. \BBOP1991\BBCP.
\newblock \BBOQ Nonmonotonic databases and epistemic queries\BBCQ\
\newblock In {\Bem Proceedings of the 12th International Joint Conference on
  Artificial Intelligence (IJCAI 1991)}, \BPGS\ 381--386, Sydney, Australia.
  Morgan Kaufmann.

\bibitem[\protect\BCAY{Lifschitz, Tang,\ \BBA\ Turner}{Lifschitz
  et~al.}{1999}]{Lifschitz1999nested}
Lifschitz, V., Tang, L.~R., \BBA\ Turner, H. \BBOP1999\BBCP.
\newblock \BBOQ Nested expressions in logic programs\BBCQ\
\newblock {\Bem Annals of Mathematics and Artificial Intelligence}, {\Bem
  25\/}(3-4), 369--389.

\bibitem[\protect\BCAY{Liu\ \BBA\ You}{Liu\ \BBA\ You}{2010}]{DBLP:Liu:FI:2010}
Liu, G.\BBACOMMA\  \BBA\ You, J.-H. \BBOP2010\BBCP.
\newblock \BBOQ Level mapping induced loop formulas for weight constraint and
  aggregate logic programs\BBCQ\
\newblock {\Bem Fundamenta Informaticae}, {\Bem 101\/}(3), 237--255.

\bibitem[\protect\BCAY{Liu\ \BBA\ You}{Liu\ \BBA\ You}{2011}]{LiuYouTPLP2011}
Liu, G.\BBACOMMA\  \BBA\ You, J.-H. \BBOP2011\BBCP.
\newblock \BBOQ Relating weight constraint and aggregate programs: Semantics
  and representation\BBCQ\
\newblock {\Bem Theory and Practice of Logic Programming}, {\Bem 1\/}(1), 1.
\newblock To appear.

\bibitem[\protect\BCAY{Lukasiewicz}{Lukasiewicz}{2010}]{DBLP:Lukasiewicz:TKDE:10}
Lukasiewicz, T. \BBOP2010\BBCP.
\newblock \BBOQ A novel combination of answer set programming with description
  logics for the semantic web\BBCQ\
\newblock {\Bem IEEE Transactions on Knowledge and Data Engineering}, {\Bem
  22\/}(11), 1577--1592.

\bibitem[\protect\BCAY{Marek\ \BBA\ Truszczynski}{Marek\ \BBA\
  Truszczynski}{1999}]{marek99}
Marek, V.~W.\BBACOMMA\  \BBA\ Truszczynski, M. \BBOP1999\BBCP.
\newblock \BBOQ Stable models and an alternative logic programming
  paradigm\BBCQ\
\newblock In Apt, K., Marek, V., Truszczynski, M., \BBA\ Warren, D.\BEDS, {\Bem
  The Logic Programming Paradigm: A 25-Year Perspective}, \BPGS\ 375--398.
  Springer-Verlag, Berlin.

\bibitem[\protect\BCAY{Motik\ \BBA\ Rosati}{Motik\ \BBA\
  Rosati}{2010}]{Motik:JACM:2010}
Motik, B.\BBACOMMA\  \BBA\ Rosati, R. \BBOP2010\BBCP.
\newblock \BBOQ Reconciling description logics and rules\BBCQ\
\newblock {\Bem Journal of the ACM}, {\Bem 57\/}(5), 1--62.

\bibitem[\protect\BCAY{Nicolas, Saubion,\ \BBA\ St{\'e}phan}{Nicolas
  et~al.}{2001}]{defaultLogicImplementation}
Nicolas, P., Saubion, F., \BBA\ St{\'e}phan, I. \BBOP2001\BBCP.
\newblock \BBOQ Heuristics for a default logic reasoning system\BBCQ\
\newblock {\Bem International Journal on Artificial Intelligence Tools}, {\Bem
  10\/}(4), 503--523.

\bibitem[\protect\BCAY{Niemel{\"{a}}}{Niemel{\"{a}}}{1999}]{Niemela99}
Niemel{\"{a}}, I. \BBOP1999\BBCP.
\newblock \BBOQ Logic programs with stable model semantics as a constraint
  programming paradigm\BBCQ\
\newblock {\Bem Annals of Mathematics and Artificial Intelligence}, {\Bem
  25\/}(3-4), 241--273.

\bibitem[\protect\BCAY{Pratt-Hartmann}{Pratt-Hartmann}{2005}]{hartmann05complexity}
Pratt-Hartmann, I. \BBOP2005\BBCP.
\newblock \BBOQ Complexity of the two-variable fragment with counting
  quantifiers\BBCQ\
\newblock {\Bem Journal of Logic, Language and Information}, {\Bem 14\/}(3),
  369--395.

\bibitem[\protect\BCAY{Reiter}{Reiter}{1980}]{Reiter1980}
Reiter, R. \BBOP1980\BBCP.
\newblock \BBOQ A logic for default reasoning\BBCQ\
\newblock {\Bem Artificial Intelligence}, {\Bem 13\/}(1-2), 81--132.

\bibitem[\protect\BCAY{Rosati}{Rosati}{2005}]{Rosati}
Rosati, R. \BBOP2005\BBCP.
\newblock \BBOQ On the decidability and complexity of integrating ontologies
  and rules\BBCQ\
\newblock {\Bem Journal of Web Semantics}, {\Bem 3\/}(1), 61--73.

\bibitem[\protect\BCAY{Rosati}{Rosati}{2006}]{DBLP:Rosati:KR:2006}
Rosati, R. \BBOP2006\BBCP.
\newblock \BBOQ \textsc{DL}+log: Tight integration of description logics and
  disjunctive datalog\BBCQ\
\newblock In {\Bem Proceedings, Tenth International Conference on Principles of
  Knowledge Representation and Reasoning (KR2006)}, \BPGS\ 68--78, Lake
  District of the United Kingdom. AAAI Press.

\bibitem[\protect\BCAY{Shen}{Shen}{2011}]{Yi-Dong:IJCAI:2011}
Shen, Y.-D. \BBOP2011\BBCP.
\newblock \BBOQ Well-supported semantics for description logic programs\BBCQ\
\newblock In {\Bem The Twenty-Second International Joint Conference on
  Artificial Intelligence (IJCAI-11)}, \BPGS\ 1081--1086, Barcelona, Spain.
  IJCAI/AAAI.

\bibitem[\protect\BCAY{Son, Pontelli,\ \BBA\ Tu}{Son
  et~al.}{2007}]{Son:JAIR2007}
Son, T.~C., Pontelli, E., \BBA\ Tu, P.~H. \BBOP2007\BBCP.
\newblock \BBOQ Answer sets for logic programs with arbitrary abstract
  constraint atoms\BBCQ\
\newblock {\Bem Journal of Artificial Intelligence Research}, {\Bem 29},
  353--389.

\bibitem[\protect\BCAY{Tobies}{Tobies}{2001}]{tobies}
Tobies, S. \BBOP2001\BBCP.
\newblock {\Bem Complexity Results and Practical Algorithms for Logics in
  Knowledge Representation}.
\newblock Ph.D.\ thesis, RWTH Aachen, Germany.

\bibitem[\protect\BCAY{Truszczynski}{Truszczynski}{2010}]{DBLP:journals/ai/Truszczynski10}
Truszczynski, M. \BBOP2010\BBCP.
\newblock \BBOQ Reducts of propositional theories, satisfiability relations,
  and generalizations of semantics of logic programs\BBCQ\
\newblock {\Bem Artificial Intelligence}, {\Bem 174\/}(16-17), 1285--1306.

\bibitem[\protect\BCAY{Wang, You, Yuan,\ \BBA\ Shen}{Wang
  et~al.}{2010}]{Yisong:ICLP:2010}
Wang, Y., You, J.-H., Yuan, L., \BBA\ Shen, Y.-D. \BBOP2010\BBCP.
\newblock \BBOQ Loop formulas for description logic programs\BBCQ\
\newblock {\Bem Theory and Practice of Logic Programming, 26th Int'l.
  Conference on Logic Programming (ICLP'10) Special Issue}, {\Bem 10\/}(4-6),
  531--545.

\end{thebibliography}

\appendix

\section{}
\label{app:monotonic}

\begin{proof*}[\textbf{Proof of Theorem~\ref{theo:mon-dl-complexity} (continued)}]
(i) To show $\EXP$-hardness for the case of $\mathcal{SHIF}$
  knowledge bases, we provide a reduction from deciding
  unsatisfiability of a given knowledge base $O$ in $\mathcal{SHIF}$, which is
  $\EXP$-complete given that deciding satisfiability  is
  $\EXP$-complete
\cite{Horrock:ISWC:2003} and $\EXP$ is closed under complementation, to checking monotonicity
  of a dl-atom $A$ relative to a dl-program $\cal K$ as follows.

Let $C$ be a fresh concept and define the following dl-atom:

\[A = \DL[C \ominus p; \top \sqsubseteq \bot]()\]

\noindent where $p$ is a fresh unary predicate. Furthermore, let
$$O' = O \cup
\{ C(o) \mid o\in {\cal C}\}$$
\noindent where without loss of generality ${\cal C}\neq \emptyset$ is the
set of individuals occurring in $O$.

It is clear that if $O$ is unsatisfiable, then $A$ is
monotonic relative to  ${\cal K} = (O',P)$, where $P = \{ p \leftarrow
A \}$ and $p$ is a fresh propositional atom.
Recall that $A$ is nonmonotonic w.r.t. $O'$ iff there exist two interpretations $I$ and $I'$ such that
$I\subset I'$, $I\models_{O'} A$, and $I'\not\models_{O'}A$.
Every interpretation $I$ such that $p(o)\notin I$ for some $o\in {\cal
  C}$ is a model of $A$ relative to $O'$, and the interpretation
$I\cup\{p(o) \mid o \in {\cal C}\}$ is not a model of $A$ relative to
$O'$ if $O$ is satisfiable. Hence, $A$ is nonmonotonic relative to
${\cal K}$ iff $O$ is satisfiable.
%
%
It follows that the $\EXP$-complete unsatisfiability test reduces to the
DL-monotonicity test, and settles the result for the $\mathcal{SHIF}$ case.

(ii) For the case of $\mathcal{SHOIN}$ knowledge bases, we show
hardness for $\Pol^{\NEXP} = \textmd{co-NP}^{\NEXP}$, building on
machinery used in \cite{DBLP:journals/ai/EiterILST08} for the
complexity analysis of strong and weak answer sets of dl-programs with $\mathcal{SHOIN}$
knowledge bases. In the course of this, an encoding of a torus-tiling
problem (that represents $\NEXP$ Turing machine computations on a
given input) into a DL knowledge base satisfiability problem was
used. We briefly recall this problem.

A {\em domino system} $\mathcal{D}= (D,H,V)$ consists of a finite
non\-empty set $D$ of {\em tiles} and two relations $H,V\subseteq
D\,{\times}\, D$ expressing horizontal and vertical compatibility
constraints between the tiles. For positive integers $s$ and~$t$, and
a word $w = w_0\ldots w_{n-1}$ over~$D$ of length $n\leq s$, we say
that $\mathcal{D}$ {\em tiles} the torus $U(s,t)=\{0,1,\ldots,s\,{-}\,1\}\times\{0,1,\ldots,t\,{-}\,1\}$ \emph{with initial
condition $w$} iff there exists a mapping $\tau\colon
U(s,t)\,{\rightarrow}\, D$ such that for all $(x,y)\in U(s,t)$:
(i)~if $\tau(x,y)=d$ and $\tau((x+1)\,\mathrm{mod}\, s,y)= d'$,
then $(d,d')\in H$,
(ii)~if $\tau(x,y)= d$ and $\tau(x,(y+1)\,\mathrm{mod}\,t)= d'$, then $(d,d')\in V$,
and (iii)~$\tau(i,0) = w_i$ for all $i\in \{0,\ldots,n\}$.
Condition~(i) is the {\em horizontal constraint},
condition~(ii) is the {\em vertical constraint}, and
condition~(iii) is the {\em initial condition}.

Similar as \cite{DBLP:journals/ai/EiterILST08}, we use the following
lemmas.

\begin{lemma}[cf.\ Lemma 5.18 and Corollary 5.22 in \cite{tobies}]\label{LEM-D-0-}
For domino systems $\mathcal{D}= (D,$ $H,V)$
and initial conditions $w= w_0\ldots$ $w_{n-1}$, there exist
DL knowledge bases  $O_{n}$, $O_{\mathcal{D}}$, and $O_w$,
and concepts $C_{i,0}$, $i\in\{0,1,\ldots,n-1\}$, and $C_d$, $d\in D$, in $\mathcal{SHOIN}$  such that:
\begin{itemize}
\item $O_n\cup O_{\mathcal{D}}\cup O_w$ is satisfiable iff $\mathcal{D}$  tiles $U(2^{n+1},$ $2^{n+1})$ with initial condition $w$;
\item $O_{n}$, $O_{\mathcal{D}}$, and $O_w$ can be constructed in polynomial time in $n$ from $n$, $\mathcal{D}$, and $w$, respectively, and
$O_w= \{C_{i,0}\sqsubseteq C_{w_i}\mid i\in\{0,1,\ldots,n-1\}\}$;
\item  in every model of $O_n\cup O_{\mathcal{D}}$,
each $C_{i,0}$ contains exactly one object representing $(i,0)\in U(2^{n+1},2^{n+1})$,
and each~$C_d$ contains all objects associated with~$d$.
\end{itemize}
\end{lemma}

\begin{lemma}[cf.\ Theorem~6.1.2 in \cite{BGG2001}]
\label{LEM-D-0+-}
Let $M$ be a nondeterministic Turing machine with time- (and thus space-) bound~$2^n$,
deciding a $\NEXP{}$-complete language $\mathcal{L}(M)$ over the alphabet $\Sigma=\{0,1,''\phantom{x}''\}$.
Then, there exists a domino system $\mathcal{D}=(D,H,V)$
and a linear-time reduction $\mathit{trans}$ that takes any input $b\in\Sigma^*$ to a word $w\in D^*$
with $|b|=n=|w|$ such that $M$~accepts~$b$ iff $\mathcal{D}$ tiles the
torus~$U(2^{n+1},$ $2^{n+1})$ with initial condition $w$.
\end{lemma}

Based on this, \cite{DBLP:journals/ai/EiterILST08} showed how
computations of a deterministic polynomial time Turing machine with an
$\NEXP$ oracle can be encoded into evaluating a dl-program, where
intuitively dl-atoms correspond to oracle calls.
For the problem at hand, we would have to provide an encoding of
such a computation into one dl-atom and the check of its
monotonicity. To simplify matters, we provide a reduction from the
following problem:

\begin{description}
\item [{\sc NEXP-JC}:] Given two partial inputs $b$ and $b'$ of the
  same $\NEXP$ Turing machine $M$ such that $|b|=|b'|$, does there
  exist a joint completion $c$ of the partial inputs of length
  $|c|=|b|=|b'|$ such that (1) $M$ accepts $bc$ and (2) $M$ does not
  accept $b'c$.
\end{description}

\begin{lemma}
Problem {\sc NEXP-JC} is complete for $\NP^{\NEXP}$ (=$\Pol^{\NEXP}$).
\end{lemma}

Intuitively, this is seen as
follows: the computation path (nondeterministic moves and query
answers) of $M$
can be guessed ahead, and after that only a deterministic computation with
oracle accesses is made, in which the oracle answers are checked with the guesses.
Witnesses for all oracle queries that should answer ``yes'' can be
found in a single $\NEXP$
computation, and all queries that should answer ``no'' can be verified
in a single $\coNEXP$ computation (i.e., a $\NEXP$ computation for
refutation does not accept). The condition $|b|=|b'|=|c|$ can be ensured by simple padding techniques.

\medskip


Now the reduction of this problem to deciding dl-atom monotonicity is
exploiting (and modifying) the torus-tiling problem encoding to
DL satisfiability testing quoted above.
It has been shown in \cite{DBLP:journals/ai/EiterILST08}
how to adapt the torus knowledge base such that the initial condition
$w$ (encoded by $O_w$) can be flexibly established by the update string $\lambda$ of a
dl-atom. Intuitively, ``switches" were used to ``activate" concepts that
represent tiles, so that tiles are put in place by the call of the
dl-atom.

Using a similar idea, we change $O_w$. As in
\cite{DBLP:journals/ai/EiterILST08},  assertions
\[C_{i,0}(o_i),\quad i=0,\ldots,n-1 \] are used to introduce
individuals $o_i$ for the torus positions $(i,0)$ that hold the
initial condition $w$ encoding a complete input $bc$ resp.\ $b'c$, where $n=2m{-}1$ and $m\,{=}\,|b|\,{=}\,|b'|$;
we have ${\cal C} = \{o_0,\ldots,o_{n-1}\}$. We
implement a ``switch" that tells whether computation of either (1)
$bc$ or (2) of $b'c$ should be considered in a call. For this, we use
a concept $S$ and put $S\ominus p$, $S\oplus
p$ in the ``update" $\lambda$ of the dl-atom $A$ that we construct,
which will effect that given any interpretation $I$, for each individual $o_i$ either
$S(o_i)$ or $\neg S(o_i)$ will be asserted in $O(I;\lambda)$.
We pick $o_0$ (i.e., position $(0,0)$ of the torus, which  is
``identified" by the concept $C_{0,0}$) and install on it the switch
between case 1) and 2): if $S(o_0)$ is true, we evaluate case 1), else case
2). To ``prepare'' the part of the initial condition encoding $b$ resp.\ $b'$, we use axioms
\begin{align*}
   C_{0,0} \sqcap S \sqsubseteq B,\\
   C_{0,0} \sqcap \neg S \sqsubseteq \neg B,
\end{align*}
where $B$ is a fresh concept (intuitively, a flag indicating case 1),
i.e., $b$), and an axiom
\begin{align*}
B \sqsubseteq \forall\mathit{east}.B
\end{align*}
where $\mathit{east}$ is a role already defined in
$O_n\cup O_{{\cal D}}$ which links position $(i,j)$ to $(i+1,j)$, for
all $i$ and $j$;
in combination with the above axioms, it effects that when evaluating
a dl-atom w.r.t. an interpretation $I$, in every model of $O(I;\lambda)$  either all
elements $e_i$ at ``input'' positions are labeled with $B$ or all are labeled with $\neg B$.
Depending on the $B$-label, we then assign $e_i$ the right tile from the initial
condition for $b$ (label $B$) respectively for $b'$ (label $\neg B$):
\begin{align*}
\left.\begin{array}{r}
C_{i,0} \sqcap B \sqsubseteq C_{w_i}\\
C_{i,0} \sqcap \neg B \sqsubseteq C_{w'_i}
\end{array}
\right\}\quad i=0,\ldots, m-1,
\end{align*}
where $w_i$ (resp.\ $w'_i$) is the $i$-th tile of $w$ (resp.\ $w'$).
Intuitively, the case of label $B$ is for input $I'$ that is ``larger'' than
input $I$ for label $\neg B$; for the former, we must have $p(o_0)\in
I'$ and for the latter $p(o_0)\notin
I$; the value of $p(o_i)$, $i>0$, does not matter, so we can assume
it is the same in $I$ and $I'$. For $I'$ we do the $\NEXP$ test, and for
the ``smaller" $I$ we do the $\coNEXP$ test. If both succeed, we have a
counterexample to monotonicity.

It remains to incorporate the guess $c$ for the completion
of the input. This guess can be built in by using concepts $S_d$
such that $S_d(o_i)$ intuitively puts tile $d$
at the position $i$ in the initial condition (where $i=m,\ldots, n-1$
runs from the first position after $b$ (resp. $b'$)  until the last position
of the fully completed input $bc$ (resp. $b'c$), viz.\ $n-1$). In the  input list
$\lambda$ of the dl-atom $A$, we put
   \[S_{d} \ominus p_{d},\ S_{d} \oplus p_{d} \quad d \in D\]
where $p_d$ is a fresh unary predicate  ($D$ is the set of tiles). Similar as above, this will
assert for each individual then either $S_d$ or $\neg S_d$.

We then add axioms which put on tiles as follows:
\[
\left.\begin{array}{r}
C_{i,0} \sqcap S_{d} \;\sqsubseteq\; C_d\\
C_{i,0} \sqcap \bigsqcap_{d\in D} \neg S_d \;\sqsubseteq\; C_{d_0}
\end{array}
\right\}\quad i=m,\ldots,n-1, d\in D
\]
where $d_0$ is some fixed tile; the second axiom puts a default tile if
in $I$ no tile has been selected (as if $p_{d_0}(o_i)$ would be in $I$). If multiple tiles have been selected, then the
$O(I;\lambda)$ is  unsatisfiable, and similarly $O(I';\lambda)$ for
each $I' \supset I$. So the interesting case is if exactly one
tile has been put on in each ``completion'' position $i=m,\ldots,n-1$
of the initial condition.
The selection of tiles is subject to further constraints  on
tiles at adjacent positions $i{-}1$,$i$ from $m,\ldots,n-1$ and on the
last position, due to the encoding of the machine input into the initial
condition in \cite{BGG2001}. Without going into detail here, let $A
\subset D^2$ and $F\subset D$ be the sets of admissible adjacent tiles
$(d,d')$ and final tiles $d$, respectively (which are easily determined). We then add axioms
\begin{align*}
C_{i,0} \sqcap C_{d'} \; \sqsubseteq & \; \forall\mathit{east}^-.\bigsqcup_{(d,d')\in A} C_d, \quad i=m,\ldots,n-1, d'\in D,\\
C_{n-1,0} \; \sqsubseteq & \; \bigsqcup_{d\in F} C_{d}.
\end{align*}
This completes the construction of $O_w$. Now let $A = \DL[\lambda; \top
\,{\sqsubseteq}\, \bot]()$ and ${\cal K} = (O,P)$, where $O=O_n \cup
O_{{\cal D}} \cup O_w$ and $P = \{ p(o_0) \leftarrow A \}$.
It can be shown that a violation of the
monotonicity of $A$  relative to $\cal K$ is witnessed by two
interpretations $I \subset I'$ of form $I' = I \cup \{ p(o_0)\}$ such
that $I' \not\models_O A$ and $I \models_O A$ and the interpretations
encode a joint completion $c$ of the inputs
$b$ and $b'$,  meaning that the computation for $bc$ is
accepting while the one for $b'c$ is not. As $\cal K$ and $A$ are
constructible in polynomial time from $b$, $b'$ and $M$, this proves the result.
\end{proof*}

\section{}
\label{app:B}
\begin{lemma}\label{lem:w}
  Let $\mathcal K=(O,P)$ be a dl-program and $I\subseteq\HB_P$. Then we have that
  \begin{enumerate}[(i)]
    \item $\pi_1(I)=\{\pi_p(\vec c)\in\HB_{\pi(P)}\}\cap \lfp(\gamma_{[\pi(\mathcal K)]^{w,\pi(I)}})$,
    \item $\pi_2(I)=\{\pi_A\in\HB_{\pi(P)}\}\cap\lfp(\gamma_{[\pi(\mathcal K)]^{w,\pi(I)}})$, and
    \item $\gamma_{{\cal K}^{w,I}}^k=\HB_P\cap \gamma^k_{[\pi(\mathcal K)]^{w,\pi(I)}}$ for any $k\ge 0$.
  \end{enumerate}
\end{lemma}
\begin{proof}
  (i) It is evident that, for any atom $\pi_p(\vec c)\in\HB_{\pi(P)}$, the rule $(\pi_p(\vec c)\lto \Not p(\vec c))$ is in $\pi(P)$.
  We have that \\
  $\pi_p(\vec c)\in\pi_1(I)$ \\
  iff $p(\vec c)\notin I$ \\
  iff $p(\vec c)\notin \pi(I)$ \\
  iff the rule $(\pi_p(\vec c)\lto)$ belongs to $w[\pi(P)]^{w,\pi(I)}_O$\\
  iff $\pi_p(\vec c)\in \lfp(\gamma_{[\pi(\mathcal K)]^{w,\pi(I)}})$.

  (ii) It is clear that, for any $\pi_A\in\pi_2(I)$, the rule $(\pi_A\lto \pi(\Not A))$ is in $\pi(P)$ such that
  $A\in \DL_P^?$ and $I\not\models_OA$. Let $A=\DL[\lambda;Q](\vec t)$. We have that \\
  $\pi_A\in \pi_2(I)$\\
  iff $\pi_A\in\HB_{\pi(P)}$ and $I\not\models_OA$\\
  iff $\pi(I)\not\models_O \DL[\pi(\lambda);Q](\vec t)$ (by (ii) of Lemma \ref{lem:main:1})\\
  iff the rule $(\pi_A\lto)$ belongs to $w[\pi(P)]^{w,\pi(I)}_O$\\
  iff $\pi_A\in \lfp(\gamma_{[\pi(\mathcal K)]^{w,\pi(I)}})$.

  (iii) We show this by induction on $k$.

  Base: It is obvious for $k=0$.

  Step: Suppose it holds for $k=n$. Let us consider the case $k=n+1$. For any atom $\alpha\in \HB_P$,
  $\alpha\in\gamma^{n+1}_{{\cal K}^{w,I}}$
  if and only if there is a rule
  \[\alpha\lto\Pos, \textit{Mdl}, \NDL,\Not\Neg\]
  in $P$ where $\Pos$ is a set of atoms, \textit{Mdl} a set of monotonic dl-atoms and $\NDL$ a set of nonmonotonic dl-atoms such that
  \begin{itemize}
    \item $\gamma^n_{{\cal K}^{w,I}}\models_OA$ for any $A\in \Pos$,
    \item $I\models_OB$ for any $B\in\NDL$,
    \item $I\models_OB'$ for any $B'\in\textit{Mdl}$, and
    \item $I\not\models_OC$ for any $C\in\Neg$.
  \end{itemize}
  It follows that<
  \begin{itemize}
    \item $\gamma^n_{{\cal K}^{w,I}}\models_OA$ if and only if
    $\gamma^n_{[\pi({\cal K})]^{w,\pi(I)}}\models_OA$ by the  inductive assumption,
    \item $I\models_OB$ if and only if $\pi_B\not\in \pi(I)$ by the definition of $\pi_2(I)$,
    i.e., $\pi(I)\not\models_O\pi_B$,
    \item $I\models_OB'$ if and only if $\pi(I)\models_OB'$, and
    \item $I\not\models_OC$ if and only if $\pi(I)\models_O\pi(\Not C)$ for any $C\in\Neg$ by Lemma \ref{lem:main:1}.
  \end{itemize}
  Thus we have that $\alpha\in\gamma^{n+1}_{{\cal K}^{w,I}}$ if and only if $\alpha\in\gamma^{n+1}_{[\pi(\mathcal K)]^{w,\pi(I)}}\cap\HB_P$.
\end{proof}

\noindent
\begin{proof*}[\textbf{Proof of Theorem~\ref{thm:delete:ominus:w}}]\\
  (i) We have that
  \begin{flalign*}
     \lfp(\gamma_{[\pi({\cal K})]^{w,\pi(I)}})
   = &\lfp(\gamma_{[\pi({\cal K})]^{w,\pi(I)}})\cap (\HB_P\cup\{\pi_p(\vec c)\in\HB_{\pi(P)}\}\cup \{\pi_A\in\HB_{\pi(P)}\})\\
   = & [\HB_P\cap \lfp(\gamma_{[\pi({\cal K})]^{w,\pi(I)}})]\\
     &  \cup[\{\pi_p(\vec c)\in\HB_{\pi(P)}\}\cap\lfp(\gamma_{[\pi({\cal K})]^{w,\pi(I)}})]\\
     &  \cup[\{\pi_A\in\HB_{\pi(P)}\}\cap\lfp(\gamma_{[\pi({\cal K})]^{w,\pi(I)}})]\\
   = & [\HB_P\cap \bigcup_{i\ge 0}\gamma_{[\pi({\cal K})]^{w,\pi(I)}}^i]\cup
       \pi_1(I)\cup \pi_2(I)\mbox{ by (i) and (ii) of Lemma \ref{lem:w}} \\
   = & \bigcup_{i\ge 0}[\HB_P\cap \gamma_{[\pi({\cal K})]^{w,\pi(I)}}^i]\cup
       \pi_1(I)\cup \pi_2(I)\\
   = & \bigcup_{i\ge 0}\gamma_{{\cal K}^{w,I}}^i\cup
       \pi_1(I)\cup \pi_2(I)\mbox{ by (iii) of Lemma \ref{lem:s}} \\
   = & I\cup \pi_1(I)\cup \pi_2(I) \mbox{ since $I$ is a strong answer set of $\cal K$}\\
   = & \pi(I).
  \end{flalign*}
  It follows that $\pi(I)$ is a weak answer set of $\pi(\cal K)$.

  (ii) We prove $I^*=\pi(\HB_P\cap I^*)$ at first.
  \begin{flalign*}
   I^* = & I^*\cap (\HB_P\cup\{\pi_p(\vec c)\in\HB_{\pi(P)}\}\cup \{\pi_A\in\HB_{\pi(P)}\})\\
       = & (I^*\cap \HB_P)\cup (I^*\cap \{\pi_p(\vec c)\in\HB_{\pi(P)}\})\cup (I^*\cap \{\pi_A\in\HB_{\pi(P)}\})\\
       = & (I^*\cap\HB_P) \cup \pi_1(\HB_P\cap I^*)\cup \pi_2(\HB_P\cap I^*) \mbox{ by (i) and (ii) of Lemma \ref{lem:w}}\\
       = & \pi(I^*\cap\HB_P).
  \end{flalign*}
  Let $I=I^*\cap\HB_P$. We have that
  \begin{flalign*}
    \lfp(\gamma_{{\cal K}^{w,I}}) = & \bigcup_{i\ge 0}\gamma^i_{{\cal K}^{w,I}}\\
    = & \bigcup_{i\ge 0}(\HB_P\cap \gamma_{[\pi({\cal K})]^{w,\pi(I)}}^i)\mbox{ by (iii) of Lemma \ref{lem:w}}\\
    = & \HB_P\cap \bigcup_{i\ge 0}\gamma_{[\pi({\cal K})]^{w,\pi(I)}}^i\\
    = & \HB_P\cap\lfp(\gamma_{[\pi({\cal K})]^{w,\pi(I)}})\\
    = & \HB_P\cap \pi(I)\mbox{ since $\pi(I)=I^*$ is a weak answer set of $\pi(\cal K)$}\\
    = & I.
  \end{flalign*}
  It follows that $I$ is a weak answer set of $\cal K$.
\end{proof*}

\end{document}